\documentclass[10pt,twocolumn,letterpaper]{article}

\usepackage[pagenumbers]{cvpr} %

\definecolor{cvprblue}{rgb}{0.21,0.49,0.74}
\definecolor{pink}{RGB}{156,128,166}
\definecolor{yellow}{RGB}{242,194,119}
\usepackage[pagebackref,breaklinks,colorlinks,allcolors=cvprblue]{hyperref}     
\usepackage{amsfonts}
\usepackage{overpic}
\usepackage{iac_pkg}
\usepackage{multirow} 
\usepackage{bbm}
\usepackage{amsthm} %
\usepackage{algorithm,algorithmic}
\usepackage{enumitem}

\usepackage{graphicx}
\usepackage{booktabs}
\usepackage{mwe} %

\usepackage{multirow}

\usepackage[x11names,table,dvipsnames]{xcolor} %
\colorlet{lightgray}{Azure3!35!white}
\newcolumntype{C}{>{\centering\arraybackslash}p{0.65cm}}
\usepackage{lipsum}
\usepackage{nicefrac}
\usepackage{makecell}

\usepackage{arydshln}

\DeclareMathOperator*{\argmin}{argmin}

\newcommand{\reals}{{\mathbb R}}
\renewenvironment{proof}[1][Proof: ]{\noindent \textbf{#1}}{\qed\medskip}
\newtheorem{theorem}{Theorem}[section]
\newtheorem{remark}{Remark}[section]

\newcommand{\softmax}{\text{SoftMax}}
\newcommand{\bg}{\mathbf{g}}
\newcommand{\bv}{\mathbf{v}}

\newcommand{\btheta}{{\boldsymbol{\theta}}}
\newcommand{\of}{\overline{f}}
\newcommand{\onefunc}{\mathbbm{1}}
\newcommand{\inner}[1]{\langle #1 \rangle}

\definecolor{cvprblue}{rgb}{0.21,0.49,0.74}
\definecolor{pink}{RGB}{156,128,166}
\definecolor{yellow}{RGB}{242,194,119}
\usepackage[pagebackref,breaklinks,colorlinks,allcolors=cvprblue]{hyperref}     
\usepackage{amsfonts}
\usepackage{overpic}
\usepackage{iac_pkg}
\usepackage{multirow} 
\usepackage{bbm}
\usepackage{amsthm} %
\usepackage{algorithm,algorithmic}
\usepackage{enumitem}

\usepackage{graphicx}
\usepackage{booktabs}
\usepackage{mwe} %

\usepackage{graphicx}
\usepackage{multirow}
\usepackage[table]{xcolor}
\usepackage{colortbl}
\usepackage{tabularx}

\usepackage[utf8]{inputenc}

\usepackage[accsupp]{axessibility}

\usepackage{multirow}

\renewenvironment{proof}[1][Proof: ]{\noindent \textbf{#1}}{\qed\medskip}

\colorlet{cianoChiaro}{LightCyan3!20!white}
\colorlet{cianoVeryChiaro}{LightCyan4!80!white}

\newcommand{\smallred}[1]{\textcolor{IndianRed4}{(-#1)}}
\newcommand{\smallblue}[1]{\textcolor{Green4}{(+#1)}}

\newcolumntype{G}{>{\columncolor{lightgray}}c}

\def\paperID{2531} %
\def\confName{CVPR}
\def\confYear{2026}

\title{A Provable Energy-Guided Test-Time Defense\\Boosting Adversarial Robustness of Large Vision-Language Models}
\author{
    Mujtaba Hussain Mirza$^{1}$ \quad 
    Antonio D’Orazio$^{1}$ \quad 
    Odelia Melamed$^{2}$ \quad 
    Iacopo Masi$^{1}$ \\[0.5em]
    $^{1}$OmnAI Lab, Computer Science Department, Sapienza University of Rome, Italy \\
    $^{2}$Weizmann Institute of Science, Israel \\
    {\tt\small \{mirza, dorazio, masi\}@di.uniroma1.it}, \quad {\tt\small odelia.melamed@weizmann.ac.il}
}

\begin{document}
\maketitle
\begin{abstract}
Despite the rapid progress in multimodal models and Large Visual-Language Models (LVLM), they remain highly susceptible to adversarial perturbations, raising serious concerns about their reliability in real-world use. While adversarial training has become the leading paradigm for building models that are robust to adversarial attacks, Test-Time Transformations (TTT) have emerged as a promising strategy to boost robustness at inference.
In light of this, we propose \textbf{Energy-Guided Test-Time Transformation (\etthree)}, a lightweight, training-free defense that enhances the robustness by minimizing the energy of the input samples.
Our method is grounded in a theory that proves our transformation succeeds in classification under reasonable assumptions. We present extensive experiments demonstrating that \etthree provides a strong defense for classifiers, zero-shot classification with CLIP, and also for boosting the robustness of LVLMs in tasks such as Image Captioning and Visual Question Answering. Code is available on \href{https://github.com/OmnAI-Lab/Energy-Guided-Test-Time-Defense}{GitHub}.

\end{abstract}

\section{Introduction} \label{sec:introduction} 

Large vision–language models (LVLMs) have achieved remarkable progress, demonstrating strong multimodal reasoning and zero-shot generalization across tasks such as image captioning, visual question answering, and open-world recognition. Recent LVLMs such as LLaVA \cite{liu2023visual}, OpenFlamingo \cite{awadalla2023openflamingo}, and Qwen‑VL \cite{Qwen-VL} present excellent performance; however, despite their broad generalization, they are still vulnerable to attacks to the visual modality~\cite{qi2024visual,carlini2023aligned}, potentially compromising downstream LVLM tasks. 

Crucially, many LVLMs rely on the visual encoder from CLIP~\cite{radford2021learning}, which provides exceptional flexibility and generalization across diverse visual inputs, but also forms the primary source of vulnerability to image-based adversarial attacks: unnoticeable, carefully crafted perturbations to input images can induce incorrect predictions~\cite{szegedy2014, goodfellow2014explaining}.

\begin{figure}[t]
\centering
\includegraphics[width=\columnwidth]{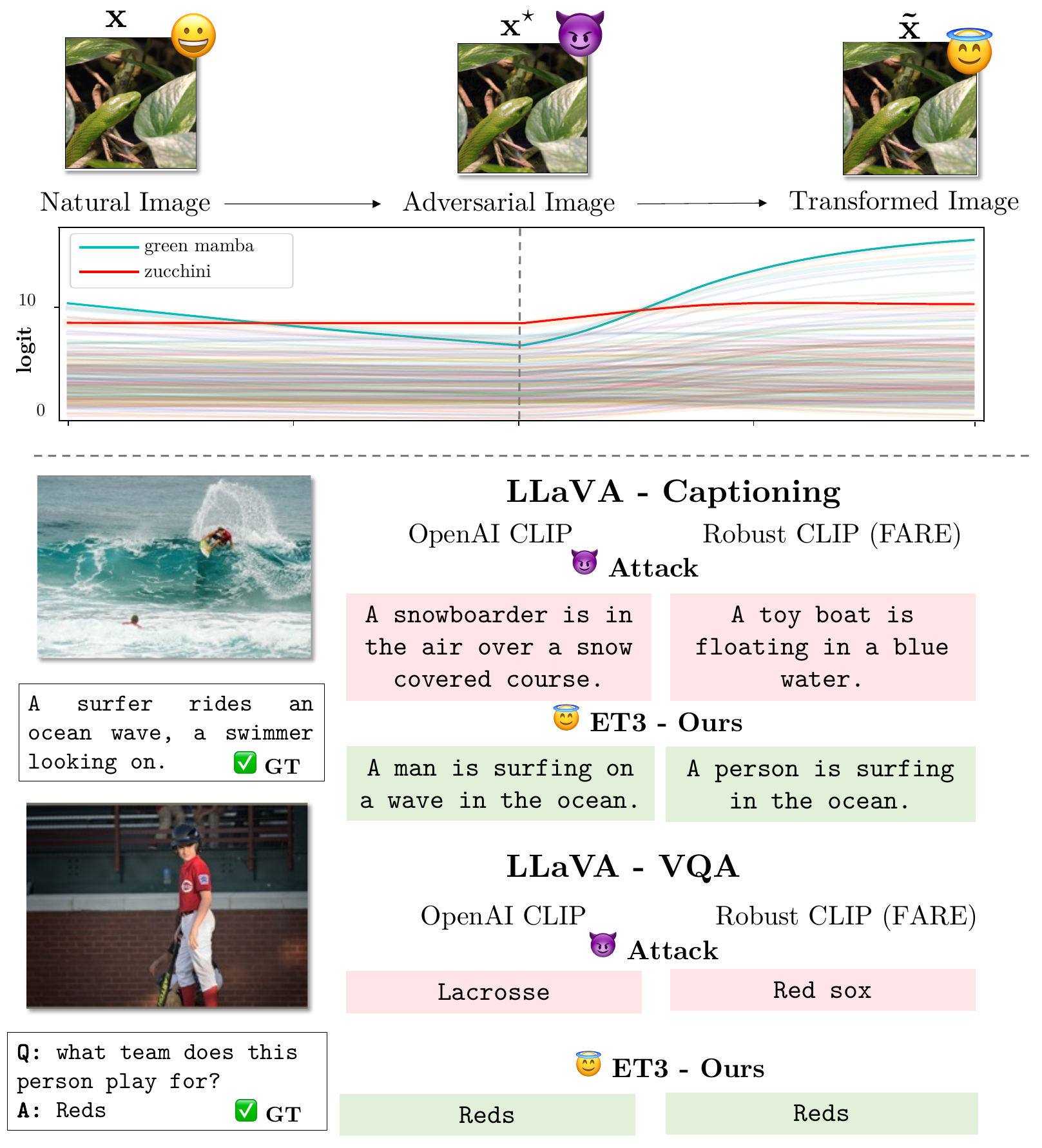}
\caption{\emph{(top)} Presenting a natural image green mamba $\bx$, and its adversarial image $\bxa$ mistakenly classified as a zucchini by a robust classifier $f_\theta$. Given only $\bxa$ and $f_\theta$, our
\etthree test-time defense 
produces the correctly classified $\bxp$, thereby boosting adversarial robustness. The plot illustrates the logit's change: even though the ground-truth class is \emph{not} the second best, \etthree still recovers it. \emph{(bottom)} \etthree boosts robust accuracy of Large VLMs like LLaVa~\cite{liu2023visual} using standard or even Robust CLIP~\cite{schlarmann2024robust} on both image captioning and Visual QA. {\footnotesize \textit{{Note in VQA example, the Reds team is a different team from the Red Sox, and refers to the Cincinnati Reds team}}}. 
}
\label{fig:teaser}
\end{figure}
\noindent During the past decade, a plethora of defenses have been proposed to improve adversarial robustness, with adversarial training (AT) \cite{madry2017towards} emerging as a leading efficient yet time-consuming paradigm. Despite this progress, AT-trained models still exhibit a substantial gap between their performance on clean and perturbed inputs, particularly under strong or previously unseen attacks. In contrast, another line of work focuses on test-time defenses that achieve robustness at inference. Within this class, adversarial purification methods~\cite{song2018pixeldefend} attempt to denoise and remove adversarial perturbations using auxiliary generative models, while Randomized Smoothing (RS)~\cite{cohen2019certified} certifies robustness by averaging predictions over noise-perturbed inputs. 

More recently, test-time transformations~(TTT)~\cite{sheng2025r, Xing_2025_CVPR} have been proposed to improve classification without explicitly removing perturbations, instead adaptively transforming inputs dynamically at test time. Compared to adversarial training, these approaches defend against previously unseen threats in a plug-and-play manner without requiring model retraining. However, they often incur significant inference overhead, may rely on additional models, and sometimes struggle to scale to stronger attacks. They can also be computationally expensive, as they typically require optimizing additional tokens or parameters or generating augmented views for each sample. In this work, we present our novel and efficient test-time defense \textbf{\emph{\etthree}}, which demonstrates superior robustness achieved without such costly adaptation. The $\emph{\etthree}$ defense transforms the image at inference to boost robustness of CLIP while remaining lightweight and computationally efficient. Furthermore, our approach can be easily incorporated with existing methods to further improve their robustness significantly, incurring only a slight higher inference cost.

$\etthree$ is inspired by energy-based models (EBMs) that were previously linked to robustness \cite{grathwohl2019your} and aims to lower the energy of a given input. Requiring extremely short optimization through the vision encoder, the defense is fast and lightweight, having a very minor overhead over the Visual LVM inference time. \cref{fig:teaser} presents a few qualitative examples of the \etthree transformation improving the robustness of classification, image captioning, and VQA.
The $\etthree$ boosts robustness for a broad range of models and tasks. In our experiments, we demonstrate increased robustness for zero-shot classification using CLIP-based vision language models and for robust image classifiers on the ImageNet dataset. We also show a significant boost in robustness in downstream tasks for LVLMs, even with a single optimization step through the CLIP encoder alone, improving robustness without noticeably impacting inference speed. Finally, we show our approach is robust against adaptive attacks.

To further support the $\emph{\etthree}$ defense, we theoretically prove its effectiveness on binary classifiers. For a binary classifiers satisfying few assumptions,
we theoretically prove that an input transformed by $\etthree$ will be correctly classified. We later provide experimental evidence of a robust model, along with examples of classifiers that have been theoretically proven to be robust in several different settings \cite{vardi2022gradient,frei2023double,pmlr-v267-min25c,min2024can}, satisfying these assumptions.

\section{Related Work} \label{sec:relatedwork}

\minisection{Adversarial Attacks} Adversarial examples are slightly perturbed inputs to induce erroneous predictions or behaviors in deep neural networks. As various methods exist to create such perturbations~\cite{goodfellow2014explaining,madry2017towards,moosavi2016deepfool,andriushchenko2020square}, within the domain of LVLMs, a growing body of work has started examining their susceptibility to adversarial attacks~\cite{qi2024visual,carlini2023aligned,schlarmann2023adversarial}.

\minisection{Adversarial Defense} Among many defense strategies, adversarial training (AT)—which incorporates adversarial examples during training—remains the most effective empirical approach. AT has also been extended to vision language models (VLM)~\cite{mao2023understanding,schlarmann2024robust,wang2024pre, pulfer2024robustness, rocamora2025robustnessdomainsclipneeds} and diffusion models~\cite{rosaria2025adversarial}. In VLM, AT addresses vulnerabilities arising from multi-modal inputs by improving robustness to perturbations in either the visual or textual components. 
Beyond training-based defenses, offline prompt tuning methods learn a set of input tokens using a training dataset, optimizing them offline to improve accuracy or robustness without updating the backbone \cite{mao2023understanding, li2024apt, zhang2024adversarial}. Test-time defenses constitute another direction in which the model itself is not trained to be robust. Instead, robustness is achieved at inference time through additional models or mechanisms. Within this broader class, Randomized Smoothing (RS) comprises a family of methods that certify robustness by averaging predictions over noise-perturbed inputs \cite{lecuyer2019certified,cohen2019certified}, with various extensions proposed in recent works \cite{alfarra2022data, lyu2024adaptive}. Another approach is adversarial purification, which uses an auxiliary generative model to denoise the test image, removing adversarial perturbations while remaining close to the original image~\cite{ song2018pixeldefend, hillstochastic, yoon2021adversarial, nie2022diffusion}. In another line of work, the input is adaptively transformed at the test time~\cite{Perez_2021_ICCV, 10943954, alfarra2022combating,wu2021attacking,Xing_2025_CVPR}, aiming for correct classification without explicitly removing adversarial perturbations or constraining a similarity to the original image, as in~\cite{10943954,Xing_2025_CVPR}.

\minisection{Test-Time Transformations for VL tasks} Test-time adaptation methods~\cite{sun2020test, wang2021tent, gandelsman2022test, sun2025learning, durasov2025it} aims to adapt models to the test data at inference time to further improve their performance. For Vision-Language tasks,
these methods usually aim to improve the robustness or generalization
by altering the prompt, adjusting the embedding, or leveraging test-time augmentations. \cite{shu2022test} introduces Test-Time Prompt Tuning (TPT), which optimizes a text prompt
by minimizing prediction entropy across augmented views.
While training-free and effective under distribution shift,
its multiple augmentations cause significant inference overhead and reduced clean accuracy.
C-TPT~\cite{yoon2024ctpt} mitigates this issue by introducing a calibration term that maximizes text-feature dispersion.
R-TPT \cite{sheng2025r}
targets adversarial robustness, optimizing the prompt against adversarial perturbations.
\cite{zanella2024test} proposes MTA, which aggregates the embeddings of multiple augmentations using a robust mean-shift procedure.
However, all of these methods introduce substantial inference overhead.
\cite{Xing_2025_CVPR} introduces TTC, which improves zero-shot adversarial robustness by applying a gradient-based perturbation that maximizes the distance between the adversarial input and its clean embedding in CLIP's feature space, yet offers limited robustness gains.

\minisection{Adversarially Robust Models and EBM} 
Energy-based models (EBMs)~\cite{lecun2006tutorial} are generative models that learn an energy function $E(\bx)$ which assigns low energy values to inputs $\bx$ in the data distribution and high energy values to other inputs. Intuitively, lower energy corresponds to samples that align well with the learned data distribution (on-manifold), while higher energy indicates atypical or out-of-distribution inputs (off-manifold). This conceptual perspective is particularly relevant for adversarial examples,  which are known to deviate from the natural data manifold and typically go off the manifold of natural data~\cite{stutz2019disentangling,shamir2021dimpled,mirza2025understanding}. 

Bridging this generative perspective with discriminative tasks, several studies ~\cite{4270060, xie2016theory, du2019implicit} explore the link between generative models such as EBMs and discriminative models. This connection is formalized in the Joint Energy-based Model (JEM)~\cite{grathwohl2019your}, which reformulates the softmax classifier within an energy-based framework. This formulation allows the classifier's output as energy can enable OOD detection~\cite{liu2020energy}. The connection between learning an energy objective and robustness has been demonstrated in both ways, showing that the addition of EBM objectives to training increases robustness~\cite{grathwohl2019your}, and performing adversarial training has been shown to implicitly learn an EBM~\cite{yin2022learning,zhu2021towards}. Incorporating explicit energy-based objectives has also been shown to further enhance both robustness and generative performance~\cite{mirza2024shedding,10485467,jiang2025your,yin2025joint}. The generative aspects of robust classifiers have also been explored in other studies~\cite{wang2022aunified,yang2023towards,yang2023mebm,rouhsedaghat2022magic}.

\minisection{Theoretically Proven Robustness} There has been a great effort in the theoretical research to characterize robustness guarantees and measure the robustness of a trained classifier. Starting in \cite{shamir2019simple}, many have tried to theoretically show adversarial vulnerability. It had been shown for random networks with different natural architectures in \cite{daniely2020most, bartlett2021adversarial,bubeck2019adversarial,bubeck2021single,montanari2022adversarial}. Later, in \cite{vardi2022gradient,frei2023double} it had been shown that gradient-based training converge to a non-robust networks while robust network do exist, while even offering a concrete robust classifier as an example. In \cite{pmlr-v267-min25c} and \cite{min2024can} the author discuss a different activation function, polynomial ReLU, and prove that training such model converges to a robust classifier.

\section{Method}\label{sec:method}

We introduce an \emph{energy-based test-time transformation}, \etthree, grounded in the energy-based modeling (EBM) perspective of discriminative classifiers. Unlike other test-time defense approaches that train an explicit generative EBM~\cite{hillstochastic} or auxiliary diffusion/score models~\cite{song2018pixeldefend}, \etthree requires no additional model training. It operates directly on a pre-trained classifier or visual encoder $f_\theta$ that we aim to defend. This is possible because a standard softmax classifier can itself be viewed through the EBM lens, by interpreting its logits as energies~\cite{grathwohl2019your}. At its core, \etthree applies a lightweight transformation that simply decreases this logit-derived energy for a given input. Our focus on robust models stems from the tight connection between energy-based and robustness objectives. For instance, adding an explicit EBM term during training can enhance adversarial robustness~\cite{grathwohl2019your}. Conversely, adversarial training has been shown to implicitly induce an EBM with smoother local energy landscapes~\cite{yin2022learning,zhu2021towards,mirza2024shedding}. For clean inputs, \etthree either preserves the original prediction or can even increase the model's confidence. For adversarial inputs, it guides the sample back toward the correct classification, effectively enhancing the model's robustness.

\subsection{Energy Definition}
\noindent Consider a set of $d$-dimensional labeled images $X = \{ (\bx,y) \sim \mathcal{D} |~\bx \in \reals^{d}$ and $ y \in \{1,..,K\} \}$, let $f_\theta(\bx):\reals^{d} \rightarrow \reals^{K}$ be a $K$-class classifier or an encoder parameterized by $\theta$ trained on a dataset $X$. For a sample $\bx \in \reals^d$, we denote by $f_\theta(\bx)_k$ the $k$-th coordinate of the output logits vector $f_\theta(\bx) \in \reals^K$.
We define the energy at $\bx$ as:
\begin{equation}
E(\bx) = -\log \Bigl(\sum_{k=1}^{K} \exp\bigl(f_\theta(\bx)_k\bigr)\Bigr),
\label{eq:energy_def}
\end{equation}
which is the negative \texttt{LogSumExp} of the output logits.

\begin{figure}[tb]
    \centering
    \begin{overpic}[width=\linewidth]{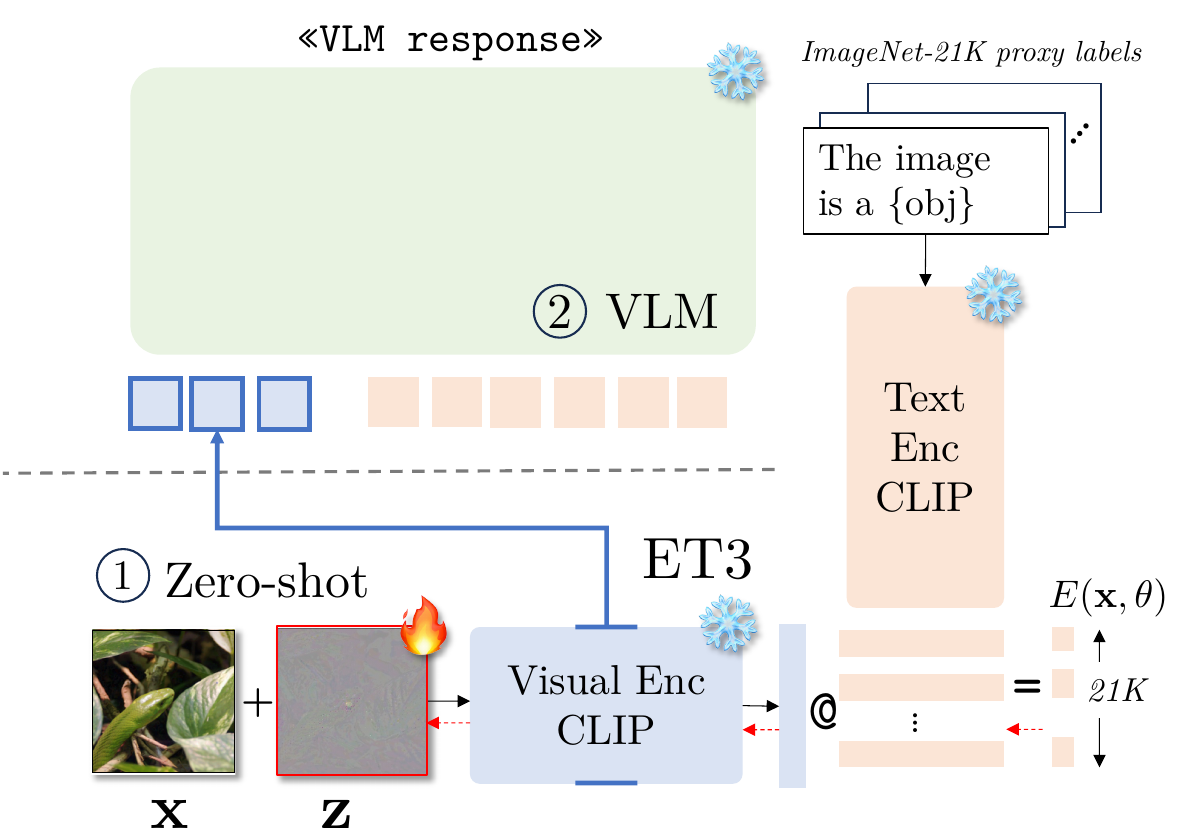}
    \end{overpic}
    \caption{\pointone~ET3 transforms the natural image $\bx$ adding a small perturbation $\bz$ optimized to lower the energy wrt to ImageNet-$21k$ proxy classes and concepts. This allows robust zero-shot classification; \pointtwo~the transformed image transfers and protects Large VLM, thereby increasing their robustness. The VLM is \emph{not} used in the optimization, and the optimized image simply transfers to VLM by using the internal representation of the visual encoder.}
    \label{fig:fig3}
\end{figure}

\subsection{The \etthree Defense}
Inspired by the connection between the energy and the perception of the data distribution by $f_
\theta$, we present a test-time defense based on energy minimization.
Given a test image $\bx$ and a defense radius $\epsilon$, our goal is to obtain a modified image $\bxp = \bx + \bz$ by iteratively minimizing $E(\cdot,\theta)$. Concretely, we solve the constraint minimization
$\bxp = \argmin_{\bxa \in \mathcal{B}_\epsilon(\bx)} E(\bxa)$,
where $\mathcal{B}_\epsilon(\bx) = \{\bxa : \norm{\bxa - \bx} \leq \epsilon\}$, using a gradient-based multi-step optimization for $T$ steps. Formally, in each step $t \in [1,...,T]$ we calculate:
\begin{equation}
\bx^{(t)} = \Pi_{\mathcal{B}_\epsilon(\bx)}\Bigl(\bx^{(t-1)} - \alpha \nabla_{\bx} E\bigl(\bx^{(t-1)}\bigr)\Bigr),
\label{eq:purification_update}
\end{equation}
starting from $\bx^{(0)} = \bx$. The step size $\alpha$ and number of iterations $T$ are hyperparameters, yet we found our method is fast and works for fewer steps.
The projection $\Pi_{\mathcal{B}_\epsilon(\bx)}(\cdot)$ enforces the defense perturbation in a $\ell_2$ ball.

\subsection{Extending \etthree to Vision-Language Tasks}
We extend \etthree to zero-shot classification and vision-language tasks. The core principle remains the same: minimize the energy function defined in \cref{eq:energy_def} by refining the input image at test time. \cref{fig:fig3} depicts how we apply \etthree to both zero-shot classification and VLMs.

\minisection{Zero-Shot Classification with CLIP} For zero-shot classification, we apply \etthree to CLIP models, where classification is performed by computing similarity scores between image and text embeddings, which serve as logits for \etthree. To compute energy, we consider two distinct label configurations: a refined subset of labels or a vast set of labels, such as the full ImageNet‑21K label set~\cite{ridnik2021imagenet21k}, yielding slightly improved robustness. We primarily adopt the latter, 
with further details on label-set selection provided in App.~\ref{app:label-set_guide}.

\minisection{Large Vision-Language Models (LVLMs)}
For multimodal tasks such as image captioning and visual question answering, we apply \etthree to the LLaVA model. Similar to the zero-shot setting, \etthree optimizes the visual input at test time by minimizing its energy through the CLIP vision encoder. After refinement, the visual embeddings from the vision encoder are passed to LLaVA's projection layer, and the subsequent language generation process remains unchanged. \etthree enhances robustness specifically in the vision modality, even without modifying or fine-tuning the vision encoder to improve its robustness. 
To further support \etthree, next we prove that the $\etthree$ defense method boosts robustness for a binary classifier, where it provably transforms a clean or adversarial sample into a correctly classified input.

\begin{figure*}[t]
    \begin{overpic}[width=\textwidth]{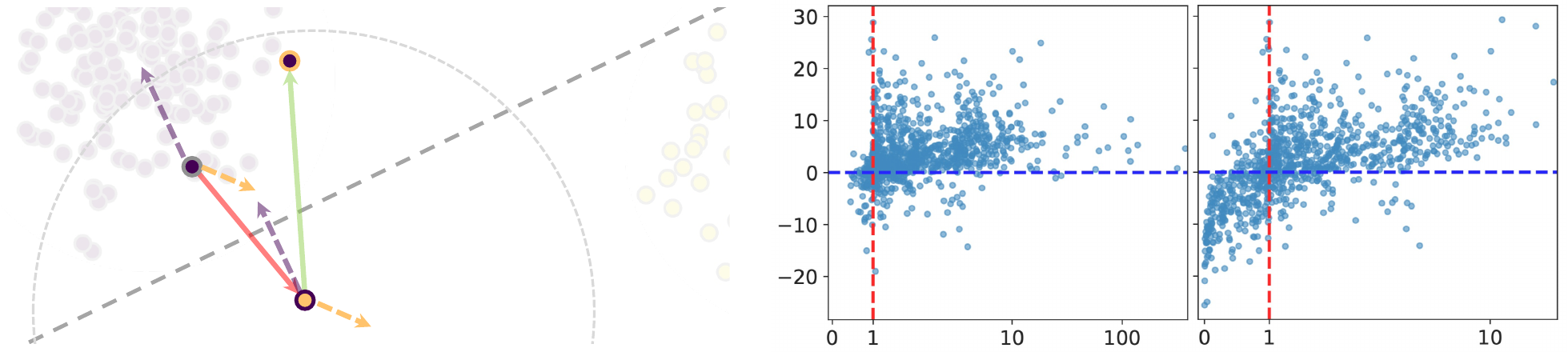}
    \put(3,2.5){\rotatebox{30}{{\small Class ${y_t}$}}}
    \put(4.25,-.5){\rotatebox{30}{{\small Class $\hat{y_t}$}}}
    \put(11,9.5){$\bx$}
    \put(18,1){$\bxa$}
    \put(20,19){$\tilde{\bx}$}
    \put(11,17){\resizebox{20pt}{!}{\textcolor{pink}{$e_1 \bg_1$}}}
    \put(13,12){\resizebox{20pt}{!}{\textcolor{yellow}{$e_0 \bg_0$}}}
    \put(58,23){{\small Natural Images}}
    \put(80,23){{\small Adversarial Images}}
    \put(58,-1.25){{\small The gradient ratio $C$}}
    \put(80,-1.25){{\small The gradient ratio $C$}}
    \put(47.5,2){\rotatebox{90}{{\small Logit margin after our \etthree}}}
    \end{overpic}
    \caption{\emph{(left)} 
    The \etthree defense transformation for adversarial examples. Assuming local linearity of the model in the defense neighborhood $\mathcal{B}_\epsilon(\bx)$, and a large enough ratio $C$ between the norms of the gradients of the energy through each class logit, $e_0 \bg_0$ and $e_1 \bg_1$. The adversarial attack, determined to reduce the ground truth logit, follows the negative direction of the larger gradient ($-\bg_1$), while our transformation follows its positive direction ($\bg_1$), increasing the ground truth logit and pulling the adversarial point back to its ground truth region. Both might also increase the other logit, corresponding to the smaller gradient $\bg_0$, that may introduce some smaller deviation.
    \emph{(right)} Scatter plot of the ratio between the gradients norms $C$ and the logit margin at the transform image $\bxp$ on ImageNet robust classifier for 1000 randomly sampled images from ImageNet. For most samples for which $C>1$, we can see that the purified image is correctly classified (logit margin $> 0$). One can see a correlation between the norms ratio and the logits difference of the transformed image.
    }
    \label{fig:figure2}
\end{figure*}

\section{\etthree defense provably boosts Robust Nets}\label{sec:provable_defense_robust}
In this section, we theoretically prove that for a binary classifier $f_\theta(\bx): \reals^d \rightarrow \reals^2$, under some assumptions, minimizing the energy with an $\epsilon$ budget will transform a given sample to be classified as its original ground truth class for both natural and adversarial samples.
First, for simplicity, we assume local linearity in an $\epsilon$-ball around the sample. We note that approximate local linearity has been proved and demonstrated before in the context of adversarial attacks and robust classifier~\cite{melamed2024malt}. Second, we assume that the gradient of the energy derived from the ground truth label is larger than the one derived from the opposite label. This assumption is particularly intriguing and arises from the vicinity of the adversarial sample and its corresponding source sample. Details about these assumptions in Remark~\ref{rem:proof_assump}.

\minisection{Data and Model} We consider $d$ dimensional data inputs $\bx \in \reals^d$, with unknown binary ground truth labels $y_t \in \{-1,1\}$. The incorrect label is $\hat{y_t} = - y_t$. Our model is a binary classifier with two output logits denoted by $f_\theta(\bx): \reals^d \rightarrow \reals^2$. We use a two-logit output to adapt to the multiclass classifier regime, for which each possible class relates to a logit. Similarly, for a given input $\bx \in \reals^d$ and its ground truth label $y_t \in \{-1,1\}$, 
if $f_\theta(\bx)_{y_t} > f_\theta(\bx)_{\hat{y_t}}$, then $\bx$ is correctly classified.

\begin{theorem}\label{thm:defense}
    Let $\bx \in \reals^d$ be a data sample, and $y_t$ be its ground truth label. Let $f_\theta : \reals^d \rightarrow \reals^2$ be a binary classifier such that it's locally linear in $\mathcal{B}_\epsilon(\bx)$. Denote $r_x = f_\theta(\bx)_{y_t} - f_\theta(\bx)_{\hat{y_t}}$, $\bg_i = \nabla_{\bx} f_\theta(\bx)_i$ and $e_i = \softmax\left(f\left(\bx\right)\right)_{i}$.
    Let $\epsilon>0$ a defense budget such that $\epsilon > \frac{-2 r_x}{\norm{\bg_{y_t}}}$. 
    Then, if 
  \[
C\norm{e_{\hat{y_t}} \bg_{\hat{y_t}}} < \norm{ e_{y_t} \bg_{y_t}}
\]  
for 
\[
C > \max \left\{\frac{ \exp(|r_x|) \epsilon \norm{\bg_{y_t}}}{\epsilon \norm{\bg_{y_t}} + 2 r_x}, 1 \right\}~,
\]
the $\etthree$ defense transformation $\bz$, parametrized by $T=1$ and $\alpha = \frac{\epsilon}{e_{y_t}  \left(1+ \frac{1}{C}\right) \norm{\bg_{y_t}}}$, will satisfy $\norm{\bz} \leq \epsilon$ and \[f_\theta(\bx+\bz)_{y_t} > f_\theta(\bx+\bz)_{\hat{y_t}}~.\]
\end{theorem}

\minisection{Proof Idea} Formally, given a data point $\bx \in \reals^d$, we denote $\bg_i$, the gradient of the network's $i$-th logit w.r.t. the input and $e_i$ the gradient of the energy w.r.t. the $i$-th logit.
Then, the gradient of \texttt{LogSumExp} w.r.t. $\bx$ is
\begin{align*}
\nabla_{\bx} E(\bx) =& - \softmax\left(f_\theta(\bx)\right)^\top \nabla_{\bx} f_\theta(\bx) \\
=& - e_{-1} \bg_{-1} - e_1 \bg_1.
\end{align*}

Thus, the purification step is of the form $\bz = \alpha  \left( e_{-1} \bg_{-1} + e_1 \bg_1 \right)$. To see the effect of the transformation on each logit, we use the local linearity and look at the classifier as two locally linear functions $f_{-1},f_1:\reals^d\rightarrow\reals$, one for each logit. We observe that for $i \in \{-1,1\}$
\[f_i(\bx + \bz) = f_i(\bx) + \bz^\top \bg_i.\]
We complete the proof by using the ratio of the gradients' norms, $C$, to show that multiplying by the larger norm gradient, $\bz^\top \bg_{y_t}$, will create a larger effect, implying correct classification. Formally, we note that the energy function allows a relaxed assumption on the ratio of the gradients as:
\[
 \exp\big(f_\theta(\bx)_{\hat{y_t}} - f_\theta(\bx)_{y_t}\big) C \norm{\bg_{\hat{y_t}}} < \norm{\bg_{y_t}}~,
\]
that is adaptive to the logits of $f_\theta(\bx)$. When $\bx$ is classified correctly, we will have $f_\theta(\bx)_{\hat{y_t}} - f_\theta(\bx)_{y_t} \leq 0$ and thus $\exp\left(f_\theta(\bx)_{\hat{y_t}} - f_\theta(\bx)_{y_t}\right) \leq 1$ allowing a smaller ratio between the gradients to preserve correct classification. Otherwise, the assumed ratio will be weighted by the difference between the logits. The full proof provided in App.~\ref{app:proof}.

\begin{remark}[The Local Linearity and Gradient Norm Ratio Assumptions]\label{rem:proof_assump}

\minisection{In Practice} \emph{ In the right part of \cref{fig:figure2} we present on the $x$-axis, for each natural image of ImageNet and its respective adversarial example, the norm ratio of the gradients $e_{y_t} \bg_{y_t}$ and $e_{y_{\text{adv}}} \bg_{y_{\text{adv}}}$, for which we used the APGD attack \cite{croce2020reliable} and a robust classifier \cite{engstrom2019adversarial}. All inputs are then transformed using our method $\etthree$, and the difference between the resulting output logits is presented in the $y$-axis. One can see that many natural and adversarial example has high ratio values, i.e., $C \gg 1$. Moreover, one can see that transforming samples with a higher ratio $C$ results in a larger logit difference for the transformed image, indicating successful transformation. For the locally linearity assumption, previous works showed that robust networks are experimentally approximately locally linear \cite{melamed2024malt}.}

\minisection{In Theory} \emph{When proving robustness of a model, a helpful property is have an approximate locally linear model, allowing a more immediate upper bound on the effectiveness of a worst-case perturbation. To this end, and to also allow correct classification, one idea for a robust network construction will be to assign higher inner products between the data points and the local gradients corresponding to the ground truth label, and smaller inner product with the others, or assigning a sufficiently large bias. In previous works, \cite{vardi2022gradient, frei2023double} construct a robust network enforcing both local linearity with a large radius and a larger gradient of the ground truth label in the settings of almost orthogonal data. In \cite{pmlr-v267-min25c} and \cite{min2024can}, the network trained using polynomial ReLU is approximately locally linear and with a similar larger ground truth gradient. ~\cref{sec:theoryapp} gives more details on the robust construction from \cite{frei2023double} for which both assumptions hold, adapting the classic binary classifier to our two-logit regime.}
\end{remark}
We note that as there are no assumptions of $\bx$, the \etthree defense will be successful for both clean sampled and adversarial data input, yet the assumption are milder on a clean data sample. First, since $r_x > 0$ for a clean sample, the assumption on $\epsilon$ will hold for any $\epsilon>0$. Second, $\exp\left(f_\theta(\bx)_{\hat{y_t}} - f_\theta(\bx)_{y_t}\right)$ will be smaller than $1$, inducing a relaxed lower bound for the ratio $C$. 

This theorem extends to the case of multi-step transformation within the same budget. Since the classifier is locally linear, each step to minimize the energy, will still increase the ground truth logit and the gap between it and the other logit. In other words, the theorem shows that 
\[
f_\theta(\bx+\bz)_{y_t} - f_\theta(\bx+\bz)_{\hat{y_t}} >  f_\theta(\bx)_{y_t} - f_\theta(\bx)_{\hat{y_t}}~
\]
for a transformation $\bz$ of any size. Therefore, the new input $\bx+\bz$ will preserve the assumptions on $C$ and $\epsilon$, allowing a multi-step attack with the same budget $\epsilon$ to take multiple small gradient steps till accumulated to a sufficient size.

\subsection{Theoretical Application}\label{sec:theoryapp}
We look at the construction for a robust two-layer ReLU network, stated in \cite{frei2023double}. 

\minisection{Data} We use the following data distribution $\mathcal{D}_{\text{clusters}}$ on $\reals^d \times \{-1,1\}$ in which we have $k$ clusters with means $\mu^{(1)},...,\mu^{(k)} \in \reals^d$ and covariance $\sigma^2 I_d$, where examples in the $j$-th cluster are labeled $y^{(j)} \in \{-1,1\}$.

\minisection{Model} We adapt the two-layer binary classification network into our multi-class network. For a given binary classifier $f_\btheta(\bx) = \sum_{j\in J} v_j \sigma(\bw_j^T \bx + b_j)$,
we define a two-logits output classifier as following. We define $J_+ = \{j : v_j \geq 0 \}$ and $J_- = \{j : v_j < 0 \}$, and denote for $l \in \{-1,1\}$
\[
\of_\btheta(\bx)_{l} = l \sum_{j\in J_i} v_j \sigma(\bw_j^T \bx + b_j)~,
\]
We note that for any $\bx$, $ f_\btheta(\bx) \equiv \of_\btheta(\bx)_{1} - \of_\btheta(\bx)_{-1}$. 

In \cite{frei2023double}, the authors present a robust network, proving its robustness to a large $O(\sqrt{d})$ adversarial perturbation.

For the robust network of width $k$, they take for any $j \in [k]$, $v_j = y^{(j)}$ and $\bw_j = \frac{4\mu^{(j)}}{d}$ and $b_j = -2$ (for wider networks, i.e. $J>k$, we set the extra weights to be zero).

Let $(\bx,y) \in \mathcal{D}_{\text{clusters}}$, meaning that $\bx = \mu^{(q)} + \xi$ for some $\xi \sim \mathcal{N}(0, \sigma^2 I_d)$ and $y = y^{(q)}$. For the gradient ratio and local linearity assumptions, we look at the analysis for $\bw_j^T \bx + b_j $ for different $j \in [k]$, and get that for $j=q$ $\bw_q^T \bx + b_q > 1$, while for $j \neq q$ , $\bw_j^T \bx + b_q < -1 $. Thus, denoting $\onefunc_{\{A\}}$ as the indicator function for an event $A$, we have with high probability (w.h.p.)
\[ 
\bg_{y_t} = \frac{\partial \of_\theta(\bx)_{y_t}}{\partial \bx} = \sum_{j\in J_+}  \bw_j \onefunc_{\{\bw_j^T \bx + b_j \geq 0\}} = y\bw_q~,
\]
and $\bg_{\hat{y_t}} =0$, satisfying the ratio assumption. In Theorem 4.1 the authors show that w.h.p. over $\bx$, the network is locally linear in $\mathcal{B}_\epsilon(\bx)$ for the large $\epsilon \leq \frac{\sqrt{d}}{8}$.
Applying the $\etthree$ defense, we can see that the transformation will be $\bz = \alpha \bw_q$ for some $\alpha>0$, which will ensure correct classification for the transformed input.

\section{Experiments}

\colorlet{cianoChiaro}{LightCyan3!20!white}
\colorlet{cianoVeryChiaro}{LightCyan4!80!white}

\newcolumntype{J}{>{\columncolor{cianoChiaro}}c}
\begin{table*}[t!]
\caption{\textbf{Zero-shot robustness of \etthree across 14 benchmark datasets 
in the defense-unaware setting.} Comparison of clean and robust accuracy for 
baseline models versus same models augmented with \etthree. Robustness is evaluated 
against Auto-Attack (AA) at $\epsilon_a = 4/255$.}

    \centering
    \scriptsize
    \setlength{\tabcolsep}{1.8pt}
    \resizebox{\textwidth}{!}{
    \begin{tabular}{cc|>{\columncolor{lightgray}}Cc
    >{\columncolor{lightgray}}Cc
    >{\columncolor{lightgray}}Cc
    >{\columncolor{lightgray}}Cc
    >{\columncolor{lightgray}}Cc
    >{\columncolor{lightgray}}Cc
    >{\columncolor{lightgray}}Cc|cc}
    \toprule
    Model & Method & 
    \rotatebox[origin=c]{60}{ImageNet} & 
    \rotatebox[origin=c]{60}{CalTech} & 
    \rotatebox[origin=c]{60}{Cars} & 
    \rotatebox[origin=c]{60}{CIFAR10} & 
    \rotatebox[origin=c]{60}{CIFAR100} & 
    \rotatebox[origin=c]{60}{DTD} & 
    \rotatebox[origin=c]{60}{EuroSAT} & 
    \rotatebox[origin=c]{60}{FGVC} & 
    \rotatebox[origin=c]{60}{Flowers} & 
    \rotatebox[origin=c]{60}{ImageNet-R} & 
    \rotatebox[origin=c]{60}{ImageNet-S} & 
    \rotatebox[origin=c]{60}{PCAM} & 
    \rotatebox[origin=c]{60}{OxfordPets} &  
    \rotatebox[origin=c]{60}{STL-10} & 
    \rotatebox[origin=c]{60}{Avg.} & 
    \rotatebox[origin=c]{60}{\textbf{Improv.}}\\
    \midrule
    \multirow{4}{*}{\shortstack{ViT-L/14 \\(TeCoA)\\$\epsilon_t=4/255$}}
    & Base (Clean)& 74.91 & 78.36 & 37.83 & 79.61 & 50.26 & 38.03 & 22.48 & 11.76 & 38.41 & 74.35 & 54.22 & 49.95 & 76.07 & 93.44 & 55.69  & 
    \\
    & \textbf{+ ET3} (Clean) & 75.20 & 78.16 & 37.21 & 81.27 & 49.97 & 38.09 & 23.31 & 11.61 & 39.31 & 76.66 & 55.22 & 50.03 & 76.18 & 93.08 & 56.09 & \smallblue{0.4}\\
    \arrayrulecolor{cianoVeryChiaro}\cmidrule(lr){2-18}
    
    & Base (Robust) & 44.50 & 60.90 & 8.50 & 37.10 & 21.50 & 16.50 & 6.40 & 2.20 & 12.60 & 41.90 & 32.80 & 45.70 & 55.00 & 74.30 & 32.85  
    \\
    & \textbf{+ ET3} (Robust) & 53.00 & 65.10 & 13.50 & 56.40 & 34.50 & 22.70 & 15.80 & 4.40 & 22.50 & 51.90 & 40.00 & 51.50 & 60.90 & 80.80 & 40.93 & {\smallblue{8.08}}\\
    \arrayrulecolor{LightCyan4}\midrule
     \multirow{4}{*}{\shortstack{ViT-L/14 \\(FARE)\\$\epsilon_t=4/255$}}
    & Base (Clean) & 70.78 & 84.70 & 63.84 & 77.67 & 56.53 & 43.83 & 18.28 & 21.96 & 58.07 & 80.24 & 56.74 & 50.02 & 87.14 & 96.04 & 61.85    
    \\
    & \textbf{+ ET3} (Clean)& 70.97 & 84.73 & 62.96 & 80.58 & 55.91 & 43.94 & 18.93 & 21.96 & 58.86 & 82.49 & 57.25 & 50.02 & 86.59 & 96.17 & 62.24 & {\smallblue{0.39}} \\
    \arrayrulecolor{cianoVeryChiaro}\cmidrule(lr){2-18}
    
    & Base (Robust)& 34.80 & 64.20 & 12.70 & 34.80 & 20.20 & 17.50 & 11.10 & 3.00 & 12.20 & 40.50 & 30.60 & 52.30 & 50.60 & 74.30 & 32.77   
    \\
    & \textbf{+ ET3} (Robust) & 41.20 & 69.40 & 18.40 & 53.50 & 33.10 & 26.80 & 14.70 & 8.30 & 24.20 & 50.80 & 37.40 & 52.30 & 58.50 & 79.80 & 40.60 & {\smallblue{7.83}} \\
    \arrayrulecolor{LightCyan4}\midrule
    \multirow{4}{*}{\shortstack{ViT-L/14 \\(TeCoA)\\$\epsilon_t=2/255$}}
    & Base (Clean) & 80.11 & 80.67 & 50.08 & 87.53 & 60.69 & 44.36 & 26.06 & 14.04 & 51.80 & 80.12 & 58.43 & 49.89 & 80.02 & 96.08 & 61.42  
    \\
    & \textbf{+ ET3} (Clean) & 79.82 & 79.62 & 46.33 & 86.34 & 59.90 & 44.15 & 31.52 & 13.35 & 50.02 & 82.05 & 58.99 & 49.99 & 80.43 & 94.97 & 61.25 & {\smallred{-0.17}} \\
    \arrayrulecolor{cianoVeryChiaro}\cmidrule(lr){2-18}
    
    & Base (Robust) & 37.00 & 57.40 & 6.40 & 31.00 & 17.90 & 14.70 & 7.80 & 1.00 & 9.60 & 36.60 & 30.90 & 17.40 & 50.40 & 69.10 & 27.66  
    \\
    & \textbf{+ ET3} (Robust) & 44.40 & 63.10 & 14.10 & 49.20 & 32.20 & 23.40 & 24.00 & 4.70 & 20.50 & 46.00 & 38.40 & 47.10 & 58.40 & 75.40 & 38.64 & {\smallblue{10.98}}\\
    \midrule
     \multirow{4}{*}{\shortstack{ViT-L/14 \\(FARE)\\$\epsilon_t=2/255$}}
    & Base (Clean) & 74.48 & 84.77 & 70.53 & 89.52 & 69.13 & 50.05 & 25.39 & 26.70 & 70.60 & 85.52 & 59.72 & 50.01 & 91.06 & 98.47 & 67.57    
    \\
    & \textbf{+ ET3} (Clean) & 74.07 & 84.60 & 68.31 & 89.64 & 66.64 & 48.03 & 32.98 & 24.90 & 69.47 & 86.94 & 59.86 & 50.03 & 90.38 & 98.00 & 67.42 & {\smallred{-0.15} }\\
    \arrayrulecolor{cianoVeryChiaro}\cmidrule(lr){2-18}
    
      & Base (Robust) & 17.80 & 46.40 & 5.00 & 25.70 & 14.20 & 11.60 & 0.40 & 0.90 & 7.10 & 25.60 & 22.10 & 19.10 & 28.10 & 61.50 & 20.39    
      \\
    & \textbf{+ ET3} (Robust) & 25.20 & 56.40 & 12.20 & 43.90 & 28.70 & 21.90 & 25.30 & 6.50 & 16.00 & 35.10 & 29.40 & 35.90 & 38.90 & 69.10 & 31.75  & {\smallblue{11.36}}\\
    \arrayrulecolor{black}\bottomrule
    \end{tabular}
    } %
    \label{tab:tab1}
\end{table*}

\newcolumntype{G}{>{\columncolor{lightgray}}c}

\begin{table*}[t]

\caption{\textbf{Robustness of \etthree on fine-grained classification in the 
defense-unaware setting.} Comparison against test-time adaptation techniques 
across eight fine-grained datasets. All defenses are applied to a TeCoA 
pre-trained CLIP-ViT-B/32 model and evaluated against PGD-100 ($\epsilon = 4/255$). 
\etthree consistently improves robustness, both as a standalone method and when 
combined with other defenses.}

  \centering
  \setlength{\tabcolsep}{3.8pt}
  \resizebox{\linewidth}{!}{%
  \begin{tabular}{l
  G G c c G G c c G G c c G G c c  G G c}
    \toprule
    \multirow{2}{*}{Method} 
    & \multicolumn{2}{c}{\textbf{Caltech101}} 
    & \multicolumn{2}{c}{\textbf{Pets}} 
    & \multicolumn{2}{c}{\textbf{Cars}} 
    & \multicolumn{2}{c}{\textbf{Flower102}} 
    & \multicolumn{2}{c}{\textbf{Aircraft}} 
    & \multicolumn{2}{c}{\textbf{DTD}} 
    & \multicolumn{2}{c}{\textbf{EuroSAT}} 
    & \multicolumn{2}{c}{\textbf{UCF101}} 
    & \multicolumn{2}{c}{\textbf{Avg.}} 
    & \multicolumn{1}{c}{\textbf{Improv.}} \\
    
    & Acc. & Rob. & Acc. & Rob. & Acc. & Rob. & Acc. & Rob. 
    & Acc. & Rob. & Acc. & Rob. & Acc. & Rob. & Acc. & Rob. & Acc. & Rob. & Rob.\\
    \midrule
CLIP (Robust)
& 78.82 & 43.45 & 66.88 & 15.94 & 10.20 & 0.99 & 30.82 & 9.09  
& 6.60 & 0.45 & 24.53 & 10.70 & 14.53 & 10.78 & 34.58 & 6.71  & 33.37 & 12.26 &\\
\arrayrulecolor{LightCyan4}\midrule
\multicolumn{20}{c}{\textit{Lightweight Defense using Image Transformation (Vision Input Only)}}\\[-1pt]
\arrayrulecolor{LightCyan4}\midrule
TTC \cite{Xing_2025_CVPR}
& 71.26 & 44.90 & 65.00 & 21.10 & 10.62  & 1.39 & 26.02 & 8.93 
& 6.93 & 0.69 &  23.99 & 11.38 & 20.84 & \underline{\textbf{12.34}} & 32.91 & 10.45 & 32.20 & 13.90 &  \smallblue{1.64}\\

\textbf{\etthree} (ours)
& 79.07 & \underline{50.59} & 66.86 & \underline{27.15} & 10.32  & \underline{2.04} & 28.91 & \underline{12.10} 
& 5.28 & \underline{0.87} & 24.00 & \underline{13.18} & 13.37 & 11.17 & 37.35 & \underline{16.02} & 33.15 & \underline{16.51} & \smallblue{4.25}  \\
\arrayrulecolor{LightCyan4}\midrule
\multicolumn{20}{c}{\textit{Defenses using Multiple Augmentations 
(Vision Input Only)}}\\[-1pt]
\arrayrulecolor{LightCyan4}\midrule
Ensemble 
& 73.02 & 55.66 & 59.96 & 38.35 & 5.55  & 2.80 & 26.35 & 16.12
& 4.29 & 1.89 & 23.82 & 15.96 & 12.51 & 11.01 & 26.35 & 14.12 & 28.98 & 19.4 & \smallblue{7.14}\\

MTA \cite{zanella2024test} *
& 79.70 & 55.70 & 66.20 & 31.20 & 9.00  & 2.50 & 29.10 & 14.00 
& 6.50 & 1.60 & 24.40 & 13.50 & 13.30 & \underline{11.20} & 34.60 & 12.50 & 32.90 & 17.80 & \smallblue{5.54}\\

Ensemble + \textbf{\etthree} & 76.23 & \underline{61.38} & 62.03 & \underline{43.17} & 5.77  & \underline{3.03} & 26.07 & \underline{18.11}
& 3.93 & \underline{2.22} & 22.64 & \underline{17.55} & 11.88  & 11.14 & 32.62 & \underline{21.44} & 30.14 & \underline{22.26} & \smallblue{10}\\

\arrayrulecolor{LightCyan4}\midrule
\multicolumn{20}{c}{\textit{Defenses using Multiple Augmentations + Prompt Tuning (Vision + Text Input)}}\\[-1pt]
\arrayrulecolor{LightCyan4}\midrule

TPT \cite{shu2022test} *
& 79.30 & 52.70 & 65.20 & 27.40 & 9.60  & 2.00 & 27.90 & 12.30 
& {6.70} & 1.70 & 25.50 & 14.60 & 12.20 & 11.20 & 34.90 & 10.20 & 32.70 & 16.50 & \smallblue{4.24}\\

C-TPT \cite{yoon2024ctpt} *
& {79.80} & 47.30 & 66.10 & 19.50 & {10.60} & 1.30 & 29.40 & 10.70 
& 6.40 & 0.70 & {26.20} & 12.40 & 13.00 & 11.10 & {36.40} & 8.10  & {33.50} & 13.90 & \smallblue{1.64} \\

R-TPT \cite{sheng2025r} 
& 75.58 & 61.01 & 63.75 & 41.43 & 5.57  & 2.79 & 26.29 & 16.57 
& 5.73 & 2.46 & 25.59 & 18.09 & 11.46 & 11.30 & 31.48 & 17.63 & 30.68 & 21.41 & \smallblue{9.15}\\

R-TPT\textbf{+ \etthree}
& 79.27 & \textbf{65.07} & 63.86 & \textbf{45.71} & 6.21  & \textbf{3.27} & 26.51 & \textbf{18.72} 
& 5.85 & \textbf{3.51} & 26.42 & \textbf{19.92} & 11.37 & 11.16 & 36.93 & \textbf{24.72} & 32.05 & \textbf{23.88} &\smallblue{11.62} \\
\arrayrulecolor{black}\midrule
\end{tabular}
}
\label{tab:tab2}
\end{table*}

\label{sec:experiments}
We evaluate \etthree in two threat models and in three settings: zero-shot 
classification with CLIP, vision-language tasks with LVLMs, and classification with image classifiers. The threat models differ in the attacker's knowledge of the test-time defense: a defense-unaware setting, where we evaluate across all three settings, and a stronger defense-aware setting, 
where the attacker knows the defense exists and adapts the attack accordingly, for which we restrict evaluation to LVLMs and classifiers. For CLIP, improvements are observed both when \etthree is applied alone and when combined with complementary techniques such as test-time augmentation and test-time prompt tuning. Furthermore, the features obtained from the CLIP vision encoder obtained after applying \etthree consistently enhance the adversarial robustness of downstream LVLMs, demonstrating that \etthree yields consistent robustness gains across all settings.

\subsection{Evaluating Zero-Shot Classification} 
We perform zero‑shot classification using CLIP models with the robust vision encoders TeCoA \cite{mao2023understanding} and FARE \cite{schlarmann2024robust}, and evaluate across 15 benchmark datasets, full implementation details are in App.~\ref{app:implement_clip} provided in supp. material.

\minisection{Attack setup} Unless stated otherwise, following standard practice~\cite{schlarmann2024robust}, we use the two attacks from AutoAttack~\cite{croce2020reliable}: APGD-CE and APGD-DLR with 100 iterations each.

\minisection{Results} 
In \cref{tab:tab1} we show our test-time defense \etthree consistently enhances robust accuracy on adversarially robust models (TeCoA and FARE) across 14 datasets under attacks with a perturbation magnitude of $\epsilon_a = 4/255$. These models trained against attacks with either $\epsilon_t=2/255 \text{ or } 4/255$, where \etthree show substantial improvements even with the ``weaker'' models, which were trained against the weaker attacks. This trend is further illustrated in \cref{fig:et3-performance}, where \etthree improves robustness even for larger magnitude attacks---see App.~\ref{app:increased_attack_strength} for per-dataset results and additional details. 
Next, in \cref{tab:tab2} we show that \etthree surpasses similar defense methods in diverse defense scopes. Following the comparison standardization of the previous methods, we report results on a subset of 8 datasets.
In the \textbf{lightweight defense scope}, \etthree consistently improves robust accuracy over baseline CLIP model and TTC~\cite{Xing_2025_CVPR}, and surpasses even the slower TPT and C-TPT methods~\cite{shu2022test,yoon2024ctpt}, despite requiring no training.
Within the \textbf{multiple augmentation scope}, we show that an easy incorporation of the lightweight \etthree yields further gains to the ensemble that averages predictions across all augmented views. It surpassing MTA~\cite{zanella2024test} and matching the state-of-the-art robustness of R-TPT~\cite{sheng2025r}. In the last \textbf{augmentation-based methods with prompt-tuning scope}, we outperform existing methods by similarly combining \etthree with the former state-of-the-art R-TPT. Unless otherwise stated, the $\epsilon$ for \etthree is set to 5 for TeCoA and 4 for FARE, with the number of steps fixed at $T=2$.

\newcommand{\clip}{\textsc{CLIP}}
\newcommand{\openf}{\textsc{OpenFlamingo}}
\newcommand{\llava}{\textsc{LLaVA}}

\newcommand{\tecoatwo}{\textsc{TeCoA}$^{2}$}
\newcommand{\tecoafour}{\textsc{TeCoA}$^{4}$}
\newcommand{\faretwo}{\textsc{FARE}$^{2}$}
\newcommand{\farefour}{\textsc{FARE}$^{4}$}

\newcommand{\Tstrut}{\rule{0pt}{2.6ex}}
\newcommand{\Bstrut}{\rule[-0.9ex]{0pt}{0pt}}

\renewcommand{\arraystretch}{1.0}
 \definecolor{darkgray}{gray}{0.57}
\begin{table*}[t]
\centering
\caption{\textbf{Evaluation of ET3 on LLaVA 1.5-7B with different vision encoders in defense-unaware setting.} Clean and $\ell_\infty$-robust performance ($\epsilon_{a}=4/255$) using standard CLIP and the TeCoA/FARE backbones adversarially trained with $\epsilon_t=2/255$ and $\epsilon_t=4/255$. Clean results use full test sets, while adversarial scores are computed on 500 APGD perturbations following the ensemble protocol of \cite{schlarmann2024robust}. Across all tasks, ImageNet-21k labels serve as the reference text embeddings for computing the energy $E(\cdot,\theta)$. ET3 performs two gradient descent iterations. COCO and Flickr30k are evaluated with CIDEr for captioning, while TextVQA and VQAv2 report VQA accuracy.}
\scriptsize
\setlength{\tabcolsep}{2pt}

\resizebox{\textwidth}{!}{
\begin{tabular}{
>{\raggedright\arraybackslash}m{8.5mm}
Gc Gc
Gc Gc
Gc Gc
Gc Gc
Gc Gc}
\toprule
& \multicolumn{4}{c}{\cellcolor{cianoChiaro}\textbf{COCO} \cite{cocodataset}} 
& \multicolumn{4}{c}{\cellcolor{cianoChiaro}\textbf{Flickr30k} \cite{flickr30k}}  
& \multicolumn{4}{c}{\cellcolor{cianoChiaro}\textbf{TextVQA} \cite{singh2019towards}} 
& \multicolumn{4}{c}{\cellcolor{cianoChiaro}\textbf{VQAv2} \cite{goyal2017making}} 
& \multicolumn{4}{c}{\cellcolor{cianoChiaro}\textbf{Average}} \\

\arrayrulecolor{LightCyan4}
\cmidrule(lr){2-5}
\cmidrule(lr){6-9}
\cmidrule(lr){10-13}
\cmidrule(lr){14-17}
\cmidrule(lr){18-21}

& Clean & \textbf{+ET3} & 4/255 & \textbf{+ET3}
& Clean & \textbf{+ET3} & 4/255 & \textbf{+ET3}
& Clean & \textbf{+ET3} & 4/255 & \textbf{+ET3}
& Clean & \textbf{+ET3} & 4/255 & \textbf{+ET3}
& Clean & \textbf{+ET3} & 4/255 & \textbf{+ET3} \\

\arrayrulecolor{LightCyan4}\midrule

\clip
&  {115.5} & 111.3 & 2.7 &  {61.1} 
&  {77.5} & 76.1 & 1.1 &  {37.1}
&  {37.1} & 34.9 & 0.2 &  {17.1}
&  {74.5} & 73.1 & 0.0 &  {40.0}
&  {76.2} & {73.9} \smallred{2.3} & 1.0 & { {38.8}} \smallblue{37.8} \\

\tecoatwo
& 98.4 &  {99.0} & 30.0 &  {56.0}
& 57.1 &  {58.1} & 14.8 &  {32.1}
& 24.1 &  {24.2} & 7.8 &  {13.2}
& 66.9 &  {67.7} & 25.1 &  {40.7}
& 61.6 &  {62.3} \smallblue{0.7} & 19.4 &  {35.5} \smallblue{16.1}\\

\faretwo
& 109.9 &  {110.2} & 32.5 &  {52.7}
& 71.1 &  {71.5} & 17.5 &  {31.9}
& 31.9 & 31.9 & 7.2 &  {15.3}
& 71.7 & 71.7 & 24.5 &  {32.6}
& 71.2 &  {71.3} \smallblue{0.1} & 20.4 &  {33.1} \smallblue{12.7} \\

\arrayrulecolor{cianoVeryChiaro}\midrule

\tecoafour
&  {88.3} & 88.2 & 34.4 &  {53.7}
&  {48.6} & 48.2 & 19.5 &  {30.2}
& 20.7 &  {21.0} & 9.5 &  {12.6}
& 63.2 & 63.2 & 31.1 &  {42.6}
& 55.2 & 55.2 \smallblue{0.0} & 23.6 &  {34.8} \smallblue{11.2} \\

\farefour
& 102.4 &  {103.0} & 42.2 &  {56.4}
& 61.6 &  {61.8} & 23.1 &  {33.5}
&  {27.6} & 27.4 & 10.2 &  {17.2}
& 68.3 & 68.3 & 29.5 &  {42.2}
& 65.0 &  {65.1} \smallblue{0.1} & 26.3 &  {37.3} \smallblue{11.0}\\

\arrayrulecolor{black}\bottomrule
\end{tabular}
}
\label{tab:robust-llava}
\end{table*}

\begin{table}[tbh]
\caption{\textbf{Defense-aware worst-case robustness on ImageNet.} Clean and $\ell_\infty$-robust accuracy ($\epsilon_a = 4/255$) of base model. All TTT methods build 
upon the same base robust model from \cite{salman2020adversarially}.}
\centering
\setlength{\tabcolsep}{3pt}
\renewcommand{\arraystretch}{1.0}
\footnotesize
\begin{tabular}{l l c c}
\toprule
\cellcolor{cianoChiaro}\textbf{Method (Architecture)} & \cellcolor{cianoChiaro}\textbf{Adaptive Attack} & \cellcolor{cianoChiaro}\textbf{Clean} & \cellcolor{cianoChiaro}\textbf{Robust} \\
\midrule
\multicolumn{4}{l}{\textit{\textbf{Adversarial Purification}}} \\
DiffPure~\cite{nie2022diffusion} (WRN-50-2) & DiffBreak~\cite{KassisHY25} & 74.22 & 12.11 \\
DiffPure~\cite{nie2022diffusion} (DeiT-S) & DiffBreak~\cite{KassisHY25} & 73.63 & 25.00 \\
GDMP~\cite{wang2022guided}  (DeiT-S) & DiffBreak~\cite{KassisHY25} & 69.14 & 20.70 \\
\arrayrulecolor{LightCyan4}\midrule
\multicolumn{4}{l}{\textit{\textbf{TTT} --- Base: Salman et al. (RN-50)~\cite{salman2020adversarially}}} \\
Base Model Only & AutoAttack~\cite{croce2020reliable} & 64.02 & 34.96 \\
\arrayrulecolor{cianoVeryChiaro}\cmidrule(lr){1-4}
+ Singh et al.~\cite{singh2024robust} & APGD-DLR~\cite{croce2020reliable} & 63.91 & 34.68 \smallred{0.28}\\
+ Kulkarni et al. ~\cite{kulkarni2024igdefense} & IW-WC~\cite{kulkarni2024igdefense} & 64.10 & 35.48 \smallblue{0.52}\\
+ \textbf{\etthree (Ours)} & \hspace{-20pt}\footnotesize{Tr. APGD-T~+~BPDA}~\cite{croce2022evaluating}& 63.12 & 37.70 \smallblue{2.74} \\
\arrayrulecolor{black}\bottomrule
\end{tabular}
\vspace{-10pt}
\label{tab:classifier-acc-comparison}
\end{table}

\subsection{Evaluating Large Vision-Language Models}\label{sec:exp_lava}
We evaluate \etthree on LLaVA 1.5-7B using as vision encoder ($i$) standard CLIP, ($ii$) TeCoA-robust CLIP~\cite{mao2023understanding}, and ($iii$) FARE-robust CLIP~\cite{schlarmann2024robust}. For both robust methods, we employ the two variants trained with $\epsilon = 2/255$ and $\epsilon = 4/255$, marked as $^2$ or $^4$ respectively.
We consider image captioning (COCO, Flickr30k) and visual question answering tasks (TextVQA, VQAv2); evaluations on clean data use the full datasets, while adversarial results are computed over 500 perturbed inputs generated with an attack budget of $\epsilon_{a}=4/255$, following the evaluation of~\cite{schlarmann2024robust}. Across all tasks, we employ ImageNet-21k text labels \cite{ridnik2021imagenet21k}, to compute the energy. Further details in App.~\ref{app:eval_lvlm}.

\minisection{Results} Under the defense-unaware setting shown in, \cref{tab:robust-llava}, \etthree consistently improves the robustness across all base models while preserving performance on clean data. We further evaluate \etthree under defense-aware setting (avg. results in \cref{tab:vlm_adaptive_avg}; detailed results in App.~\ref{app:adap_LVLM}). \etthree introduces minimal overhead with increase in inference time by as little as 2.3\% in single-step defense (details in App.~\ref{app:comp_time}).
\subsection{Evaluation with Robust Classifiers}
We evaluate \etthree under defense-aware setting on robust classifiers from \texttt{RobustBench}~\cite{croce2021robustbench} using the ImageNet dataset. We compare our method against other test-time transformation defenses as well as popular strong defenses, including adversarial training and adversarial purification.
We assess \etthree under \emph{adaptive attacks}, where the attacker has full knowledge of the defense and actively attempts to circumvent it as suggested by~\cite{croce2022evaluating}.
Following \cite{croce2022evaluating}, we adopt a strong attack configuration, targeted APGD-DLR with gradients approximated through the defense via Backward Pass Differentiable Approximation (BPDA)~\cite{athalye2018obfuscated}, effectively accounting for the entire test-time optimization process. We also perform a transfer attack from the static base model and report the worst-case accuracy (more details in App.~\ref{app:adaptive_class}). Due to the computational cost of adaptive attacks, we restrict evaluation to ResNet-50 classifiers. For other competing defenses, we show robustness using the strongest attacks appropriate for each method—for example, DiffBreaker~\cite{KassisHY25} for diffusion-based purification defenses, and AutoAttack for robust classifiers. 

\minisection{Results} \cref{tab:classifier-acc-comparison} reports the  clean and worst-case robust accuracies i.e. the lowest accuracies achieved under the strongest attacks reported in the original papers, including both defense and attack evaluations where applicable. Methods are organized by defense family—adversarial purification, and test-time transformation (TTT)—and worst-case robustness is reported using the strongest adaptive attack available for each method. Consistent with prior work, DiffBreak significantly reduces the robustness of diffusion-based purification methods, whereas adversarially trained models maintain high robustness. In the TTT setting, where all methods use the same robust model as their base classifier~\cite{salman2020adversarially}, \etthree performs the best and improves worst-case robust accuracy. To ensure meaningful comparisons, we focus on worst-case robustness under the relevant adaptive attack for each method and omit defenses previously shown to fail under such evaluations~\cite{croce2022evaluating}. We also present more results under the defense-unaware setting with larger transformer-based architectures in App.~\ref{app:additionalclassifier}.

For this section, we further provide discussion on budget fairness of defense w.r.t to attack, label set ablation, and a single-step 
variant of \etthree for VLMs in App.~\ref{app:ablation}.

\begin{figure}[tbh]
  \caption{\textbf{Robust accuracy across increasing 
attack strengths in the defense-unaware setting.} Average zero-shot accuracy of CLIP over 14 benchmark datasets, showing that \etthree consistently improves 
the robustness of the TeCoA models trained with different defense strengths ($\epsilon_t$) as the attack strength ($\epsilon_a$) increases.}

    \centering
    \begin{overpic}[width=\linewidth]{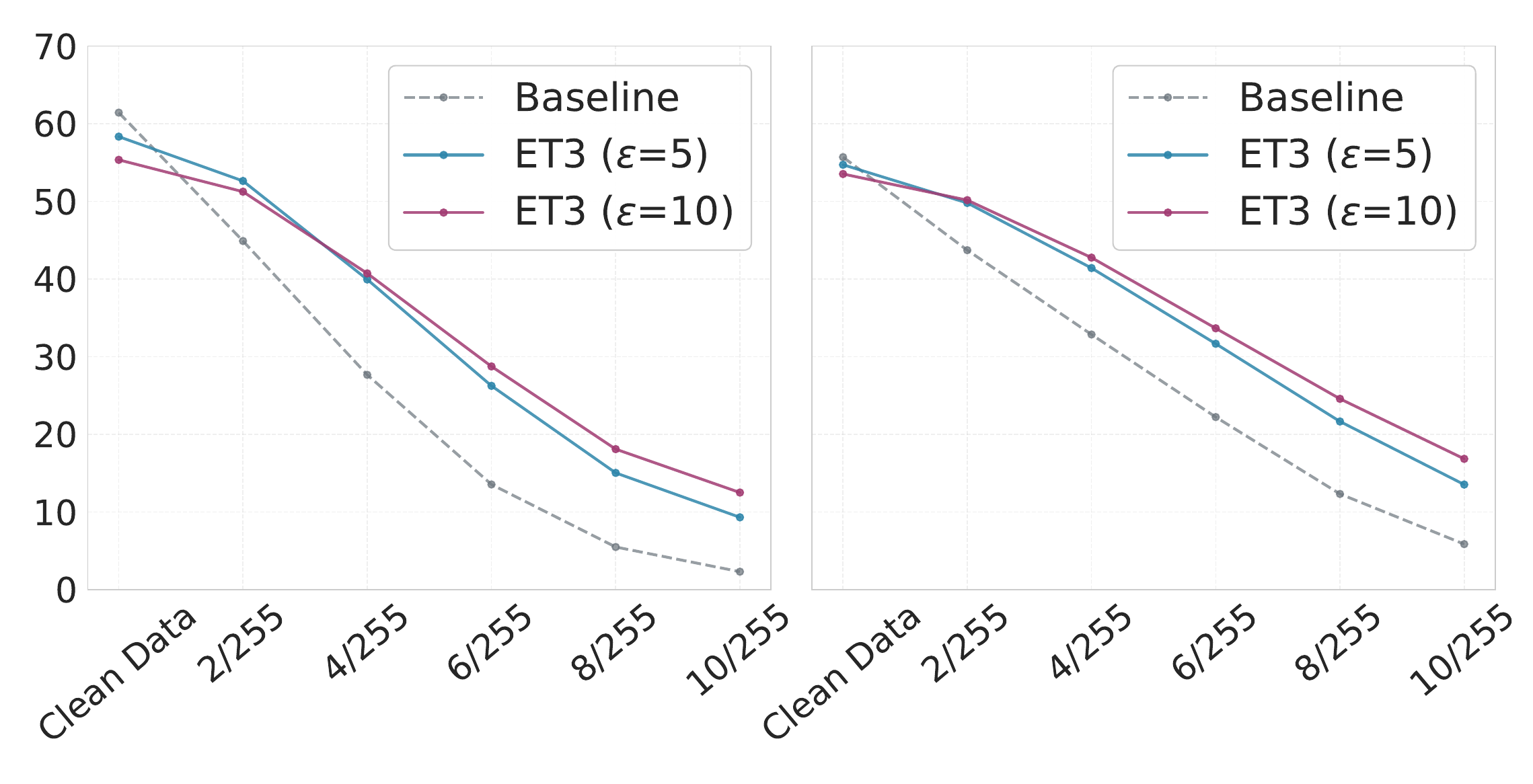}
        \put(5,50){\footnotesize ViT-L/14 (TeCoA)~$\epsilon_t=2/255$}
        \put(55,50){\footnotesize ViT-L/14 (TeCoA)~$\epsilon_t=4/255$}
        \put(15,3){\footnotesize Attack Strength}
        \put(60,3){\footnotesize Attack Strength}
        \put(-2,15){\rotatebox{90}{\small Avg. Zero-shot Acc.}}
    \end{overpic}
  
    \label{fig:et3-performance}
\end{figure}

\captionof{table}{\textbf{Defense-aware adaptive attack evaluation on LLaVA.} 
Average robust accuracy over four datasets (200 samples). \emph{Legend}: B = baseline (no defense); 
N-Ad. = non-adaptive; Ad. = adaptive; 
{\setlength{\fboxsep}{0.5pt}\colorbox{blue!20}{blue}} = \etthree with $\ell_2$ proj.; 
{\setlength{\fboxsep}{0.5pt}\colorbox{orange!20}{orange}} = \etthree with $\ell_\infty$ proj.}
\label{tab:vlm_adaptive_avg}
\noindent
\renewcommand{\arraystretch}{1.0}
\setlength{\tabcolsep}{2pt}
\resizebox{\linewidth}{!}{%
\begin{tabular}{
    Gc Gc Gc
    Gc Gc Gc
    Gc Gc Gc
    Gc Gc Gc
    Gc Gc Gc}
\toprule
\multicolumn{3}{c}{\cellcolor{cianoChiaro}\textbf{CLIP}} & 
\multicolumn{3}{c}{\cellcolor{cianoChiaro}\textbf{FARE$^2$}} & 
\multicolumn{3}{c}{\cellcolor{cianoChiaro}\textbf{TeCoA$^2$}} & 
\multicolumn{3}{c}{\cellcolor{cianoChiaro}\textbf{FARE$^4$}} & 
\multicolumn{3}{c}{\cellcolor{cianoChiaro}\textbf{TeCoA$^4$}} \\
\arrayrulecolor{LightCyan4}
\cmidrule(lr){1-3}\cmidrule(lr){4-6}\cmidrule(lr){7-9}\cmidrule(lr){10-12}\cmidrule(lr){13-15}
B & N-Ad. & Ad. & 
B & N-Ad. & Ad. & 
B & N-Ad. & Ad. & 
B & N-Ad. & Ad. & 
B & N-Ad. & Ad. \\
\arrayrulecolor{LightCyan4}\midrule
\rowcolor{blue!20}
0.9 & \textbf{16.7} & 4.5 & 
17.6 & \textbf{28.9} & 25.0 & 
18.3 & \textbf{30.9} & 25.5 & 
24.1 & \textbf{32.8} & 31.0 & 
19.7 & \textbf{30.9} & 28.6 \\
\rowcolor{orange!20}
0.9 & \textbf{12.4} & 3.3 & 
17.6 & \textbf{27.0} & 23.7 & 
18.3 & \textbf{31.3} & 26.9 & 
24.1 & \textbf{32.5} & 31.1 & 
19.7 & \textbf{31.1} & 28.1 \\
\arrayrulecolor{black}\bottomrule
\end{tabular}}

\section{Conclusion}
In this paper, we present the \etthree defense, a novel test-time transformation method based on energy minimization. We showed extensive experiments of boosting robustness for image classification, zero-shot classification using CLIP, and with Large Vision Language Models, demonstrating superiority over many datasets and downstream tasks.  We also present theoretical justification for which our \etthree will provably transform a clean or adversarial image into correctly classified images.
A promising future direction is to optimize the network to support conditions that enable the \etthree defense, such as increasing its local linearity radius or the energy gradient ratio.

\minisection{Acknowledgment} Supported by projects PNRR MUR PE0000013-FAIR under the MUR National Recovery and Resilience Plan funded by the European Union - NextGenerationEU, PRIN 2022 project 20227YET9B ``AdVVent'' CUP code B53D23012830006. Also partially supported by Sapienza research project BEAT (Better dEep leArning securiTy), bando per la ricerca di Ateneo 2024. We acknowledge CINECA HPC resources via ISCRA (HP10CIWXKH) and EuroHPC (EHPC-DEV-2025D07-037) for computing.
{
    \small
    \bibliographystyle{ieeenat_fullname}
    \bibliography{refs}
}

\appendix
\clearpage
\maketitlesupplementary

\section{Implementation details}
This section provides the implementation details for all experiments presented in the main paper.
\subsection{Details on Zero-shot evaluation}\label{app:implement_clip}

Unless stated otherwise, all zero-shot experiments, consistent with the setup adopted in~\cite{schlarmann2024robust}, follow the standard CLIP evaluation protocol used in the CLIP Benchmark~\cite{cherti_2025_15403103} and OpenCLIP~\cite{cherti2023reproducible}.

For each dataset, every class name is paired with a set of prompt templates, producing multiple natural-language descriptions per class. These prompts are encoded with the CLIP text encoder to obtain their corresponding textual embeddings. For each class, we average all template-derived embeddings to form a single class-level representation. Zero-shot predictions are then computed by taking the cosine similarity between the CLIP image embedding and all class embeddings, assigning the label with the highest similarity score.

Throughout the experiments, we report results from the following datasets: Caltech101~\cite{fei2004learning}, Stanford Cars~\cite{krause20133d}, CIFAR-10 and CIFAR-100~\cite{krizhevsky2009learning}, DTD~\cite{cimpoi2014describing}, EuroSAT~\cite{helber2019eurosat}, FGVC Aircraft~\cite{maji2013fine}, Flowers~\cite{nilsback2008automated}, ImageNet-R~\cite{hendrycks2021many}, ImageNet-Sketch~\cite{wang2019learning}, PCAM~\cite{veeling2018rotation}, Oxford Pets~\cite{parkhi2012cats}, and STL-10~\cite{coates2011analysis} and UCF-101\cite{soomro2012ucf101}. We also report results on the validation set of ImageNet-1k \cite{deng2009imagenet} consistent with prior work\cite{mao2023understanding,schlarmann2024robust}.

\minisection{Attack setup} Consistent with the established evaluation setup from ~\cite{schlarmann2024robust}, we measure adversarial robustness on a subset of $1000$ randomly selected samples from each dataset, while clean accuracy is computed over all clean samples. Adversarial examples are generated using the first two attacks from the AutoAttack suite~\cite{croce2020reliable}: APGD with cross-entropy loss (APGD-CE) and APGD with the DLR loss (APGD-DLR), each executed for $100$ iterations. For binary datasets such as PCAM, where the DLR loss is not applicable, only APGD-CE is used. All evaluations assume an $\ell_\infty$ threat model with perturbation magnitudes of $\varepsilon_a = 4/255$.  Unless noted otherwise, all robustness experiments are conducted at $224\times224$ resolution, while CIFAR-10, CIFAR-100, and STL-10 are evaluated at their native image sizes.

\begin{table*}[t!]
\caption{\textbf{Zero-shot robustness of \etthree across smaller model architectures in defense-unaware setting.} Comparison of clean and robust accuracy for baseline models versus the same models augmented with \etthree.  Robustness is evaluated against Auto-Attack (AA) at $\epsilon_a = 4/255$.}
    \centering
    \scriptsize
    \setlength{\tabcolsep}{1.8pt}
    \resizebox{\textwidth}{!}{
    \begin{tabular}{cc|>{\columncolor{lightgray}}Cc
    >{\columncolor{lightgray}}Cc
    >{\columncolor{lightgray}}Cc
    >{\columncolor{lightgray}}Cc
    >{\columncolor{lightgray}}Cc
    >{\columncolor{lightgray}}Cc
    >{\columncolor{lightgray}}Cc|c|c}
    \toprule
    Model & Method & 
    \rotatebox[origin=c]{60}{ImageNet} & 
    \rotatebox[origin=c]{60}{CalTech} & 
    \rotatebox[origin=c]{60}{Cars} & 
    \rotatebox[origin=c]{60}{CIFAR10} & 
    \rotatebox[origin=c]{60}{CIFAR100} & 
    \rotatebox[origin=c]{60}{DTD} & 
    \rotatebox[origin=c]{60}{EuroSAT} & 
    \rotatebox[origin=c]{60}{FGVC} & 
    \rotatebox[origin=c]{60}{Flowers} & 
    \rotatebox[origin=c]{60}{ImageNet-R} & 
    \rotatebox[origin=c]{60}{ImageNet-S} & 
    \rotatebox[origin=c]{60}{PCAM} & 
    \rotatebox[origin=c]{60}{OxfordPets} &  
    \rotatebox[origin=c]{60}{STL-10} & 
    \rotatebox[origin=c]{60}{Avg.} & 
    \rotatebox[origin=c]{60}{Improv.}\\
    \midrule
    \multirow{4}{*}{\shortstack{ViT-B/32 \\(TeCoA)\\$\epsilon_t=4/255$}}
    & Base (Clean) & 56.16 & 73.39 & 13.77 & 74.89 & 40.93 & 24.57 & 22.67 & 5.79 & 29.31 & 49.11 & 29.58 & 50.01 & 70.89 & 87.30 & 44.88 & \multirow{2}{*}{\smallred{0.6}} \\
    & +~\etthree (Clean) & 55.15 & 75.20 & 11.25 & 74.74 & 38.19 & 23.83 & 19.26 & 4.95 & 28.85 & 51.65 & 30.73 & 50.00 & 69.66 & 85.41 & 44.21 \\*[3pt]
    & Base (Robust) & 24.05 & 52.18 & 2.79 & 32.38 & 17.06 & 11.54 & 7.35 & 0.30 & 7.42 & 22.04 & 13.87 & 49.90 & 33.22 & 58.55 & 23.76 & \multirow{2}{*}{\smallblue{8.11}} \\
    & +~\etthree (Robust) & 34.47 & 63.02 & 5.29 & 51.60 & 26.85 & 16.44 & 13.59 & 2.76 & 14.91 & 33.81 & 21.59 & 49.98 & 44.73 & 67.09 & 31.87   \\
    \midrule
    \multirow{4}{*}{\shortstack{ViT-B/32 \\(FARE)\\$\epsilon_t=4/255$}}
    & Base (Clean) & 51.38 & 78.98 & 38.52 & 68.18 & 45.69 & 31.17 & 17.54 & 10.74 & 37.68 & 53.60 & 32.27 & 50.02 & 78.09 & 89.41 & 48.80 & \multirow{2}{*}{\smallred{0.75}} \\
    & +~\etthree (Clean) & 49.98 & 78.44 & 36.08 & 70.87 & 37.77 & 29.84 & 18.02 & 9.36 & 38.17 & 54.94 & 31.73 & 50.02 & 78.52 & 88.90 & 48.05  \\*[3pt]
     & Base (Robust) & 14.62 & 50.30 & 2.33 & 28.10 & 14.33 & 13.46 & 9.63 & 0.39 & 5.40 & 19.05 & 11.79 & 49.20 & 23.55 & 55.20 & 21.24 & \multirow{2}{*}{\smallblue{7.26}} \\
    & +~\etthree (Robust) & 21.31 & 57.22 & 7.67 & 46.51 & 23.63 & 18.56 & 13.00 & 3.78 & 13.74 & 28.27 & 17.63 & 49.21 & 36.96 & 61.50 & 28.50  \\
    \midrule
    \multirow{4}{*}{\shortstack{ConvNeXt-B \\(TeCoA)\\$\epsilon_t=4/255$}}
    & Base (Clean) & 67.68 & 79.95 & 61.32 & 74.18 & 49.02 & 43.14 & 25.13 & 12.84 & 47.88 & 67.37 & 50.38 & 49.24 & 80.54 & 90.81 & 57.11  & \multirow{2}{*}{\smallred{0.72}} \\
    & +~\etthree (Clean) & 67.05 & 79.54 & 60.91 & 74.24 & 46.89 & 45.00 & 22.26 & 11.31 & 45.80 & 68.64 & 50.03 & 48.89 & 79.42 & 89.51 & 56.39 \\*[3pt]
    & Base (Robust) & 37.10 & 62.20 & 22.20 & 35.90 & 20.20 & 22.50 & 13.50 & 1.60 & 17.80 & 35.30 & 31.20 & 34.50 & 48.90 & 69.30 & 32.30  & \multirow{2}{*}{\smallblue{8.56}} \\
    & +~\etthree (Robust) & 48.40 & 67.30 & 33.20 & 52.90 & 32.40 & 32.60 & 15.70 & 5.40 & 26.30 & 47.30 & 37.10 & 39.10 & 58.70 & 75.60 & 40.86 \\
    \midrule
    \multirow{4}{*}{\shortstack{ConvNeXt-B \\(FARE)\\$\epsilon_t=4/255$}}
    & Base (Clean) & 63.45 & 82.53 & 84.75 & 74.26 & 53.33 & 48.14 & 23.04 & 14.52 & 52.07 & 74.42 & 54.55 & 48.17 & 81.98 & 92.17 & 60.53   & \multirow{2}{*}{\smallred{1.35}} \\  %
    & +~\etthree (Clean) & 62.67 & 81.91 & 84.19 & 61.84 & 45.41 & 47.50 & 24.13 & 14.37 & 53.03 & 74.81 & 53.71 & 49.93 & 83.76 & 91.24 & 59.18 \\*[3pt]
    & Base (Robust) & 23.80 & 63.20 & 27.20 & 29.10 & 17.70 & 21.50 & 13.00 & 1.10 & 13.10 & 34.20 & 27.60 & 15.00 & 35.50 & 67.00 & 27.79  & \multirow{2}{*}{\smallblue{5.36}} \\
    & +~\etthree (Robust) &30.40 & 65.70 & 32.80 & 33.60 & 24.50 & 27.50 & 16.40 & 3.70 & 20.60 & 39.40 & 31.90 & 23.60 & 44.00 & 70.00 & 33.15 \\
    \midrule
    \multirow{4}{*}{\shortstack{ViT-B/32 \\(TeCoA)\\$\epsilon_t=1/255$}}
    & Base (Clean) & 70.53 & 77.14 & 28.88 & 85.89 & 54.96 & 32.82 & 28.80 & 12.30 & 48.41 & 61.65 & 41.16 & 44.19 & 81.22 & 93.45 & 54.39  & \multirow{2}{*}{\smallred{0.87}} \\
    & +~\etthree (Clean) & 69.83 & 76.24 & 26.28 & 82.33 & 50.65 & 31.33 & 33.20 & 11.13 & 47.91 & 63.81 & 41.68 & 42.25 & 81.25 & 91.41 & 53.52 \\*[3pt]
     & Base (Robust) & 2.83 & 15.00 & 0.57 & 9.99 & 2.12 & 5.59 & 5.24 & 0.54 & 1.63 & 3.03 & 3.19 & 34.77 & 7.39 & 23.60 & 8.25  & \multirow{2}{*}{\smallblue{2.88}} \\
    & +~\etthree (Robust) & 4.57 & 17.44 & 1.22 & 13.47 & 4.99 & 10.48 & 10.83 & 1.11 & 4.15 & 5.07 & 4.71 & 36.86 & 12.37 & 28.52 & 11.13 \\
    \midrule
    \multirow{4}{*}{\shortstack{ViT-B/32 \\(FARE)\\$\epsilon_t=1/255$}}
    & Base (Clean) & 62.60 & 82.45 & 56.29 & 88.52 & 64.22 & 40.85 & 30.81 & 16.98 & 61.83 & 67.40 & 41.45 & 52.06 & 86.94 & 96.16 & 60.61    & \multirow{2}{*}{\smallred{1.37}} \\
    & +~\etthree (Clean) & 60.33 & 80.12 & 53.03 & 85.29 & 56.51 & 37.50 & 43.50 & 14.13 & 59.26 & 67.81 & 40.21 & 52.21 & 85.20 & 94.24 & 59.24  \\*[3pt]
    & Base (Robust) & 0.14 & 4.54 & 0.41 & 8.18 & 1.25 & 3.14 & 5.41 & 0.63 & 0.73 & 1.27 & 1.28 & 35.10 & 0.90 & 9.80 & 5.20 & \multirow{2}{*}{\smallblue{3.05}} \\
    & +~\etthree (Robust) & 1.30 & 7.23 & 1.49 & 12.35 & 4.41 & 9.20 & 13.85 & 1.68 & 1.72 & 2.76 & 2.61 & 38.31 & 2.97 & 15.61 & 8.25 \\
    \bottomrule
    \end{tabular}
    } %

    \label{tab:more_robust_models}
\end{table*}

\begin{table}[t]
\caption{\textbf{Performance of robust models from RobustBench with and without \etthree in the defense-unaware setting.} We compare several robust models obtained from RobustBench, reporting performance both with and without \etthree. All evaluations use the APGD-T attack. For reference, the AutoAttack robust accuracy of each base model is included (in parentheses and shown in gray). All models are trained and evaluated under a standard $\ell_\infty$ threat model with a perturbation budget of $\epsilon = 4/255$.}
\centering
\small
\setlength{\tabcolsep}{2pt} %
\renewcommand{\arraystretch}{1.3} %
\begin{tabular}{c c c c} %
\toprule
\rowcolor{lightgray}
\textbf{Model} & \textbf{Defense} & \textbf{Clean Acc.} & \textbf{Robust Acc.} \\
\midrule
\multirow{2}{*}{\shortstack{ResNet-50 \\ \scriptsize Salman et al.~\cite{salman2020adversarially}}} 
    & Base & 64.02 & 35.40 \color{gray}\footnotesize{(35.20)} \\ 
    & \cellcolor{lightgray!20}+ \etthree & 63.12 & \hspace{-27pt}46.20 \\ 
\midrule
\multirow{2}{*}{\shortstack{ConvNeXt-B \\ \scriptsize Liu et al.~\cite{liu2025comprehensive}}} 
    & Base & 76.38 & 55.60 \color{gray}\color{gray}\footnotesize{(55.00)} \\ 
    & \cellcolor{lightgray!20}+ \etthree & 75.95 & \hspace{-27pt}61.40 \\ 
\midrule
\multirow{2}{*}{\shortstack{ConvNeXt-L \\ \scriptsize Liu et al.~\cite{liu2025comprehensive}}} 
    & Base & 77.47 & 57.70 \color{gray}\color{gray}\footnotesize{(57.40)} \\ 
    & \cellcolor{lightgray!20}+ \etthree & 76.36 & \hspace{-27pt}64.50 \\ 
\midrule
\multirow{2}{*}{\shortstack{Swin-B \\ \scriptsize Liu et al.~\cite{liu2025comprehensive}}} 
    & Base & 76.21 & 55.00 \color{gray}\color{gray}\footnotesize{(54.80)} \\ 
    & \cellcolor{lightgray!20}+ \etthree & 75.75 & \hspace{-27pt}61.70 \\ 
\midrule
\multirow{2}{*}{\shortstack{Swin-L \\ \scriptsize Liu et al.~\cite{liu2025comprehensive}}} 
    & Base & 78.18 & 58.10 \color{gray}\color{gray}\footnotesize{(57.80)} \\ 
    & \cellcolor{lightgray!20}+ \etthree & 77.21 & \hspace{-27pt}64.20 \\ 
\bottomrule
\end{tabular}
\label{tab:classifier_et3}
\end{table}

\renewcommand{\arraystretch}{1.0}
 \definecolor{darkgray}{gray}{0.57}
\begin{table*}[t]
\centering
\caption{\textbf{Evaluating LLaVA 1.5-7B with different vision encoders using \emph{one-step} \etthree in the defense-unaware setting.} Clean and $\ell_\infty$-robust performance ($\epsilon_{a}=4/255$) using standard CLIP and the TeCoA/FARE backbones adversarially trained with $\epsilon_t=2/255$ and $\epsilon_t=4/255$. Clean results use full test sets, while adversarial scores are computed on 500 APGD perturbations following the ensemble protocol of \cite{schlarmann2024robust}. Across all tasks, ImageNet-21k labels serve as the reference text embeddings for computing the energy $E(\cdot,\theta)$. ET3 performs one gradient descent iteration. COCO and Flickr30k are evaluated with CIDEr for captioning, while TextVQA and VQAv2 report VQA accuracy.}
\scriptsize
\setlength{\tabcolsep}{2pt}

\resizebox{\textwidth}{!}{
\begin{tabular}{
>{\raggedright\arraybackslash}m{8.5mm}
Gc Gc
Gc Gc
Gc Gc
Gc Gc
Gc Gc}
\toprule
& \multicolumn{4}{c}{\cellcolor{cianoChiaro}\textbf{COCO} \cite{cocodataset}} 
& \multicolumn{4}{c}{\cellcolor{cianoChiaro}\textbf{Flickr30k} \cite{flickr30k}}  
& \multicolumn{4}{c}{\cellcolor{cianoChiaro}\textbf{TextVQA} \cite{singh2019towards}} 
& \multicolumn{4}{c}{\cellcolor{cianoChiaro}\textbf{VQAv2} \cite{goyal2017making}} 
& \multicolumn{4}{c}{\cellcolor{cianoChiaro}\textbf{Average}} \\

\arrayrulecolor{LightCyan4}
\cmidrule(lr){2-5}
\cmidrule(lr){6-9}
\cmidrule(lr){10-13}
\cmidrule(lr){14-17}
\cmidrule(lr){18-21}

& Clean & \textbf{+ET3} & 4/255 & \textbf{+ET3}
& Clean & \textbf{+ET3} & 4/255 & \textbf{+ET3}
& Clean & \textbf{+ET3} & 4/255 & \textbf{+ET3}
& Clean & \textbf{+ET3} & 4/255 & \textbf{+ET3}
& Clean & \textbf{+ET3} & 4/255 & \textbf{+ET3} \\

\arrayrulecolor{LightCyan4}\midrule

\clip
& \textbf{115.5} & 112.2 & 2.7 & \textbf{68.2} 
& \textbf{77.5} & 75.3 & 1.1 & \textbf{38.9}
& \textbf{37.1} & 34.7 & 0.2 & \textbf{18.0}
& \textbf{74.5} & 73.3 & 0.0 & \textbf{43.9}
& \textbf{76.2} & 73.9 & 1.0 & \underline{\textbf{42.3}} \smallblue{41.3}\\

\tecoatwo
& 98.4 & \textbf{98.9} & 30.0 & \textbf{57.3}
& 57.1 & \textbf{57.3} & 14.8 & \textbf{33.1}
& 24.1 & 24.1 & 7.8 & \textbf{13.4}
& \textbf{66.9} & 66.8 & 25.1 & \textbf{41.9}
& 61.6 & \textbf{61.8} & 19.4 & \textbf{36.4} \smallblue{17.0}\\

\faretwo
& 109.9 & \textbf{110.3} & 32.5 & \textbf{57.0}
& 71.1 & \textbf{71.3} & 17.5 & \textbf{32.4}
& \textbf{31.9} & 31.8 & 7.2 & \textbf{15.0}
& 71.7 & 71.7 & 24.5 & \textbf{38.1}
& 71.2 & \textbf{71.3} & 20.4 & \textbf{35.6} \smallblue{15.2}\\

\arrayrulecolor{cianoVeryChiaro}\midrule

\tecoafour
& \textbf{88.3} & 88.1 & 34.4 & \textbf{55.5}
& \textbf{48.6} & 47.8 & 19.5 & \textbf{29.8}
& 20.7 & \textbf{20.9} & 9.5 & \textbf{12.46}
& 63.2 & 63.2 & 31.1 & \textbf{42.9}
& \textbf{55.2} & 55.0 & 23.6 & \textbf{35.2} \smallblue{11.6} \\

\farefour
& 102.4 & \textbf{102.6} & 42.2 & \textbf{57.7}
& \textbf{61.6} & 60.9 & 23.1 & \textbf{33.6}
& \textbf{27.6} & 27.4 & 10.2 & \textbf{16.8}
& 68.3 & 68.3 & 29.5 & \textbf{43.3}
& \textbf{65.0} & 64.8 & 26.3 & \textbf{37.9} \smallblue{11.6}\\

\arrayrulecolor{black}\bottomrule
\end{tabular}
}
\label{tab:robust-llava-one_step}
\end{table*}

\definecolor{neutralColor}{HTML}{FFF9E6}   %
\definecolor{correctColor}{HTML}{DFF0D8} %
\definecolor{incorrectColor}{HTML}{F2DEDE} %
\definecolor{headerColor}{HTML}{F5F5F5}   %

\begin{figure*}[t]
\centering

\begin{tabular}{c @{\hspace{1em}} p{0.80\textwidth}}

\raisebox{-0.5\height}{\includegraphics[width=0.2\textwidth, keepaspectratio]{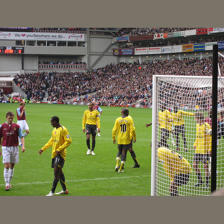}} & 

\begin{tabularx}{\linewidth}{l X}
    \rowcolor{headerColor}
    \multicolumn{2}{l}{\textbf{GT:} Sports team on a field wearing yellow jerseys with a goal net to the right.} \\
    
    \rowcolor{incorrectColor}
    \textbf{CLIP:} & A group of people are playing a game of dunking a hot dog in a bun. \\
    
    \rowcolor{correctColor}
    \hspace{5pt}\textbf{+ \etthree:} & A group of people playing soccer on a field. \\
    \hline

    \rowcolor{incorrectColor}
    \textbf{TeCoA$^2$:} & A group of people are playing with a net full of tennis balls. \\ 
    
    \rowcolor{correctColor}
    \hspace{5pt}\textbf{+ \etthree:} & A group of people are standing on a field with a soccer goal in the background. \\
        \hline

    \rowcolor{correctColor}
    \textbf{TeCoA$^4$:} & A group of people are standing on a field, with some of them wearing yellow shirts. \\
    
    \rowcolor{correctColor}
    \hspace{5pt}\textbf{+ \etthree:} & A group of people are standing on a field, with some of them wearing yellow shirts. \\
        \hline

    \rowcolor{correctColor}
    \textbf{FARE$^2$:} & A group of young boys playing soccer on a field. \\
    
    \rowcolor{correctColor}
    \hspace{5pt}\textbf{+ \etthree:} & A group of soccer players on a field. \\
        \hline

    \rowcolor{correctColor}
    \textbf{FARE$^4$:} & A group of men are standing on a field, some of them wearing yellow shirts. \\

    \rowcolor{correctColor}
    \hspace{5pt}\textbf{+ \etthree:} & A group of soccer players standing on a field. \\

\end{tabularx}
\end{tabular}

\vspace{25pt}

\begin{tabular}{c @{\hspace{1em}} p{0.80\textwidth}}

\raisebox{-0.5\height}{\includegraphics[width=0.2\textwidth, keepaspectratio]{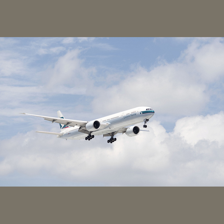}} & 

\begin{tabularx}{\linewidth}{l X}
    \rowcolor{headerColor}
    \multicolumn{2}{l}{\textbf{GT:} Jet flying in the sky among the clouds.} \\
    
    \rowcolor{incorrectColor}
    \textbf{CLIP:} & Angry Angry Birds are angry at the airport. \\
    
    \rowcolor{neutralColor}
    \hspace{5pt}\textbf{+ \etthree:} & A plane with a bunch of angry looking Sesame Street characters on it. \\
    \hline

    \rowcolor{neutralColor}
    \textbf{TeCoA$^2$:} & A large airplane is on the runway. \\ 
    
    \rowcolor{neutralColor}
    \hspace{5pt}\textbf{+ \etthree:} & A large airplane is taking off from a runway. \\
        \hline

    \rowcolor{correctColor}
    \textbf{TeCoA$^4$:} & A large jetliner is flying through the sky. \\
    
    \rowcolor{correctColor}
    \hspace{5pt}\textbf{+ \etthree:} &A large airplane flying in the sky. \\
        \hline

    \rowcolor{neutralColor}
    \textbf{FARE$^2$:} & A large jet airplane is taking off from a runway. \\
    
    \rowcolor{correctColor}
    \hspace{5pt}\textbf{+ \etthree:} & A large airplane is taking off into the sky. \\
        \hline

    \rowcolor{neutralColor}
    \textbf{FARE$^4$:} & A large airplane is on the runway. \\

    \rowcolor{correctColor}
    \hspace{5pt}\textbf{+ \etthree:} & A large airplane is flying through the sky. \\

\end{tabularx}
\end{tabular}

\vspace{25pt}

\begin{tabular}{c @{\hspace{1em}} p{0.80\textwidth}}

\raisebox{-0.5\height}{\includegraphics[width=0.2\textwidth, keepaspectratio]{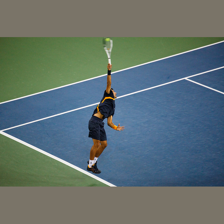}} & 

\begin{tabularx}{\linewidth}{l X}
    \rowcolor{headerColor}
    \multicolumn{2}{l}{\textbf{GT:} A man hitting a tennis ball with a racquet.} \\
    
    \rowcolor{incorrectColor}
    \textbf{CLIP:} & A cartoon cat with a football in its mouth. \\
    
    \rowcolor{neutralColor}
    \hspace{5pt}\textbf{+ \etthree:} & A woman in a purple shirt and black shorts is playing tennis. \\
    \hline

    \rowcolor{neutralColor}
    \textbf{TeCoA$^2$:} & A woman is playing tennis on a court. \\ 
    
    \rowcolor{neutralColor}
    \hspace{5pt}\textbf{+ \etthree:} & A woman is playing tennis on a court. \\
        \hline

    \rowcolor{correctColor}
    \textbf{TeCoA$^4$:} & A tennis player is in the middle of a serve, holding a tennis racket and jumping up. \\
    
    \rowcolor{correctColor}
    \hspace{5pt}\textbf{+ \etthree:} & A tennis player is swinging a racket on a tennis court. \\
        \hline

    \rowcolor{neutralColor}
    \textbf{FARE$^2$:} & A woman is playing tennis and is in the middle of a serve. \\
    
    \rowcolor{neutralColor}
    \hspace{5pt}\textbf{+ \etthree:} & A woman is playing tennis and is about to hit the ball. \\
        \hline

    \rowcolor{neutralColor}
    \textbf{FARE$^4$:} & A woman is playing tennis on a court. \\

    \rowcolor{correctColor}
    \hspace{5pt}\textbf{+ \etthree:} & A tennis player is in the middle of a serve. \\

\end{tabularx}
\end{tabular}

\vspace{25pt}

\begin{tabular}{c @{\hspace{1em}} p{0.80\textwidth}}

\raisebox{-0.5\height}{\includegraphics[width=0.2\textwidth, keepaspectratio]{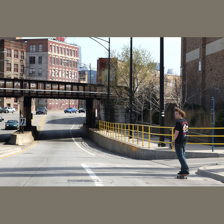}} & 

\begin{tabularx}{\linewidth}{l X}
    \rowcolor{headerColor}
    \multicolumn{2}{l}{\textbf{GT:} A man on a skateboard waits at the corner of a city street, with an overpass in the background.} \\
    
    \rowcolor{incorrectColor}
    \textbf{CLIP:} & Two girls are hugging each other in a parking lot. \\
    
    \rowcolor{neutralColor}
    \hspace{5pt}\textbf{+ \etthree:} & A woman with long hair is skateboarding in a parking lot. \\
    \hline

    \rowcolor{neutralColor}
    \textbf{TeCoA$^2$:} & A person is snowboarding on a ramp. \\ 
    
    \rowcolor{correctColor}
    \hspace{5pt}\textbf{+ \etthree:} & A person is skateboarding on a sidewalk. \\
        \hline

    \rowcolor{neutralColor}
    \textbf{TeCoA$^4$:} & A person is walking on a sidewalk near a bridge. \\
    
    \rowcolor{neutralColor}
    \hspace{5pt}\textbf{+ \etthree:} & A person is walking on a sidewalk near a bridge. \\
        \hline

    \rowcolor{neutralColor}
    \textbf{FARE$^2$:} & A person is standing on a sidewalk near a train track. \\
    
    \rowcolor{correctColor}
    \hspace{5pt}\textbf{+ \etthree:} & A person is skateboarding on a street. \\
        \hline

    \rowcolor{neutralColor}
    \textbf{FARE$^4$:} & A man is standing on a sidewalk next to a bus. \\

    \rowcolor{neutralColor}
    \hspace{5pt}\textbf{+ \etthree:} & A man is standing on a sidewalk next to a street. \\

\end{tabularx}
\end{tabular}

\caption{Qualitative comparison of generated captions for a sample image. \etthree corrects captions affected by adversarial attacks on standard CLIP and further refines captions produced by robust TeCoA and FARE. Green rows indicate semantically correct captions, red rows denote incorrect captions, and yellow rows highlight outputs with partial errors that still broadly reflect the image content. All attacks are generated with $\epsilon_a =4/255$.}
\label{fig:qualitative_example_cap}
\end{figure*}

\definecolor{neutralColor}{HTML}{FFF9E6}   %
\definecolor{correctColor}{HTML}{DFF0D8} %
\definecolor{incorrectColor}{HTML}{F2DEDE} %
\definecolor{headerColor}{HTML}{F5F5F5}   %

\begin{figure*}[t]
\centering

\begin{tabular}{@{}p{0.48\textwidth}@{\hspace{0.04\textwidth}}p{0.48\textwidth}@{}}

\begin{minipage}[t]{\linewidth}

\begin{tabular}{c @{\hspace{0.5em}} p{0.65\linewidth}}
\parbox{0.35\linewidth}{\small\textbf{Q:} Is this photo taken indoors or outdoors?} \\ [-20pt]
\raisebox{-0.7\height}{\includegraphics[width=0.35\linewidth, keepaspectratio]{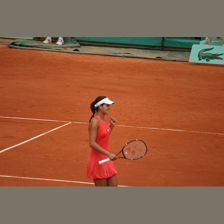}} & 
\begin{tabularx}{\linewidth}{l X}
    \rowcolor{headerColor}
    \multicolumn{2}{l}{\textbf{Answer:} Outdoors.} \\
    \rowcolor{incorrectColor}
    \textbf{CLIP:} & maybe. \\
    \rowcolor{incorrectColor}
    \hspace{5pt}\textbf{+ \etthree:} & Indoors. \\
    \hline
    \rowcolor{incorrectColor}
    \textbf{TeCoA$^2$:} & Indoors. \\ 
    \rowcolor{correctColor}
    \hspace{5pt}\textbf{+ \etthree:} & Outdoors. \\
    \hline
    \rowcolor{incorrectColor}
    \textbf{TeCoA$^4$:} & Indoors. \\
    \rowcolor{correctColor}
    \hspace{5pt}\textbf{+ \etthree:} & Outdoors. \\
    \hline
    \rowcolor{incorrectColor}
    \textbf{FARE$^2$:} & Indoors. \\
    \rowcolor{correctColor}
    \hspace{5pt}\textbf{+ \etthree:} & Outdoors. \\
    \hline
    \rowcolor{incorrectColor}
    \textbf{FARE$^4$:} & Indoors. \\
    \rowcolor{correctColor}
    \hspace{5pt}\textbf{+ \etthree:} & Outdoors. \\
\end{tabularx}
\end{tabular}

\vspace{10pt}

\begin{tabular}{c @{\hspace{0.5em}} p{0.65\linewidth}}
\parbox{0.35\linewidth}{\small\textbf{Q:} what does this sign say to do?} \\ [-20pt]
\raisebox{-0.7\height}{\includegraphics[width=0.35\linewidth, keepaspectratio]{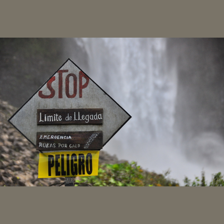}} & 
\begin{tabularx}{\linewidth}{l X}
    \rowcolor{headerColor}
    \multicolumn{2}{l}{\textbf{Answer:} Stop.} \\
    \rowcolor{incorrectColor}
    \textbf{CLIP:} & Stop limiting pelicans. \\
    \rowcolor{correctColor}
    \hspace{5pt}\textbf{+ \etthree:} & Stop. \\
    \hline
    \rowcolor{incorrectColor}
    \textbf{TeCoA$^2$:} & No liquor. \\ 
    \rowcolor{correctColor}
    \hspace{5pt}\textbf{+ \etthree:} & Stop. \\
    \hline
    \rowcolor{correctColor}
    \textbf{TeCoA$^4$:} & Stop. \\
    \rowcolor{correctColor}
    \hspace{5pt}\textbf{+ \etthree:} & Stop. \\
    \hline
    \rowcolor{incorrectColor}
    \textbf{FARE$^2$:} & Stop at geyser. \\
    \rowcolor{correctColor}
    \hspace{5pt}\textbf{+ \etthree:} & Stop. \\
    \hline
    \rowcolor{correctColor}
    \textbf{FARE$^4$:} & Stop. \\
    \rowcolor{correctColor}
    \hspace{5pt}\textbf{+ \etthree:} & Stop. \\
\end{tabularx}
\end{tabular}

\vspace{10pt}

\begin{tabular}{c @{\hspace{0.5em}} p{0.65\linewidth}}
\parbox{0.35\linewidth}{\small\textbf{Q:} which program is seen on the screen?} \\ [-20pt]
\raisebox{-0.7\height}
{\includegraphics[width=0.35\linewidth, keepaspectratio]{figs/qual_examples/34881_textvqa_openai_orig.png}} & 
\begin{tabularx}{\linewidth}{l X}
    \rowcolor{headerColor}
    \multicolumn{2}{l}{\textbf{Answer:} Office.} \\
    \rowcolor{incorrectColor}
    \textbf{CLIP:} & Windows xp. \\
    \rowcolor{correctColor}
    \hspace{5pt}\textbf{+ \etthree:} & Office. \\
    \hline
    \rowcolor{incorrectColor}
    \textbf{TeCoA$^2$:} & Windows. \\ 
    \rowcolor{incorrectColor}
    \hspace{5pt}\textbf{+ \etthree:} & Windows. \\
    \hline
    \rowcolor{incorrectColor}
    \textbf{TeCoA$^4$:} & Windows. \\
    \rowcolor{incorrectColor}
    \hspace{5pt}\textbf{+ \etthree:} & Windows. \\
    \hline
    \rowcolor{incorrectColor}
    \textbf{FARE$^2$:} & Flickr. \\
    \rowcolor{correctColor}
    \hspace{5pt}\textbf{+ \etthree:} & Office. \\
    \hline
    \rowcolor{incorrectColor}
    \textbf{FARE$^4$:} & Windows. \\
    \rowcolor{incorrectColor}
    \hspace{5pt}\textbf{+ \etthree:} & Windows. \\
\end{tabularx}
\end{tabular}

\vspace{10pt}

\begin{tabular}{c @{\hspace{0.5em}} p{0.65\linewidth}}
\parbox{0.35\linewidth}{\small\textbf{Q:} which food is being advertised?} \\ [-20pt]
\raisebox{-0.7\height}{\includegraphics[width=0.35\linewidth, keepaspectratio]{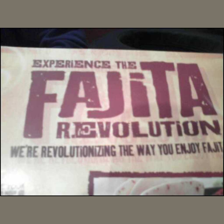}} & 
\begin{tabularx}{\linewidth}{l X}
    \rowcolor{headerColor}
    \multicolumn{2}{l}{\textbf{Answer:} Fajita.} \\
    \rowcolor{incorrectColor}
    \textbf{CLIP:} & Tortilla. \\
    \rowcolor{incorrectColor}
    \hspace{5pt}\textbf{+ \etthree:} & Tortilla. \\
    \hline
    \rowcolor{incorrectColor}
    \textbf{TeCoA$^2$:} & Taco. \\ 
    \rowcolor{correctColor}
    \hspace{5pt}\textbf{+ \etthree:} & Fajita. \\
    \hline
    \rowcolor{incorrectColor}
    \textbf{TeCoA$^4$:} & Taco. \\
    \rowcolor{incorrectColor}
    \hspace{5pt}\textbf{+ \etthree:} & Taco. \\
    \hline
    \rowcolor{incorrectColor}
    \textbf{FARE$^2$:} & Pizza. \\
    \rowcolor{correctColor}
    \hspace{5pt}\textbf{+ \etthree:} & Fajita. \\
    \hline
    \rowcolor{correctColor}
    \textbf{FARE$^4$:} & Fajita. \\
    \rowcolor{correctColor}
    \hspace{5pt}\textbf{+ \etthree:} & Fajita. \\
\end{tabularx}
\end{tabular}

\end{minipage}

&

\begin{minipage}[t]{\linewidth}

\begin{tabular}{c @{\hspace{0.5em}} p{0.65\linewidth}}
\parbox{0.35\linewidth}{\small\textbf{Q:} what brewery \\
makes this beer?} \\ [-20pt]
\raisebox{-0.7\height}{\includegraphics[width=0.35\linewidth, keepaspectratio]{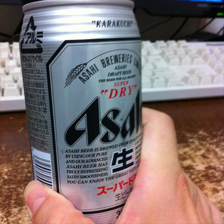}} & 
\begin{tabularx}{\linewidth}{l X}
    \rowcolor{headerColor}
    \multicolumn{2}{l}{\textbf{Answer:} Asahi.} \\
    \rowcolor{incorrectColor}
    \textbf{CLIP:} & Asain. \\
    \rowcolor{correctColor}
    \hspace{5pt}\textbf{+ \etthree:} & Asahi. \\
    \hline
    \rowcolor{correctColor}
    \textbf{TeCoA$^2$:} & Asahi. \\ 
    \rowcolor{correctColor}
    \hspace{5pt}\textbf{+ \etthree:} & Asahi. \\
    \hline
    \rowcolor{correctColor}
    \textbf{TeCoA$^4$:} & Asahi. \\
    \rowcolor{correctColor}
    \hspace{5pt}\textbf{+ \etthree:} & Asahi. \\
    \hline
    \rowcolor{incorrectColor}
    \textbf{FARE$^2$:} & Pabst blue ribbon. \\
    \rowcolor{correctColor}
    \hspace{5pt}\textbf{+ \etthree:} & Asahi. \\
    \hline
    \rowcolor{correctColor}
    \textbf{FARE$^4$:} & Asahi. \\
    \rowcolor{correctColor}
    \hspace{5pt}\textbf{+ \etthree:} & Asahi. \\
\end{tabularx}
\end{tabular}

\vspace{10pt}

\begin{tabular}{c @{\hspace{0.5em}} p{0.65\linewidth}}
\parbox{0.35\linewidth}{\small\textbf{Q:} What color is the vehicle?} \\ [-20pt]
\raisebox{-0.7\height}{\includegraphics[width=0.35\linewidth, keepaspectratio]{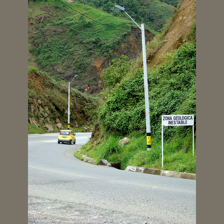}} & 
\begin{tabularx}{\linewidth}{l X}
    \rowcolor{headerColor}
    \multicolumn{2}{l}{\textbf{Answer:} Yellow.} \\
    \rowcolor{incorrectColor}
    \textbf{CLIP:} & Black. \\
    \rowcolor{incorrectColor}
    \hspace{5pt}\textbf{+ \etthree:} & Black. \\
    \hline
    \rowcolor{incorrectColor}
    \textbf{TeCoA$^2$:} & White. \\ 
    \rowcolor{correctColor}
    \hspace{5pt}\textbf{+ \etthree:} & Yellow. \\
    \hline
    \rowcolor{correctColor}
    \textbf{TeCoA$^4$:} & Yellow. \\
    \rowcolor{correctColor}
    \hspace{5pt}\textbf{+ \etthree:} & Yellow. \\
    \hline
    \rowcolor{incorrectColor}
    \textbf{FARE$^2$:} & White. \\
    \rowcolor{correctColor}
    \hspace{5pt}\textbf{+ \etthree:} & Yellow. \\
    \hline
    \rowcolor{incorrectColor}
    \textbf{FARE$^4$:} & White. \\
    \rowcolor{correctColor}
    \hspace{5pt}\textbf{+ \etthree:} & Yellow. \\
\end{tabularx}
\end{tabular}

\vspace{10pt}

\begin{tabular}{c @{\hspace{0.5em}} p{0.65\linewidth}}
\parbox{0.35\linewidth}{\small\textbf{Q:} How many people are there?} \\ [-20pt]
\raisebox{-0.7\height}{\includegraphics[width=0.35\linewidth, keepaspectratio]{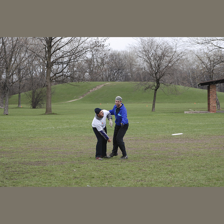}} & 
\begin{tabularx}{\linewidth}{l X}
    \rowcolor{headerColor}
    \multicolumn{2}{l}{\textbf{Answer} 2.} \\
    \rowcolor{incorrectColor}
    \textbf{CLIP:} & 5. \\
    \rowcolor{incorrectColor}
    \hspace{5pt}\textbf{+ \etthree:} & 3. \\
    \hline
    \rowcolor{incorrectColor}
    \textbf{TeCoA$^2$:} & 3. \\ 
    \rowcolor{incorrectColor}
    \hspace{5pt}\textbf{+ \etthree:} & 3. \\
    \hline
    \rowcolor{incorrectColor}
    \textbf{TeCoA$^4$:} & 3. \\
    \rowcolor{correctColor}
    \hspace{5pt}\textbf{+ \etthree:} & 2. \\
    \hline
    \rowcolor{incorrectColor}
    \textbf{FARE$^2$:} & 3. \\
    \rowcolor{incorrectColor}
    \hspace{5pt}\textbf{+ \etthree:} & 3. \\
    \hline
    \rowcolor{incorrectColor}
    \textbf{FARE$^4$:} & 3. \\
    \rowcolor{correctColor}
    \hspace{5pt}\textbf{+ \etthree:} & 2. \\
\end{tabularx}
\end{tabular}

\vspace{10pt}

\begin{tabular}{c @{\hspace{0.5em}} p{0.65\linewidth}}
\parbox{0.35\linewidth}{\small\textbf{Q:} What kind of animal is this?} \\ [-20pt]
\raisebox{-0.7\height}{\includegraphics[width=0.35\linewidth, keepaspectratio]{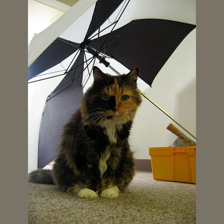}} & 
\begin{tabularx}{\linewidth}{l X}
    \rowcolor{headerColor}
    \multicolumn{2}{l}{\textbf{Answer} Cat.} \\
    \rowcolor{incorrectColor}
    \textbf{CLIP:} & Snake. \\
    \rowcolor{correctColor}
    \hspace{5pt}\textbf{+ \etthree:} & Cat. \\
    \hline
    \rowcolor{incorrectColor}
    \textbf{TeCoA$^2$:} & Dog. \\ 
    \rowcolor{incorrectColor}
    \hspace{5pt}\textbf{+ \etthree:} & Dog. \\
    \hline
    \rowcolor{incorrectColor}
    \textbf{TeCoA$^4$:} & Dog. \\
    \rowcolor{correctColor}
    \hspace{5pt}\textbf{+ \etthree:} & Cat. \\
    \hline
    \rowcolor{incorrectColor}
    \textbf{FARE$^2$:} & Dog. \\
    \rowcolor{correctColor}
    \hspace{5pt}\textbf{+ \etthree:} & Cat. \\
    \hline
    \rowcolor{incorrectColor}
    \textbf{FARE$^4$:} & Dog. \\
    \rowcolor{incorrectColor}
    \hspace{5pt}\textbf{+ \etthree:} & Dog. \\
\end{tabularx}
\end{tabular}

\end{minipage}

\end{tabular}

\vspace{4pt}
\caption{Qualitative comparison across 8 examples with short Q\&A format. \etthree corrects answers affected by adversarial attacks on standard CLIP and further refines the ones produced by robust TeCoA and FARE. Green rows indicate correct captions, while red indicates incorrect ones. All attacks are generated with $\epsilon_a =4/255$.}
\label{fig:qa_comparison}
\end{figure*}

\begin{table*}[t!]

    \caption{\textbf{Zero-shot robustness when using \emph{one-step} \etthree across 14 benchmark datasets 
in the defense-unaware setting.} Comparison of clean and robust accuracy for 
baseline models versus same models augmented with \etthree. Robustness is evaluated 
against Auto-Attack (AA) at $\epsilon_a = 4/255$. Across all datasets, ImageNet-21k labels serve as the reference text embeddings for computing the energy $E(\cdot,\theta)$}
    
    \centering
    \scriptsize
    \setlength{\tabcolsep}{1.8pt}
    \resizebox{\textwidth}{!}{
    \begin{tabular}{cc|>{\columncolor{lightgray}}Cc
    >{\columncolor{lightgray}}Cc
    >{\columncolor{lightgray}}Cc
    >{\columncolor{lightgray}}Cc
    >{\columncolor{lightgray}}Cc
    >{\columncolor{lightgray}}Cc
    >{\columncolor{lightgray}}Cc|c|c}
    \toprule
    Model & Defense & 
    \rotatebox[origin=c]{60}{ImageNet} & 
    \rotatebox[origin=c]{60}{CalTech} & 
    \rotatebox[origin=c]{60}{Cars} & 
    \rotatebox[origin=c]{60}{CIFAR10} & 
    \rotatebox[origin=c]{60}{CIFAR100} & 
    \rotatebox[origin=c]{60}{DTD} & 
    \rotatebox[origin=c]{60}{EuroSAT} & 
    \rotatebox[origin=c]{60}{FGVC} & 
    \rotatebox[origin=c]{60}{Flowers} & 
    \rotatebox[origin=c]{60}{ImageNet-R} & 
    \rotatebox[origin=c]{60}{ImageNet-S} & 
    \rotatebox[origin=c]{60}{PCAM} & 
    \rotatebox[origin=c]{60}{OxfordPets} &  
    \rotatebox[origin=c]{60}{STL-10} & 
    \rotatebox[origin=c]{60}{Avg.} & 
    \rotatebox[origin=c]{60}{Improv.}\\
    \midrule
    \multirow{4}{*}{\shortstack{ViT-L/14 \\(TeCoA)\\$\epsilon_t=4/255$}}
    & None (Clean)& 74.91 & 78.36 & 37.83 & 79.61 & 50.26 & 38.03 & 22.48 & 11.76 & 38.41 & 74.35 & 54.22 & 49.95 & 76.07 & 93.44 & 55.69  & \multirow{2}{*}{\smallred{0.73}} \\
    & +~\etthree (Clean) & 74.50 & 77.85 & 36.54 & 73.55 & 44.90 & 37.93 & 24.91 & 12.03 & 39.21 & 73.41 & 54.89 & 49.98 & 75.63 & 94.06 & 54.96 \\*[3pt]
    & None (Robust) & 44.50 & 60.90 & 8.50 & 37.10 & 21.50 & 16.50 & 6.40 & 2.20 & 12.60 & 41.90 & 32.80 & 45.70 & 55.00 & 74.30 & 32.85  & \multirow{2}{*}{\smallblue{7.71}} \\
    & +~\etthree (Robust) & 52.70 & 64.40 & 11.20 & 54.10 & 32.70 & 21.70 & 18.30 & 5.30 & 21.00 & 49.80 & 40.10 & 51.00 & 62.10 & 83.50 & 40.56 \\
    \midrule
     \multirow{4}{*}{\shortstack{ViT-L/14 \\(FARE)\\$\epsilon_t=4/255$}}
    & None (Clean) & 70.78 & 84.70 & 63.84 & 77.67 & 56.53 & 43.83 & 18.28 & 21.96 & 58.07 & 80.24 & 56.74 & 50.02 & 87.14 & 96.04 & 61.85    & \multirow{2}{*}{\smallred{2.77}} \\
    & +~\etthree (Clean)& 70.11 & 84.14 & 61.62 & 67.43 & 43.18 & 43.46 & 17.04 & 22.05 & 55.91 & 76.88 & 56.39 & 50.02 & 85.53 & 93.41 & 59.08  \\*[3pt]
    & None (Robust)& 34.80 & 64.20 & 12.70 & 34.80 & 20.20 & 17.50 & 11.10 & 3.00 & 12.20 & 40.50 & 30.60 & 52.30 & 50.60 & 74.30 & 32.77   & \multirow{2}{*}{\smallblue{6.39}} \\
    & +~\etthree (Robust) & 42.80 & 68.70 & 18.10 & 47.20 & 30.50 & 24.90 & 14.20 & 6.50 & 20.50 & 46.10 & 38.60 & 52.30 & 57.90 & 80.00 & 39.16 \\
    \midrule
    \multirow{4}{*}{\shortstack{ViT-L/14 \\(TeCoA)\\$\epsilon_t=2/255$}}
    & None (Clean) & 80.11 & 80.67 & 50.08 & 87.53 & 60.69 & 44.36 & 26.06 & 14.04 & 51.80 & 80.12 & 58.43 & 49.89 & 80.02 & 96.08 & 61.42  & \multirow{2}{*}{\smallred{2.90}} \\
    & +~\etthree (Clean) & 78.48 & 79.34 & 42.88 & 76.78 & 49.71 & 42.39 & 29.67 & 15.36 & 48.46 & 76.30 & 57.80 & 49.87 & 76.53 & 95.73 & 58.52 \\*[3pt]
    & None (Robust) & 37.00 & 57.40 & 6.40 & 31.00 & 17.90 & 14.70 & 7.80 & 1.00 & 9.60 & 36.60 & 30.90 & 17.40 & 50.40 & 69.10 & 27.66  & \multirow{2}{*}{\smallblue{12.28}} \\
    & +~\etthree (Robust)& 47.80 & 63.00 & 13.90 & 52.10 & 31.90 & 21.90 & 24.30 & 8.30 & 22.80 & 45.50 & 40.20 & 46.80 & 59.70 & 81.00 & 39.94 \\
    \midrule
     \multirow{4}{*}{\shortstack{ViT-L/14 \\(FARE)\\$\epsilon_t=2/255$}}
    & None (Clean) & 74.48 & 84.77 & 70.53 & 89.52 & 69.13 & 50.05 & 25.39 & 26.70 & 70.60 & 85.52 & 59.72 & 50.01 & 91.06 & 98.47 & 67.57    & \multirow{2}{*}{\smallred{3.48}} \\
    & +~\etthree (Clean) & 73.29 & 83.94 & 65.68 & 80.07 & 53.46 & 47.55 & 28.43 & 25.47 & 64.29 & 81.72 & 58.55 & 50.02 & 88.72 & 96.03 & 64.09 \\*[3pt]
      & None (Robust) & 17.80 & 46.40 & 5.00 & 25.70 & 14.20 & 11.60 & 0.40 & 0.90 & 7.10 & 25.60 & 22.10 & 19.10 & 28.10 & 61.50 & 20.39    & \multirow{2}{*}{\smallblue{11.61}} \\
    & +~\etthree (Robust) & 28.20 & 56.20 & 12.80 & 45.30 & 27.00 & 19.60 & 23.00 & 6.60 & 14.60 & 36.20 & 31.40 & 37.30 & 39.10 & 70.70 & 32.00  \\
    \bottomrule
    \end{tabular}
    } %

    \label{tab:imagnet_1k_one_step_et}
\end{table*}

\subsubsection{Settings for Comparisons with Prior Work}
For the comparisons reported in \textbf{Table~2}  of the main paper, we use the exact same model checkpoints of the robust CLIP model and strictly follow the experimental
settings established by ~\cite{sheng2025r}, ensuring a fair and consistent comparison across all test-time adaptation and test-time augmentation defenses. 

\etthree is applied under the identical settings used by each respective baseline, enabling a direct and principled evaluation. For the base Robust CLIP model and the state-of-the-art image–transformation defense for CLIP, TTC~\cite{Xing_2025_CVPR}, \etthree is used in a zero-shot setting  as exactly described in the method section of our main paper, with no modifications to their original inference pipelines.

The evaluation also includes the standard set of test-time prompt-tuning and augmentation baselines commonly used in CLIP robustness research~\cite{sheng2025r,yoon2024ctpt}. Following~\cite{sheng2025r}, we also include the simple \emph{Ensemble} baseline, which aggregates predictions across multiple augmented views. All methods operate under identical constraints: they use CLIP as the underlying vision–language model and rely exclusively on AugMix~\cite{hendrycks2019augmix} to generate test-time augmentations.

For clarity, these baselines can be grouped into those that rely solely on test-time augmentation (e.g., MTA~\cite{zanella2024test} and Ensemble) and those based on prompt tuning (e.g., R-TPT~\cite{sheng2025r}, TPT~\cite{shu2022test}, and C-TPT~\cite{yoon2024ctpt}). For these baseline methods, the generation of multiple augmented views is an integral and necessary part of their method. To evaluate \etthree under the same input conditions, it must therefore operate on the identical augmented input distribution used by each baseline.

Consequently,  we apply \etthree directly on top of their existing mechanisms. Because \etthree is orthogonal to both test-time augmentation and prompt tuning, it can be integrated without altering the underlying baseline methods. Specifically, \etthree operates exclusively on the visual input space, leaving textual parameters and prompt embeddings unmodified. In all these settings, multiple augmented views are generated per input, and \etthree is applied to these augmented images before it is processed by the baseline method such as Ensemble or R-TPT.

All experiments share a common set of hyperparameters following the implementation of ~\cite{sheng2025r}. The text prompt template is initialized as ``a photo of a''. For prompt-tuning methods, the learnable component consisted of a four-token prompt, updated via a single step with learning rate: $5\times 10^{-3}$ using the Adam optimizer~\cite{kingma2014adam}. The adversarial examples are generated on all the dataset samples using the exact configuration in~\cite{sheng2025r}: \textbf{a 100-step Projected Gradient Descent (PGD) attack~\cite{madry2017towards} with a perturbation budget of $\varepsilon_a = 4.0$}. Note that this specific attack setting is used exclusively for this comparative table; stronger attacks are employed in all other experiments throughout our paper.

The results for MTA~\cite{zanella2024test}, TPT~\cite{shu2022test}, and C-TPT~\cite{yoon2024ctpt} are taken directly from ~\cite{sheng2025r}. For the methods we re-evaluated, we verified that our reproduced results match those reported in~\cite{sheng2025r}; therefore, we rely on their reported numbers for the remaining baselines. For TTC~\cite{Xing_2025_CVPR}, we use the author's official code base with default hyperparameter and evaluate on the same model checkpoint used for the other baselines to ensure comparability.

\begin{table*}[t!]
    \caption{\textbf{Zero-shot robustness of \etthree across 14 benchmark datasets 
in the defense-unaware setting.} Comparison of clean and robust accuracy for 
baseline models versus same models augmented with \etthree. Robustness is evaluated 
against Auto-Attack (AA) at $\epsilon_a = 4/255$. Across all datasets, ImageNet-21k labels serve as the reference text embeddings for computing the energy $E(\cdot,\theta)$}
    \centering
    \scriptsize
    \setlength{\tabcolsep}{1.8pt}
    \resizebox{\textwidth}{!}{
    \begin{tabular}{cc|>{\columncolor{lightgray}}Cc
    >{\columncolor{lightgray}}Cc
    >{\columncolor{lightgray}}Cc
    >{\columncolor{lightgray}}Cc
    >{\columncolor{lightgray}}Cc
    >{\columncolor{lightgray}}Cc
    >{\columncolor{lightgray}}Cc|c|c}
    \toprule
    Model & Defense & 
    \rotatebox[origin=c]{60}{ImageNet} & 
    \rotatebox[origin=c]{60}{CalTech} & 
    \rotatebox[origin=c]{60}{Cars} & 
    \rotatebox[origin=c]{60}{CIFAR10} & 
    \rotatebox[origin=c]{60}{CIFAR100} & 
    \rotatebox[origin=c]{60}{DTD} & 
    \rotatebox[origin=c]{60}{EuroSAT} & 
    \rotatebox[origin=c]{60}{FGVC} & 
    \rotatebox[origin=c]{60}{Flowers} & 
    \rotatebox[origin=c]{60}{ImageNet-R} & 
    \rotatebox[origin=c]{60}{ImageNet-S} & 
    \rotatebox[origin=c]{60}{PCAM} & 
    \rotatebox[origin=c]{60}{OxfordPets} &  
    \rotatebox[origin=c]{60}{STL-10} & 
    \rotatebox[origin=c]{60}{Avg.} & 
    \rotatebox[origin=c]{60}{Improv.}\\
    \midrule
    \multirow{4}{*}{\shortstack{ViT-L/14 \\(TeCoA)\\$\epsilon_t=4/255$}}
    & None (Clean)& 74.91 & 78.36 & 37.83 & 79.61 & 50.26 & 38.03 & 22.48 & 11.76 & 38.41 & 74.35 & 54.22 & 49.95 & 76.07 & 93.44 & 55.69  & \multirow{2}{*}{\smallred{0.98}} \\
    & +~\etthree (Clean) & 74.21 & 77.95 & 35.79 & 73.41 & 45.09 & 37.61 & 23.15 & 12.54 & 39.29 & 72.81 & 54.85 & 50.00 & 75.06 & 94.15 & 54.71 \\*[3pt]
    & None (Robust) & 44.50 & 60.90 & 8.50 & 37.10 & 21.50 & 16.50 & 6.40 & 2.20 & 12.60 & 41.90 & 32.80 & 45.70 & 55.00 & 74.30 & 32.85  & \multirow{2}{*}{\smallblue{8.56}} \\
    & +~\etthree (Robust) & 54.70 & 66.00 & 11.50 & 57.70 & 32.60 & 22.40 & 16.00 & 6.00 & 21.40 & 50.70 & 40.80 & 51.40 & 63.10 & 85.40 & 41.41 \\
    \midrule
     \multirow{4}{*}{\shortstack{ViT-L/14 \\(FARE)\\$\epsilon_t=4/255$}}
    & None (Clean) & 70.78 & 84.70 & 63.84 & 77.67 & 56.53 & 43.83 & 18.28 & 21.96 & 58.07 & 80.24 & 56.74 & 50.02 & 87.14 & 96.04 & 61.85    & \multirow{2}{*}{\smallred{3.26}} \\
    & +~\etthree (Clean)& 69.86 & 83.68 & 60.34 & 66.23 & 42.89 & 42.98 & 18.20 & 21.48 & 54.82 & 76.00 & 56.25 & 50.02 & 84.93 & 92.56 & 58.59  \\*[3pt]
    & None (Robust)& 34.80 & 64.20 & 12.70 & 34.80 & 20.20 & 17.50 & 11.10 & 3.00 & 12.20 & 40.50 & 30.60 & 52.30 & 50.60 & 74.30 & 32.77   & \multirow{2}{*}{\smallblue{6.29}} \\
    & +~\etthree (Robust)& 42.10 & 69.00 & 17.80 & 47.10 & 30.30 & 24.50 & 13.70 & 6.50 & 20.90 & 46.30 & 38.30 & 52.30 & 57.50 & 80.60 & 39.06 \\
    \midrule
    \multirow{4}{*}{\shortstack{ViT-L/14 \\(TeCoA)\\$\epsilon_t=2/255$}}
    & None (Clean) & 80.11 & 80.67 & 50.08 & 87.53 & 60.69 & 44.36 & 26.06 & 14.04 & 51.80 & 80.12 & 58.43 & 49.89 & 80.02 & 96.08 & 61.42  & \multirow{2}{*}{\smallred{3.09}} \\
    & +~\etthree (Clean) & 77.79 & 78.98 & 40.74 & 78.50 & 50.55 & 42.50 & 29.94 & 15.21 & 47.99 & 75.40 & 57.29 & 49.97 & 75.91 & 95.79 & 58.33 \\*[3pt]
    & None (Robust) & 37.00 & 57.40 & 6.40 & 31.00 & 17.90 & 14.70 & 7.80 & 1.00 & 9.60 & 36.60 & 30.90 & 17.40 & 50.40 & 69.10 & 27.66  & \multirow{2}{*}{\smallblue{12.30}} \\
    & +~\etthree (Robust)&  47.40 & 63.50 & 13.10 & 51.20 & 31.40 & 21.60 & 22.60 & 8.80 & 24.00 & 47.00 & 40.50 & 46.30 & 59.80 & 82.20 & 39.96 \\
    \midrule
     \multirow{4}{*}{\shortstack{ViT-L/14 \\(FARE)\\$\epsilon_t=2/255$}}
    & None (Clean) & 74.48 & 84.77 & 70.53 & 89.52 & 69.13 & 50.05 & 25.39 & 26.70 & 70.60 & 85.52 & 59.72 & 50.01 & 91.06 & 98.47 & 67.57    & \multirow{2}{*}{\smallred{4.00}} \\
    & +~\etthree (Clean) & 72.90 & 83.75 & 63.93 & 78.19 & 53.22 & 46.70 & 30.61 & 24.81 & 63.05 & 80.68 & 58.18 & 50.02 & 88.31 & 95.59 & 63.57 \\*[3pt]
      & None (Robust) & 17.80 & 46.40 & 5.00 & 25.70 & 14.20 & 11.60 & 0.40 & 0.90 & 7.10 & 25.60 & 22.10 & 19.10 & 28.10 & 61.50 & 20.39    & \multirow{2}{*}{\smallblue{11.74}} \\
    & +~\etthree (Robust) & 27.10 & 56.40 & 12.70 & 49.10 & 27.80 & 20.00 & 19.20 & 5.80 & 15.40 & 37.00 & 32.10 & 36.30 & 38.80 & 72.10 & 32.13  \\
    \bottomrule
    \end{tabular}
    } %

    \label{tab:imagnet_2_step_et}
\end{table*}

\subsection{Details on LVLM evaluation}\label{app:eval_lvlm}
In addition to evaluating robustness on zero-shot classification with CLIP models, we extend our analysis to Large Vision-Language Models (LVLMs) that employ these CLIP models as visual encoders, following the approach of prior work~\cite{schlarmann2024robust}. We specifically examine the susceptibility of the visual modality to adversarial perturbations and seek to enhance robustness against such attacks. Consistent with the procedure described in the Method section of the main paper, \etthree is applied exclusively to the visual encoder, offering a fast and computationally efficient transformation. As shown in Figure 2 of the main paper, embeddings are extracted from the CLIP visual encoder after the \etthree transformation. Following the original LLAVA implementation, we use the feature obtained from the layer before the last layer of the visual encoder.\\
\minisection{Attack setup}
We adopt the \textit{ensemble} adversarial evaluation procedure introduced in~\cite{schlarmann2024robust}. 
For each test instance, we run a sequence of APGD attacks ($\ell_\infty$, $\epsilon = 4/255$, 100 steps) with different initialization conditions. 
In captioning tasks, we retain the perturbation that yields the lowest CIDEr score; 
in VQA tasks, we retain the perturbation that yields the lowest answer accuracy. 
The procedure begins with clean inference, followed by five APGD runs initialized from different ground-truth references, and concludes with a refinement step initialized from the current best perturbation. 
After each round, if the newly generated output worsens the evaluation metric, the perturbation is kept. 
If the metric crosses a stopping threshold (low CIDEr or zero VQA accuracy), the attack is terminated early for that sample.
For each evaluation setting, we report both the original model performance and the performance obtained when applying \etthree at test time on the same adversarial inputs. Consistent with prior work~\cite{schlarmann2024robust},  we use randomly sampled 500 images from each dataset
for the adversarial evaluations, and all clean samples for
clean evaluations.

\subsection{Computational overhead on LLaVA with \etthree}\label{app:comp_time}
We measure the inference latency of \etthree on an NVIDIA A100 GPU with the LLaVA-1.5 7B model. The baseline inference time is 593.9 ms per sample. Incorporating the \etthree step increases the latency to 607.3 ms per sample ($+2.3\%$) for a single step and 640.0 ms per sample ($+7.7\%$) for two steps. These results are averaged over 500 samples.

\subsection{Details on Robust Classifiers}

We evaluate the robust ImageNet classifiers obtained from RobustBench~\cite{croce2021robustbench} under a standard $\ell_\infty$ threat model with perturbation budget $\varepsilon = 4/255$, in accordance with the RobustBench evaluation protocol~\cite{croce2021robustbench}. Clean accuracy is reported on the full validation set, while robust accuracy is computed on a subset of 1,000 randomly selected images, following the standard practice in prior works. We also provide additional experiments in \cref{app:additionalclassifier} besides the one presented in the main paper.

\subsection{Details on Adaptive attacks}\label{app:adaptive_class}

To rigorously assess the robustness of our defense \etthree, we evaluate it on robust image classifiers under adaptive attacks tailored to test‑time defenses, ensuring a fair and meaningful comparison. We present the results in Table~4 of the main paper and provide additional details here. Specifically, we follow exactly the protocol proposed for test-time defenses by \cite{croce2022evaluating}, adopting their ``Transfer APGD-T + BPDA'' attack. We exactly follow their implementation and use APGD-T with the DLR loss, 5 restarts, and 100 iterations per restart. We approximate gradients through our defense using Backward Pass Differentiable Approximation (BPDA), following the standard practice in \cite{croce2022evaluating}. This is necessary because directly differentiating through the full unrolled iterative procedure of our defense leads to gradient shattering and memory instabilities. Such effects can artificially impair the attacker and result in gradient obfuscation. By replacing the backward pass with a stable surrogate, BPDA provides a reliable gradient signal. In addition, we also perform \textbf{transfer attacks} by generating adversarial perturbations using APGD-T (with DLR loss, 5 restarts, and 100 iterations per restart) on the underlying static base model (without \etthree) and applying them to the defended model to evaluate whether the defense provides genuine robustness beyond gradient masking. We report \textbf{worst-case} robust accuracy, where a sample is considered misclassified if it is successfully misclassified by either the adaptive attack (APGD-T + BPDA) or the transfer attack.
 Note that we do not use EOT~\cite{athalye2018obfuscated} in evaluations as used in \cite{croce2022evaluating}, as \etthree does not introduce stochasticity. 
We specifically adopt the Transfer APGD-T + BPDA attack specifically because it has been demonstrated in \cite{croce2022evaluating} to reliably circumvent most iterative test-time defenses similar to ours.

\section{Qualitative examples}\label{app:blank}

We provide a series of qualitative examples to illustrate the effect of \etthree across captioning, question answering, and image classification. These examples demonstrate how \etthree mitigates the effect of adversarial perturbations across various tasks discussed in the paper.

We present \cref{fig:qualitative_example_cap} which shows qualitative comparisons of generated captions for several sample images. \etthree consistently mitigates the impact of adversarial perturbations on standard CLIP and improves the robustness of both TeCoA and FARE. Green rows indicate semantically correct captions, red rows denote incorrect captions, and yellow rows correspond to partially correct descriptions that still broadly reflect the scene. All adversarial examples are generated with $\epsilon_a = 4/255$.

\cref{fig:qa_comparison} shows short question–answer evaluations across various representative images. As with captioning, \etthree corrects adversarially induced errors in standard CLIP and further refines outputs from TeCoA and FARE. Green rows correspond to correct predictions, while red rows indicate incorrect ones. All adversarial samples use $\epsilon_a = 4/255$.

We also analyze the effect of \etthree on classification. To better understand the effect of perturbations, \cref{fig:logits} plots logit as perturbations are smoothly scaled. For each example, we interpolate linearly from $0\%$ to $100\%$ of the perturbation in 100 steps. The top row shows the model’s logit evolution under the adversarial perturbation, while the bottom row shows the corresponding evolution under the \etthree transformation.

Finally, \cref{fig:logits_bunny} specifically illustrates how \etthree enhances salient, class-relevant features. The left panel shows an Angora bunny image originally misclassified as a Blue Tick. Applying \etthree highlights key attributes—most notably the pinkish eye region—allowing the model to recover the correct prediction. The right panel provides a clean reference image for comparison. Similar behavior is observed throughout the paper, including the teaser example where \etthree makes a snake’s eye features more prominent—features absent from its adversarial misclassifications as a zucchini. These examples collectively illustrate that \etthree transformation amplifies discriminative, class-relevant features, enabling recovery from adversarial perturbation.

\section{Additional Experimental Results}\label{app:additionalresults}

\subsection{Zero-shot Robustness with additional models}
As shown in Table~1 of the main paper, \etthree improves zero-shot robustness. Here, we report analogous results on additional CLIP models under same settings, demonstrating that the observed improvements hold consistently across a broader set of models.
The \cref{tab:more_robust_models} reports zero-shot performance of \etthree across a diverse set architectures, including transformer-based ViT and ConvNeXt models, as well as models, as well as models trained with a smaller $\epsilon_t =1/255$ perturbation. Across all configurations, \etthree consistently improves robust accuracy, with minimal or modest impact on clean accuracy. These results demonstrate that the benefits of \etthree are consistent across model architectures and training configurations, further illustrating its effectiveness in enhancing zero-shot robustness. 

\subsection{Robustness with larger classifier architectures} \label{app:additionalclassifier} 

We further evaluate the effectiveness of \etthree on an extended set of robust ImageNet classifiers obtained from RobustBench, shown in \cref{tab:classifier_et3}. Specifically, we include larger and architecturally distinct models, such as Swin Transformers and ConvNeXt variants, to assess the generality of \etthree. For a fair comparison, we evaluate the base models with the same attack used to assess the \etthree, as this protocol is best suited for test-time defenses in classifiers, following \cite{croce2022evaluating}. We use APGD-T with DLR loss, 5 restarts, and 100 iterations per restart, following the attack protocol described previously. Clean accuracy is measured on the full ImageNet validation set, while robust accuracy is computed on a 1,000-image subset, consistent with established evaluation practices. We also report robust accuracy obtained with AutoAttack on the same samples.

Across all evaluated classifiers, \etthree consistently enhances robust accuracy, with only minor reductions in clean accuracy.

\subsection{Defense-aware attacks for LVLM}\label{app:adap_LVLM} 

To rigorously assess the reliability of \etthree in the context of Large Vision-Language Models (LVLMs), we evaluate its performance against defense-aware adaptive attacks. Specifically, we employ Backward Pass Differentiable Approximation (BPDA) combined with AutoAttack, maintaining the exact same LLaVA evaluation settings as those established in Table~3. 

In \cref{tab:vlm_adaptive_l2} we report the average robust accuracy across the four evaluation datasets, utilizing a total of 200 samples. When using a standard, non-robust CLIP model as the underlying image encoder, we observe a sharp drop in robustness against these adaptive attacks. Conversely, when integrated with robustly trained CLIP encoders, \etthree consistently boosts the model's overall robustness. Crucially, the results of this adaptive evaluation explicitly validate the theoretical analysis presented in~\cref{sec:provable_defense_robust}. This empirical evidence confirms that \etthree effectively leverages and amplifies the inherent robustness of the foundation encoder, successfully mitigating stronger defense-aware attacks. Under the same settings, we further evaluate \etthree when it uses $\ell_\infty$ projection instead the default $\ell_2$ projection, more details are provided in \cref{app:budget}.

\begin{table*}[t]
\caption{\textbf{\etthree improves robustness across increasing attack strengths in defense-unaware setting.} We report clean and robust accuracy on 14 datasets as the attack strength increases, comparing the baseline model to its ET3-augmented variant. $\epsilon_a$ indicates the strength of the attack. Across all datasets, ImageNet-21k labels serve as the reference text embeddings for computing the energy $E(\cdot,\theta)$.}
    \centering
    \begin{subtable}{\textwidth}
        \centering
        \caption{ViT-L/14 TeCoA ($\epsilon_t=2/255$)}
        \scriptsize
        \setlength{\tabcolsep}{1.8pt}
        \resizebox{\textwidth}{!}{
        \begin{tabular}{cc|>{\columncolor{lightgray}}Cc
        >{\columncolor{lightgray}}Cc
        >{\columncolor{lightgray}}Cc
        >{\columncolor{lightgray}}Cc
        >{\columncolor{lightgray}}Cc
        >{\columncolor{lightgray}}Cc
        >{\columncolor{lightgray}}Cc|c|c}
        \toprule
        $\epsilon_a$ & Defense & 
        \rotatebox[origin=c]{60}{ImageNet} & 
        \rotatebox[origin=c]{60}{CalTech} & 
        \rotatebox[origin=c]{60}{Cars} & 
        \rotatebox[origin=c]{60}{CIFAR10} & 
        \rotatebox[origin=c]{60}{CIFAR100} & 
        \rotatebox[origin=c]{60}{DTD} & 
        \rotatebox[origin=c]{60}{EuroSAT} & 
        \rotatebox[origin=c]{60}{FGVC} & 
        \rotatebox[origin=c]{60}{Flowers} & 
        \rotatebox[origin=c]{60}{ImageNet-R} & 
        \rotatebox[origin=c]{60}{ImageNet-S} & 
        \rotatebox[origin=c]{60}{PCAM} & 
        \rotatebox[origin=c]{60}{OxfordPets} &  
        \rotatebox[origin=c]{60}{STL-10} & 
        \rotatebox[origin=c]{60}{Avg.} & 
        \rotatebox[origin=c]{60}{Improv.}\\
        \midrule
        \multirow{3}{*}{Clean Data}
        & None & 80.11 & 80.67 & 50.08 & 87.53 & 60.69 & 44.36 & 26.06 & 14.04 & 51.80 & 80.12 & 58.43 & 49.89 & 80.02 & 96.08 & 61.42 & \multirow{3}{*}{\smallred{3.09}} \\
        & +\etthree & 77.79 & 78.98 & 40.74 & 78.50 & 50.55 & 42.50 & 29.94 & 15.21 & 47.99 & 75.40 & 57.29 & 49.97 & 75.91 & 95.79 & 58.33 & \\
        \midrule
        \multirow{3}{*}{$2/255$}
        & None & 61.90 & 70.20 & 21.90 & 63.50 & 34.90 & 27.10 & 12.60 & 6.40 & 27.50 & 58.70 & 43.00 & 42.60 & 69.60 & 88.60 & 44.89 & \multirow{3}{*}{\smallblue{7.73}} \\
        & +\etthree & 69.70 & 73.70 & 30.60 & 72.40 & 44.70 & 34.30 & 27.30 & 12.10 & 38.00 & 64.90 & 50.50 & 51.70 & 73.90 & 92.90 & 52.62 & \\
        \midrule
        \multirow{3}{*}{$4/255$}

         & None & 37.00 & 57.40 & 6.40 & 31.00 & 17.90 & 14.70 & 7.80 & 1.00 & 9.60 & 36.60 & 30.90 & 17.40 & 50.40 & 69.10 & 27.66  & \multirow{2}{*}{\smallblue{12.28}} \\
        & +~\etthree& 47.80 & 63.00 & 13.90 & 52.10 & 31.90 & 21.90 & 24.30 & 8.30 & 22.80 & 45.50 & 40.20 & 46.80 & 59.70 & 81.00 & 39.94 \\
        \midrule
        \multirow{3}{*}{$6/255$}
        & None & 16.30 & 36.00 & 1.40 & 11.90 & 6.80 & 7.90 & 0.00 & 0.20 & 2.80 & 20.60 & 21.30 & 1.70 & 21.60 & 41.10 & 13.54 & \multirow{3}{*}{\smallblue{12.70}} \\
        & \etthree & 27.40 & 45.00 & 7.70 & 30.20 & 20.10 & 14.10 & 18.90 & 6.10 & 13.20 & 29.40 & 29.30 & 32.70 & 35.80 & 57.50 & 26.24 \\
        \midrule
        \multirow{3}{*}{$8/255$}
        & Base & 4.70 & 18.40 & 0.30 & 2.70 & 2.20 & 2.90 & 0.00 & 0.00 & 1.00 & 10.80 & 14.20 & 0.10 & 4.70 & 14.90 & 5.49 & \multirow{3}{*}{\smallblue{9.54}} \\
        & \etthree & 13.90 & 28.30 & 4.80 & 16.40 & 13.20 & 8.60 & 10.00 & 4.10 & 8.40 & 18.00 & 22.20 & 16.80 & 15.40 & 30.30 & 15.03 \\
        \midrule
        \multirow{3}{*}{$10/255$}
        & Base& 1.00 & 8.80 & 0.00 & 0.30 & 0.70 & 1.10 & 0.00 & 0.00 & 0.00 & 6.40 & 9.60 & 0.00 & 0.30 & 4.10 & 2.31 & \multirow{3}{*}{\smallblue{7.00}} \\
        & \etthree & 7.30 & 15.60 & 4.10 & 9.50 & 9.50 & 5.00 & 9.70 & 3.90 & 6.00 & 11.80 & 16.70 & 8.50 & 7.70 & 15.00 & 9.31 \\
        \bottomrule
        \end{tabular}
        } %
    \end{subtable}
    
    \vspace{1.5em}
    
    \begin{subtable}{\textwidth}
        \centering
        \caption{ViT-L/14 TeCoA ($\epsilon_t=4/255$)}
        \scriptsize
        \setlength{\tabcolsep}{1.8pt}
        \resizebox{\textwidth}{!}{
        \begin{tabular}{cc|>{\columncolor{lightgray}}Cc
        >{\columncolor{lightgray}}Cc
        >{\columncolor{lightgray}}Cc
        >{\columncolor{lightgray}}Cc
        >{\columncolor{lightgray}}Cc
        >{\columncolor{lightgray}}Cc
        >{\columncolor{lightgray}}Cc|c|c}
        \toprule
        $\epsilon_a$ & Defense & 
        \rotatebox[origin=c]{60}{ImageNet} & 
        \rotatebox[origin=c]{60}{CalTech} & 
        \rotatebox[origin=c]{60}{Cars} & 
        \rotatebox[origin=c]{60}{CIFAR10} & 
        \rotatebox[origin=c]{60}{CIFAR100} & 
        \rotatebox[origin=c]{60}{DTD} & 
        \rotatebox[origin=c]{60}{EuroSAT} & 
        \rotatebox[origin=c]{60}{FGVC} & 
        \rotatebox[origin=c]{60}{Flowers} & 
        \rotatebox[origin=c]{60}{ImageNet-R} & 
        \rotatebox[origin=c]{60}{ImageNet-S} & 
        \rotatebox[origin=c]{60}{PCAM} & 
        \rotatebox[origin=c]{60}{OxfordPets} &  
        \rotatebox[origin=c]{60}{STL-10} & 
        \rotatebox[origin=c]{60}{Avg.} & 
        \rotatebox[origin=c]{60}{Improv.}\\
        \midrule
        \multirow{3}{*}{Clean Data}
        & None & 74.91 & 78.36 & 37.83 & 79.61 & 50.26 & 38.03 & 22.48 & 11.76 & 38.41 & 74.35 & 54.22 & 49.95 & 76.07 & 93.44 & 55.69 & \multirow{3}{*}{\smallred{0.98}} \\
        & +\etthree & 74.21 & 77.95 & 35.79 & 73.41 & 45.09 & 37.61 & 23.15 & 12.54 & 39.29 & 72.81 & 54.85 & 50.00 & 75.06 & 94.15 & 54.71 & \\
        \midrule
        \multirow{3}{*}{$2/255$}
        & None& 59.20 & 69.70 & 18.10 & 59.60 & 33.60 & 26.50 & 7.90 & 5.60 & 23.90 & 59.10 & 42.90 & 51.10 & 68.00 & 86.80 & 43.71 & \multirow{3}{*}{\smallblue{6.08}} \\
        & + \etthree & 68.30 & 73.40 & 23.30 & 69.10 & 41.20 & 30.40 & 19.60 & 9.50 & 30.70 & 65.40 & 49.50 & 52.10 & 71.70 & 92.80 & 49.79 & \\
        \midrule
        \multirow{3}{*}{$4/255$}
           & None & 44.50 & 60.90 & 8.50 & 37.10 & 21.50 & 16.50 & 6.40 & 2.20 & 12.60 & 41.90 & 32.80 & 45.70 & 55.00 & 74.30 & 32.85  & \multirow{2}{*}{\smallblue{8.56}} \\
         & +~\etthree & 54.70 & 66.00 & 11.50 & 57.70 & 32.60 & 22.40 & 16.00 & 6.00 & 21.40 & 50.70 & 40.80 & 51.40 & 63.10 & 85.40 & 41.41 \\
        \midrule
        \multirow{3}{*}{$6/255$}
        & None & 27.50 & 49.40 & 3.40 & 19.80 & 11.50 & 11.30 & 0.20 & 0.50 & 5.80 & 29.40 & 25.30 & 34.00 & 37.30 & 55.70 & 22.22 & \multirow{3}{*}{\smallblue{9.44}} \\
        & \etthree & 37.80 & 56.00 & 6.20 & 39.70 & 24.40 & 15.40 & 12.30 & 3.50 & 13.20 & 36.10 & 32.40 & 48.30 & 47.90 & 70.00 & 31.66 \\
        \midrule
        \multirow{3}{*}{$8/255$}
        & Base & 15.40 & 33.70 & 0.60 & 9.20 & 5.80 & 6.70 & 0.00 & 0.00 & 2.60 & 17.30 & 17.90 & 11.00 & 16.90 & 35.40 & 12.32 & \multirow{3}{*}{\smallblue{9.33}} \\
        & \etthree  & 24.50 & 41.00 & 3.50 & 24.60 & 17.80 & 10.80 & 4.20 & 2.20 & 9.40 & 24.40 & 25.80 & 36.60 & 28.30 & 50.00 & 21.65 \\
        \midrule
        \multirow{3}{*}{$10/255$}
        & Base & 6.10 & 20.90 & 0.20 & 2.80 & 2.40 & 3.80 & 0.00 & 0.00 & 1.10 & 10.90 & 13.00 & 0.70 & 5.30 & 14.90 & 5.86 & \multirow{3}{*}{\smallblue{7.67}} \\
        & \etthree & 14.60 & 26.30 & 1.90 & 15.30 & 13.60 & 7.20 & 3.80 & 2.00 & 6.40 & 17.00 & 19.70 & 19.40 & 13.60 & 28.60 & 13.53 \\
        \bottomrule
        \end{tabular}
        } %
    \end{subtable}

    \label{tab:attack_strength_ablation_2_step}
\end{table*}

\begin{table*}[t]
\caption{\textbf{\emph{One-step} \etthree improves robustness across increasing attack strengths in defense-unaware setting.} We report clean and robust accuracy on 14 datasets as the attack strength increases, comparing the baseline model to its ET3-augmented variant. $\epsilon_a$ indicates the strength of attack. Across all datasets, ImageNet-21k labels serve as the reference text embeddings for computing the energy $E(\cdot,\theta)$.}
    \centering
    \begin{subtable}{\textwidth}
        \centering
        \caption{ViT-L/14 TeCoA ($\epsilon_t=2/255$)}
        \scriptsize
        \setlength{\tabcolsep}{1.8pt}
        \resizebox{\textwidth}{!}{
        \begin{tabular}{cc|>{\columncolor{lightgray}}Cc
        >{\columncolor{lightgray}}Cc
        >{\columncolor{lightgray}}Cc
        >{\columncolor{lightgray}}Cc
        >{\columncolor{lightgray}}Cc
        >{\columncolor{lightgray}}Cc
        >{\columncolor{lightgray}}Cc|c|c}
        \toprule
        $\epsilon_a$ & Defense & 
        \rotatebox[origin=c]{60}{ImageNet} & 
        \rotatebox[origin=c]{60}{CalTech} & 
        \rotatebox[origin=c]{60}{Cars} & 
        \rotatebox[origin=c]{60}{CIFAR10} & 
        \rotatebox[origin=c]{60}{CIFAR100} & 
        \rotatebox[origin=c]{60}{DTD} & 
        \rotatebox[origin=c]{60}{EuroSAT} & 
        \rotatebox[origin=c]{60}{FGVC} & 
        \rotatebox[origin=c]{60}{Flowers} & 
        \rotatebox[origin=c]{60}{ImageNet-R} & 
        \rotatebox[origin=c]{60}{ImageNet-S} & 
        \rotatebox[origin=c]{60}{PCAM} & 
        \rotatebox[origin=c]{60}{OxfordPets} &  
        \rotatebox[origin=c]{60}{STL-10} & 
        \rotatebox[origin=c]{60}{Avg.} & 
        \rotatebox[origin=c]{60}{Improv.}\\
        \midrule
        \multirow{3}{*}{Clean Data}
        & None & 80.11 & 80.67 & 50.08 & 87.53 & 60.69 & 44.36 & 26.06 & 14.04 & 51.80 & 80.12 & 58.43 & 49.89 & 80.02 & 96.08 & 61.42 & \multirow{3}{*}{\smallred{6.08}} \\
        & +\etthree & 75.79 & 77.65 & 31.75 & 72.03 & 44.36 & 39.36 & 32.78 & 12.93 & 42.25 & 72.44 & 56.04 & 50.11 & 72.39 & 94.83 & 55.34 & \\
        \midrule
        \multirow{3}{*}{$2/255$}
        & None & 61.90 & 70.20 & 21.90 & 63.50 & 34.90 & 27.10 & 12.60 & 6.40 & 27.50 & 58.70 & 43.00 & 42.60 & 69.60 & 88.60 & 44.89 & \multirow{3}{*}{\smallblue{6.35}} \\
        & +\etthree & 67.90 & 73.20 & 25.70 & 66.10 & 39.40 & 34.20 & 31.80 & 11.70 & 36.60 & 64.20 & 50.80 & 52.40 & 71.20 & 92.10 & 51.24 & \\
        \midrule
        \multirow{3}{*}{$4/255$}

         & None & 37.00 & 57.40 & 6.40 & 31.00 & 17.90 & 14.70 & 7.80 & 1.00 & 9.60 & 36.60 & 30.90 & 17.40 & 50.40 & 69.10 & 27.66  & \multirow{3}{*}{\smallblue{13.06}} \\
        & +~\etthree & 48.50 & 62.90 & 15.30 & 48.60 & 29.50 & 24.00 & 28.60 & 8.90 & 24.20 & 46.50 & 42.10 & 50.70 & 59.30 & 81.00 & 40.72 \\
        \midrule
        \multirow{3}{*}{$6/255$}
        & None & 16.30 & 36.00 & 1.40 & 11.90 & 6.80 & 7.90 & 0.00 & 0.20 & 2.80 & 20.60 & 21.30 & 1.70 & 21.60 & 41.10 & 13.54 & \multirow{3}{*}{\smallblue{15.20}} \\
        & \etthree & 29.40 & 46.60 & 8.60 & 30.10 & 19.60 & 16.40 & 23.90 & 6.90 & 15.50 & 30.60 & 31.70 & 45.30 & 39.10 & 58.60 & 28.74 \\
        \midrule
        \multirow{3}{*}{$8/255$}
        & Base & 4.70 & 18.40 & 0.30 & 2.70 & 2.20 & 2.90 & 0.00 & 0.00 & 1.00 & 10.80 & 14.20 & 0.10 & 4.70 & 14.90 & 5.49 & \multirow{3}{*}{\smallblue{12.60}} \\
        & \etthree & 15.50 & 29.70 & 5.60 & 16.70 & 13.50 & 9.80 & 14.80 & 5.80 & 10.90 & 18.70 & 24.90 & 36.60 & 18.10 & 32.60 & 18.09 \\
        \midrule
        \multirow{3}{*}{$10/255$}
        & Base& 1.00 & 8.80 & 0.00 & 0.30 & 0.70 & 1.10 & 0.00 & 0.00 & 0.00 & 6.40 & 9.60 & 0.00 & 0.30 & 4.10 & 2.31 & \multirow{3}{*}{\smallblue{10.19}} \\
        & \etthree & 9.50 & 17.60 & 4.60 & 10.10 & 9.60 & 6.50 & 13.80 & 4.70 & 8.00 & 13.40 & 19.30 & 29.90 & 10.30 & 17.70 & 12.50 \\
        \bottomrule
        \end{tabular}
        } %
    \end{subtable}
    
    \vspace{1.5em}
    
    \begin{subtable}{\textwidth}
        \centering
        \caption{ViT-L/14 TeCoA ($\epsilon_t=4/255$)}
        \scriptsize
        \setlength{\tabcolsep}{1.8pt}
        \resizebox{\textwidth}{!}{
        \begin{tabular}{cc|>{\columncolor{lightgray}}Cc
        >{\columncolor{lightgray}}Cc
        >{\columncolor{lightgray}}Cc
        >{\columncolor{lightgray}}Cc
        >{\columncolor{lightgray}}Cc
        >{\columncolor{lightgray}}Cc
        >{\columncolor{lightgray}}Cc|c|c}
        \toprule
        $\epsilon_a$ & Defense & 
        \rotatebox[origin=c]{60}{ImageNet} & 
        \rotatebox[origin=c]{60}{CalTech} & 
        \rotatebox[origin=c]{60}{Cars} & 
        \rotatebox[origin=c]{60}{CIFAR10} & 
        \rotatebox[origin=c]{60}{CIFAR100} & 
        \rotatebox[origin=c]{60}{DTD} & 
        \rotatebox[origin=c]{60}{EuroSAT} & 
        \rotatebox[origin=c]{60}{FGVC} & 
        \rotatebox[origin=c]{60}{Flowers} & 
        \rotatebox[origin=c]{60}{ImageNet-R} & 
        \rotatebox[origin=c]{60}{ImageNet-S} & 
        \rotatebox[origin=c]{60}{PCAM} & 
        \rotatebox[origin=c]{60}{OxfordPets} &  
        \rotatebox[origin=c]{60}{STL-10} & 
        \rotatebox[origin=c]{60}{Avg.} & 
        \rotatebox[origin=c]{60}{Improv.}\\
        \midrule
        \multirow{3}{*}{Clean Data}
        & None & 74.91 & 78.36 & 37.83 & 79.61 & 50.26 & 38.03 & 22.48 & 11.76 & 38.41 & 74.35 & 54.22 & 49.95 & 76.07 & 93.44 & 55.69 & \multirow{3}{*}{\smallred{2.17}} \\
        & +\etthree & 72.75 & 77.12 & 32.82 & 69.83 & 41.19 & 36.01 & 26.07 & 12.42 & 38.05 & 71.13 & 54.31 & 50.01 & 73.81 & 93.71 & 53.52 & \\
        \midrule
        \multirow{3}{*}{$2/255$}
        & None& 59.20 & 69.70 & 18.10 & 59.60 & 33.60 & 26.50 & 7.90 & 5.60 & 23.90 & 59.10 & 42.90 & 51.10 & 68.00 & 86.80 & 43.71 & \multirow{3}{*}{\smallblue{6.44}} \\
        & + \etthree & 68.00 & 73.30 & 24.60 & 66.30 & 37.30 & 31.40 & 25.20 & 10.50 & 33.60 & 66.30 & 49.40 & 52.20 & 71.70 & 92.30 & 50.15 & \\
        \midrule
        \multirow{3}{*}{$4/255$}
           & None & 44.50 & 60.90 & 8.50 & 37.10 & 21.50 & 16.50 & 6.40 & 2.20 & 12.60 & 41.90 & 32.80 & 45.70 & 55.00 & 74.30 & 32.85  & \multirow{3}{*}{\smallblue{9.90}} \\
         & +~\etthree & 55.30 & 66.80 & 13.30 & 56.10 & 31.50 & 24.90 & 22.40 & 7.30 & 23.40 & 52.20 & 42.40 & 52.20 & 64.70 & 86.00 & 42.75 \\
        \midrule
        \multirow{3}{*}{$6/255$}
        & None & 27.50 & 49.40 & 3.40 & 19.80 & 11.50 & 11.30 & 0.20 & 0.50 & 5.80 & 29.40 & 25.30 & 34.00 & 37.30 & 55.70 & 22.22 & \multirow{3}{*}{\smallblue{11.43}} \\
        & \etthree & 39.90 & 56.40 & 7.30 & 41.00 & 23.80 & 16.90 & 18.20 & 5.30 & 17.40 & 38.50 & 33.80 & 51.10 & 50.10 & 71.40 & 33.65 \\
        \midrule
        \multirow{3}{*}{$8/255$}
        & Base & 15.40 & 33.70 & 0.60 & 9.20 & 5.80 & 6.70 & 0.00 & 0.00 & 2.60 & 17.30 & 17.90 & 11.00 & 16.90 & 35.40 & 12.32 & \multirow{3}{*}{\smallblue{12.25}} \\
        & \etthree  & 25.70 & 41.90 & 4.50 & 26.90 & 18.00 & 13.20 & 12.90 & 4.10 & 12.40 & 26.70 & 27.50 & 46.10 & 32.10 & 52.00 & 24.57 \\
        \midrule
        \multirow{3}{*}{$10/255$}
        & Base & 6.10 & 20.90 & 0.20 & 2.80 & 2.40 & 3.80 & 0.00 & 0.00 & 1.10 & 10.90 & 13.00 & 0.70 & 5.30 & 14.90 & 5.86 & \multirow{3}{*}{\smallblue{10.98}} \\
        & \etthree & 15.70 & 30.20 & 3.50 & 16.70 & 13.70 & 9.30 & 8.60 & 3.50 & 9.90 & 18.80 & 21.00 & 37.50 & 16.70 & 30.70 & 16.84 \\
        \bottomrule
        \end{tabular}
        } %
    \end{subtable}

    \label{tab:attack_strength_ablation_1_step}
\end{table*}

\section{Ablation Study}\label{app:ablation}
In this section, we provide additional ablation studies to further analyze the behavior and key design choices of our proposed defense, \etthree.

\subsection{Single-Step \etthree Defense}

Our proposed \etthree method uses a small, fixed number of iterative steps to perform the energy minimization. To demonstrate that \etthree\ can be made faster at inference if needed,  we conduct an ablation in which the defense is restricted to a \emph{single} transformation step. We evaluate this ``single-step \etthree'' on both zero-shot classification and downstream LVLM tasks.

The results—shown in \cref{tab:robust-llava-one_step} for the LVLM experiments and in \cref{tab:imagnet_1k_one_step_et,tab:attack_strength_ablation_1_step} for the zero-shot evaluations—demonstrate that even a single step yields a substantial robustness improvement over the baseline. Although the full multi-step version of \etthree achieves the slightly stronger overall performance. For this ablation, we keep the overall perturbation budget $\epsilon$ identical to the multi-step setup, increasing the step size $\alpha$ to 5 for TeCoA and 4 for FARE.

\renewcommand{\arraystretch}{1.0}
\definecolor{darkgray}{gray}{0.57}
\begin{table*}[t]
\centering

\caption{\textbf{Evaluating LLaVA 1.5-7B with \etthree across different vision encoders in the \emph{defense-aware setting}.} Robust scores is reported under $\epsilon_{a}=4/255$ using standard CLIP and the TeCoA/FARE backbones adversarially trained with $\epsilon_t=2/255$ and $\epsilon_t=4/255$. Across all tasks, ImageNet-21k labels serve as the reference text embeddings for computing the energy $E(\cdot,\theta)$. The column labeled $(4/255)$ reports results without any test-time defense. \textbf{+ET3} denotes evaluation under a non-adaptive attack, while \textbf{+ET3$^*$} denotes evaluation under an defense-aware adaptive attack. COCO and Flickr30k are evaluated using CIDEr for captioning, while TextVQA and VQAv2 report VQA accuracy.}

\scriptsize
\setlength{\tabcolsep}{2pt}
\resizebox{\textwidth}{!}{
\begin{tabular}{
  >{\raggedright\arraybackslash}m{8.5mm}
  Gc Gc Gc
  Gc Gc Gc
  Gc Gc Gc
  Gc Gc Gc
  Gc Gc Gc}
\toprule
& \multicolumn{3}{c}{\cellcolor{cianoChiaro}\textbf{COCO} \cite{cocodataset}} 
& \multicolumn{3}{c}{\cellcolor{cianoChiaro}\textbf{Flickr30k} \cite{flickr30k}}  
& \multicolumn{3}{c}{\cellcolor{cianoChiaro}\textbf{TextVQA} \cite{singh2019towards}} 
& \multicolumn{3}{c}{\cellcolor{cianoChiaro}\textbf{VQAv2} \cite{goyal2017making}} 
& \multicolumn{3}{c}{\cellcolor{cianoChiaro}\textbf{Average}} \\
\arrayrulecolor{LightCyan4}
\cmidrule(lr){2-4}
\cmidrule(lr){5-7}
\cmidrule(lr){8-10}
\cmidrule(lr){11-13}
\cmidrule(lr){14-16}
& 4/255 &   {+ET3} &   {+ET3}$^*$
& 4/255 &   {+ET3} &   {+ET3}$^*$
& 4/255 &   {+ET3} &   {+ET3}$^*$
& 4/255 &   {+ET3} &   {+ET3}$^*$
& 4/255 &   {+ET3} &   {+ET3}$^*$ \\
\arrayrulecolor{LightCyan4}\midrule

\clip
& 2.8  &   {19.2} & 3.5
& 0.9  &   {16.0} & 5.0
& 0.0  &   {12.0} & 2.0
& 0.0  &   {19.6} & 7.4
& 0.9  &   {16.7} \smallblue{15.8} & 4.5 \smallblue{3.6} \\

\tecoatwo
& 24.9 &   {43.8} & 34.9
& 20.9 &   {30.3} & 29.4
& 5.8  &   {13.8} & 11.8
& 21.8 &   {35.8} & 25.8
& 18.3 &   {30.9} \smallblue{12.6} & 25.5 \smallblue{7.2} \\

\faretwo
& 21.8 &   {37.0} & 33.4
& 22.5 &   {29.5} & 28.4
& 5.8  &   {15.8} & 11.8
& 20.4 &   {33.2} & 26.4
& 17.6 &   {28.9} \smallblue{11.3} & 25.0 \smallblue{7.4} \\

\arrayrulecolor{cianoVeryChiaro}\midrule

\tecoafour
& 25.4 &   {45.0} & 41.4
& 22.4 &   {31.6} & 29.6
& 7.8  &   {9.8}  & 9.8
& 23.0 &   {37.4} & 33.4
& 19.7 &   {30.9} \smallblue{11.2} & 28.6 \smallblue{8.9} \\

\farefour
& 28.1 &   {36.2} & 34.3
& 28.0 &   {39.5} & 39.9
& 13.8 &   {19.8} & 17.8
& 26.6 &   {35.8} & 31.8
& 24.1 &   {32.8} \smallblue{8.7} & 31.0 \smallblue{6.9} \\

\arrayrulecolor{black}\bottomrule
\end{tabular}
}
\label{tab:vlm_adaptive_l2}
\end{table*}

\renewcommand{\arraystretch}{1.0}
\definecolor{darkgray}{gray}{0.57}

\begin{table*}[t]
\centering
\caption{\textbf{Evaluating LLaVA 1.5-7B with \etthree with \emph{$\ell_\infty$ projection} across different vision encoders in the \emph{defense-aware setting}.} Robust scores is reported under $\epsilon_{a}=4/255$ using standard CLIP and the TeCoA/FARE backbones adversarially trained with $\epsilon_t=2/255$ and $\epsilon_t=4/255$. The \etthree update is projected into the same $\ell_\infty$ ball as the attack with radius $4/255$. Across all tasks, ImageNet-21k labels serve as the reference text embeddings for computing the energy $E(\cdot,\theta)$. The column labeled $(4/255)$ reports results without any test-time defense. \textbf{+ET3} denotes evaluation under a non-adaptive attack, while \textbf{+ET3$^*$} denotes evaluation under an defense-aware adaptive attack. COCO and Flickr30k are evaluated using CIDEr for captioning, while TextVQA and VQAv2 report VQA accuracy.}

\scriptsize
\setlength{\tabcolsep}{2pt}
\resizebox{\textwidth}{!}{
\begin{tabular}{
  >{\raggedright\arraybackslash}m{8.5mm}
  Gc Gc Gc
  Gc Gc Gc
  Gc Gc Gc
  Gc Gc Gc
  Gc Gc Gc}
\toprule
& \multicolumn{3}{c}{\cellcolor{cianoChiaro}\textbf{COCO} \cite{cocodataset}} 
& \multicolumn{3}{c}{\cellcolor{cianoChiaro}\textbf{Flickr30k} \cite{flickr30k}}  
& \multicolumn{3}{c}{\cellcolor{cianoChiaro}\textbf{TextVQA} \cite{singh2019towards}} 
& \multicolumn{3}{c}{\cellcolor{cianoChiaro}\textbf{VQAv2} \cite{goyal2017making}} 
& \multicolumn{3}{c}{\cellcolor{cianoChiaro}\textbf{Average}} \\
\arrayrulecolor{LightCyan4}
\cmidrule(lr){2-4}
\cmidrule(lr){5-7}
\cmidrule(lr){8-10}
\cmidrule(lr){11-13}
\cmidrule(lr){14-16}
& 4/255 &   {+ET3} &   {+ET3}$^*$
& 4/255 &   {+ET3} &   {+ET3}$^*$
& 4/255 &   {+ET3} &   {+ET3}$^*$
& 4/255 &   {+ET3} &   {+ET3}$^*$
& 4/255 &   {+ET3} &   {+ET3}$^*$ \\
\arrayrulecolor{LightCyan4}\midrule

\clip
& 2.8  &   {16.1} & 4.0
& 0.9  &   {12.8} & 2.9
& 0.0  &   {6.0}  & 0.0
& 0.0  &   {14.6} & 6.2
& 0.9  &   {12.4} \smallblue{11.5} & 3.3 \smallblue{2.4} \\

\tecoatwo
& 24.9 &   {42.1} & 36.1
& 20.9 &   {31.6} & 29.8
& 5.8  &   {13.8} & 11.8
& 21.8 &   {37.8} & 29.8
& 18.3 &   {31.3} \smallblue{13.0} & 26.9 \smallblue{8.6} \\

\faretwo
& 21.8 &   {31.4} & 30.0
& 22.5 &   {33.2} & 30.2
& 5.8  &   {13.8} & 10.6
& 20.4 &   {29.8} & 23.8
& 17.6 &   {27.0} \smallblue{9.4} & 23.7 \smallblue{6.1} \\

\arrayrulecolor{cianoVeryChiaro}\midrule

\tecoafour
& 25.4 &   {45.3} & 39.9
& 22.4 &   {32.2} & 31.1
& 7.8  &   {9.8}  & 7.8
& 23.0 &   {37.4} & 33.4
& 19.7 &   {31.1} \smallblue{11.4} & 28.1 \smallblue{8.4} \\

\farefour
& 28.1 &   {34.2} & 35.1
& 28.0 &   {40.2} & 37.9
& 13.8 &   {19.8} & 19.8
& 26.6 &   {35.8} & 31.8
& 24.1 &   {32.5} \smallblue{8.4} & 31.1 \smallblue{7.0} \\

\arrayrulecolor{black}\bottomrule
\end{tabular}
}
\label{tab:vlm_adaptive_linf}
\end{table*}

\colorlet{cianoChiaro}{LightCyan3!20!white}
\colorlet{cianoVeryChiaro}{LightCyan4!80!white}

\newcolumntype{J}{>{\columncolor{cianoChiaro}}c}
\begin{table*}[t!]
\caption{\textbf{Zero-shot robustness of \etthree with \emph{$\ell_\infty$ projection} across 14 benchmark datasets in the defense-unaware setting.} Comparison of clean and robust accuracy for baseline models versus the same models augmented with \etthree. Robustness is evaluated against Auto-Attack (AA) at $\epsilon_a = 4/255$. The \etthree update is projected into the same $\ell_\infty$ ball as the attack, with radius $4/255$.}
    \centering
    \scriptsize
    \setlength{\tabcolsep}{1.8pt}
    \resizebox{\textwidth}{!}{
    \begin{tabular}{cc|>{\columncolor{lightgray}}Cc
    >{\columncolor{lightgray}}Cc
    >{\columncolor{lightgray}}Cc
    >{\columncolor{lightgray}}Cc
    >{\columncolor{lightgray}}Cc
    >{\columncolor{lightgray}}Cc
    >{\columncolor{lightgray}}Cc|cc}
    \toprule
    Model & Method & 
    \rotatebox[origin=c]{60}{ImageNet} & 
    \rotatebox[origin=c]{60}{CalTech} & 
    \rotatebox[origin=c]{60}{Cars} & 
    \rotatebox[origin=c]{60}{CIFAR10} & 
    \rotatebox[origin=c]{60}{CIFAR100} & 
    \rotatebox[origin=c]{60}{DTD} & 
    \rotatebox[origin=c]{60}{EuroSAT} & 
    \rotatebox[origin=c]{60}{FGVC} & 
    \rotatebox[origin=c]{60}{Flowers} & 
    \rotatebox[origin=c]{60}{ImageNet-R} & 
    \rotatebox[origin=c]{60}{ImageNet-S} & 
    \rotatebox[origin=c]{60}{PCAM} & 
    \rotatebox[origin=c]{60}{OxfordPets} &  
    \rotatebox[origin=c]{60}{STL-10} & 
    \rotatebox[origin=c]{60}{Avg.} & 
    \rotatebox[origin=c]{60}{\textbf{Improv.}}\\
    \midrule
    \multirow{4}{*}{\shortstack{ViT-L/14 \\(TeCoA)\\$\epsilon_t=4/255$}}
    & Base (Clean)& 74.91 & 78.36 & 37.83 & 79.61 & 50.26 & 38.03 & 22.48 & 11.76 & 38.41 & 74.35 & 54.22 & 49.95 & 76.07 & 93.44 & 55.69  & 
    \\
    & \textbf{+ ET3} (Clean) & 74.78 & 78.09 & 36.81 & 80.77 & 50.26 & 38.09 & 21.72 & 12.27 & 39.05 & 73.66 & 54.95 & 49.98 & 75.93 & 94.49 & 55.78 & \smallblue{0.09}\\
    \arrayrulecolor{cianoVeryChiaro}\cmidrule(lr){2-18}
     & Base (Robust) & 44.50 & 60.90 & 8.50 & 37.10 & 21.50 & 16.50 & 6.40 & 2.20 & 12.60 & 41.90 & 32.80 & 45.70 & 55.00 & 74.30 & 32.85  
    \\
    & \textbf{+ ET3} (Robust) & 52.00 & 64.40 & 10.70 & 51.30 & 31.60 & 21.70 & 13.40 & 4.60 & 20.10 & 48.30 & 39.80 & 51.00 & 60.60 & 82.20 & 39.41 & {\smallblue{6.56}}\\
    \arrayrulecolor{LightCyan4}\midrule
    
    \end{tabular}
    } %
    \label{tab:tab_clip_linf}
\end{table*}

\subsection{Budget fairness}\label{app:budget}
Evaluating the robustness of Test-Time Training (TTT) methods introduces unique considerations compared to adversarial training defenses. In Adversarial Training (AT), fair evaluation strictly requires matching the defense budget to the threat model's attack budget. Our base robust models adhere to this standard: they are trained with $\ell_\infty$ constraints and evaluated against $\ell_\infty$ attacks that equal or exceed the training budget. However, because \etthree is a TTT method operating actively within the \textit{inference pipeline}, its parameters serve a fundamentally different purpose. 

The primary goal of \etthree is to amplify the features of the ground-truth concept rather than to simply mask adversarial noise. Consequently, \etthree utilizes an $\ell_2$ projection, even when evaluating against $\ell_\infty$ attacks. The $\ell_2$ constraint regulates the extent of the TTT transformation to ensure clean accuracy is preserved; it is not designed to match the adversarial threat constraint. Thus, the relatively ``large'' $\ell_2$ radius utilized by \etthree should be understood as a hyperparameter of the defense pipeline rather than an unfair budget advantage. 

This inference-time transformation also changes how we evaluate perceptual quality. In adversarial purification, where external generative models reconstruct images independently of the classifier, pixel-level fidelity is a primary constraint to avoid introducing new perceptual artifacts. In contrast, \etthree performs its transformation directly within the predictive model itself. Therefore, the relevant criterion is the preservation of \emph{clean accuracy} rather than the raw visual fidelity of the image (visualizations of \etthree-transformed images are provided in the Fig. \ref{fig:logits}).
 
While our $\ell_2$ formulation is by design, we also demonstrate that \etthree's effectiveness is not merely an artifact of utilizing a larger or different norm. To confirm this, we evaluate \etthree using an $\ell_\infty$ projection strictly equal to the attack budget ($\epsilon=4/255$). Under the same settings as Table 1, the average performance across 14 datasets for TeCoA (ViT-L/14, $\epsilon=4/255$) demonstrates consistent gains: clean accuracy improves from 55.69\% to 55.78\%, and robust accuracy improves significantly from 32.85\% to 39.41\% as shown in \cref{tab:tab_clip_linf}. Finally, we also evaluate \etthree with this strict $\ell_\infty$ projection on Large Vision-Language Models (LVLMs) against adaptive attacks as shown in \cref{tab:vlm_adaptive_linf}. Across all benchmark datasets, the improvement persists when using a robust vision encoder, further confirming that genuine feature induction drives the observed robustness gains.

\subsection{Performance under Increased Attack Strength}\label{app:increased_attack_strength}

To further evaluate the resilience of \etthree, we conduct an ablation in which we systematically increase the attack strength. For this study, we focus on zero-shot evaluation and use CLIP models whose image encoders are finetuned with TECoA~\cite{mao2023understanding}, trained with perturbation budgets of $\epsilon_t = 2/255$ and $\epsilon_t = 4/255$, respectively.

We then evaluate two configurations of our defense under attacks of varying strength:  
\textbf{Default Defense:} our standard \etthree configuration with a transformation budget of $\epsilon = 5$ and $\alpha = 2.5$.  \textbf{Stronger Defense:} a configuration with increased bound for defense transformation, $\epsilon = 10$ and $\alpha = 5$. In both setting, the number of steps is set to 2.

As shown in Figure~4 of the main paper, \etthree maintains a consistent robustness advantage as the attack strength increases, with results averaged across all benchmarks. Detailed numerical results are provided in \cref{tab:attack_strength_ablation_2_step}. We observe that \etthree improves robustness across all attack strengths in a stable  manner. Furthermore, increasing the defense budget to $\epsilon = 10$ yields additional improvements, particularly under the strongest adversarial settings. These findings indicate that \etthree is not only effective under threat levels the robust model has been trained for but also scales gracefully to stronger adversaries without any additional training. We also provide the same analysis when the using only single-step \etthree defense in \cref{tab:attack_strength_ablation_1_step}.

\subsection{Impact of Label Set Choice for \etthree}\label{app:label-set_guide}
As described in the main paper, our energy-based defense, \etthree, leverages a set of class labels to guide its energy minimization process. A critical design choice is the composition of this label set. We considered two primary options: 

\textbf{A Vast, General-Purpose Label Set:} Using a comprehensive set of labels such as the $\sim$21,000 classes from the full ImageNet-21k dataset. We use these labels without any further preprocessing, treating each row as one class, for example: \{person, individual, someone, somebody, mortal, soul\} as one class. We obtain the full set from \footnote{\url{https://github.com/mosjel/ImageNet_21k_Original_OK}}.

\textbf{A Refined, curated set of labels:} Manually curating and refining a label set of such magnitude for every potential use case is impractical and outside the scope of this work. Therefore, for our experiments, we adopt the more practical approach of using the refined label set included in the evaluation dataset itself: in this case, we use the set of labels associated with the specific downstream benchmark (e.g., using all the 1,000 class labels of ImageNet-1k when evaluating on it).

Throughout this work, we evaluate both label-set choices, though we predominantly rely on the 21-k proxy ImageNet labels. 
Specifically, the LVLM experiments shown in Table 3 of the main paper, as well as the zero-shot robustness results presented in Figure 4, use the full 21k ImageNet label set. Additional results using this label set appear in \cref{tab:robust-llava,tab:imagnet_2_step_et,tab:imagnet_1k_one_step_et,tab:attack_strength_ablation_1_step,tab:attack_strength_ablation_2_step}. In contrast, Tables 1 and 2 of the main paper, along with \cref{tab:more_robust_models} and \cref{tab:tab_clip_linf}, use the label sets associated with their respective evaluation benchmarks.

Using the refined, dataset-specific label set has a negligible impact on clean accuracy while still providing comparable improvements in robustness. Overall, we observe that the 21k label set yields slightly higher robustness than the refined label set, albeit at a modest cost in clean accuracy. The extent of this drop varies across models and methods—TeCoA is minimally affected, whereas FARE is impacted more noticeably.

To better understand FARE’s clean-accuracy drop, we conducted additional analysis and used the \textbf{full 21k labels for evaluation rather than the dataset-specific labels} commonly adopted in standard zero-shot evaluation practices (without applying \etthree or any attacks). We found that classes such as “cat” and “dog” are frequently mapped to semantically related but incorrect labels, including “petfood,” “pet-food,” or “pet food.” When this occurs, the transformed image obtained after \etthree becomes more likely to be misclassified, as the \etthree amplifies features associated with these incorrect labels. Our analysis here is intentionally preliminary and does not constitute a comprehensive study of how zero-shot evaluation should be designed or assessed; a more thorough investigation lies beyond the scope of this work.

Nevertheless, in realistic deployment scenarios, large and diverse label sets are typically more appropriate for zero-shot classification. Under such conditions, we would not expect to observe the same degree of clean-accuracy degradation that appears in these controlled experimental settings.

\begin{figure*}
    \centering
    
    \begin{subfigure}{\textwidth}
        \centering
        \begin{subfigure}{0.48\textwidth}
            \centering
            \begin{overpic}[width=\linewidth]{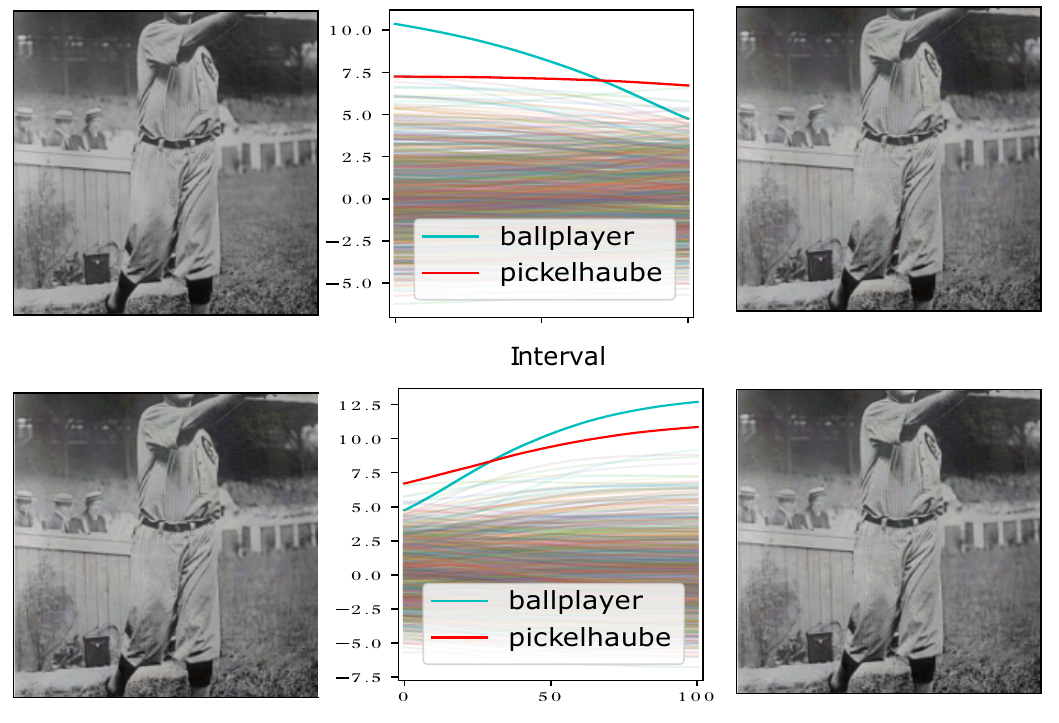}
            \put(14,68){$\bx$}
            \put(14,32){$\bxa$}
            \put(83,68.5){$\bxa$}
            \put(83,32){$\bxp$}
        \end{overpic}
        \end{subfigure}
        \hfill
        \begin{subfigure}{0.48\textwidth}
            \centering
            \begin{overpic}[width=\linewidth]{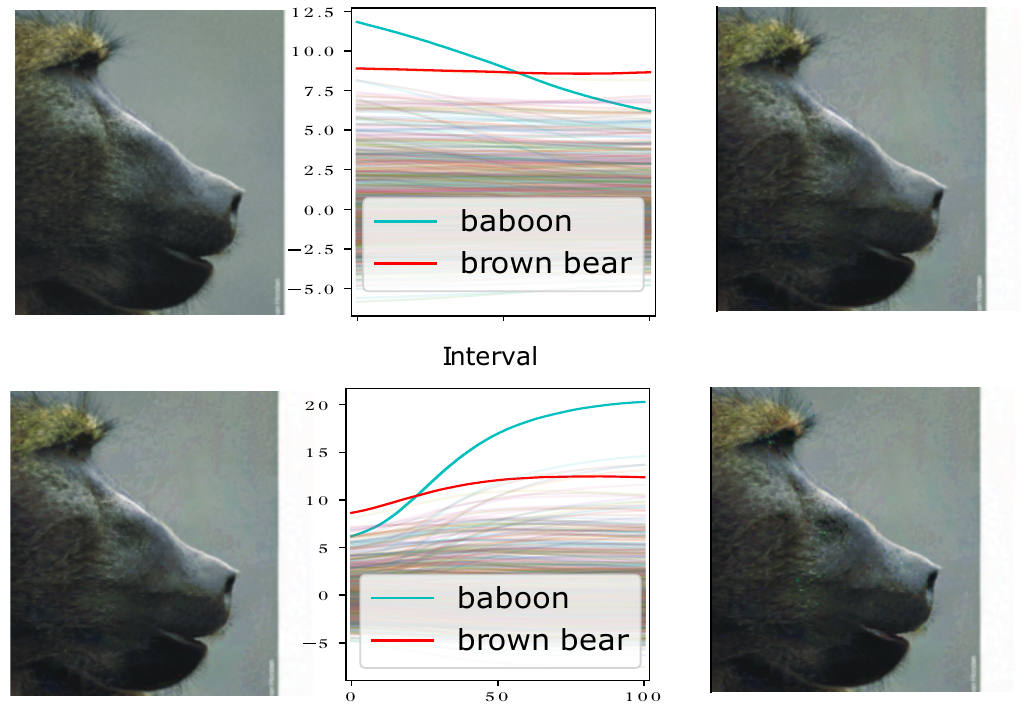}
              \put(14,70){$\bx$}
            \put(14,33){$\bxa$}
            \put(83,70){$\bxa$}
            \put(83,33){$\bxp$}
        \end{overpic}
        \end{subfigure}
        \\[1em]
        \begin{subfigure}{0.48\textwidth}
            \centering
            \begin{overpic}[width=\linewidth]{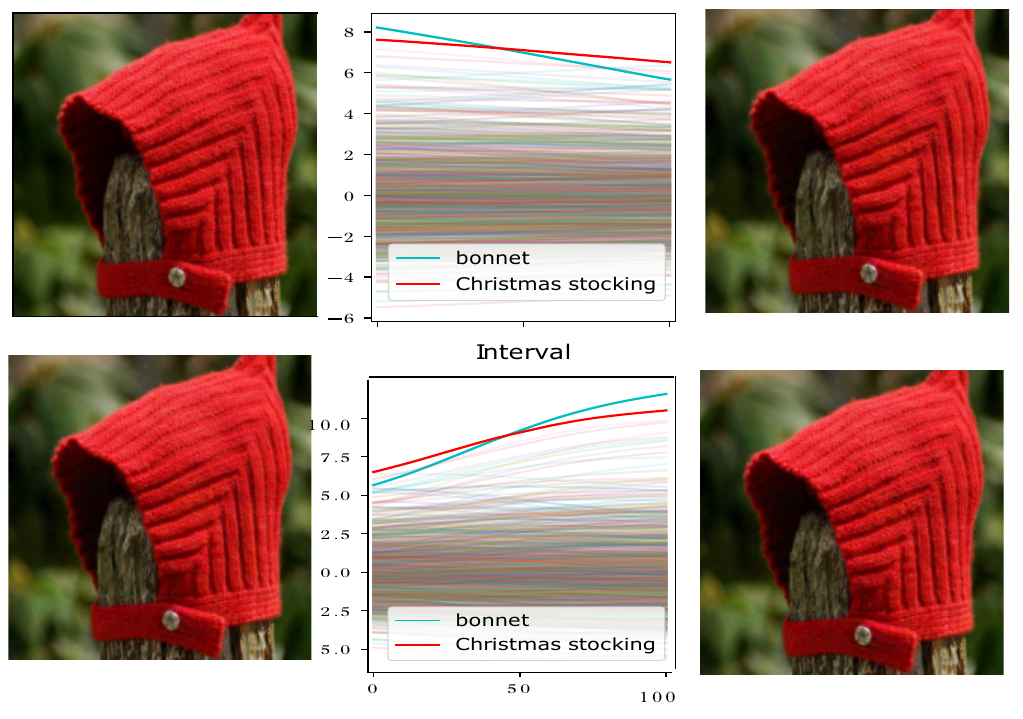}
              \put(14,70){$\bx$}
            \put(14,36){$\bxa$}
            \put(83,70){$\bxa$}
            \put(83,35){$\bxp$}
        \end{overpic}
        \end{subfigure}
        \hfill
        \begin{subfigure}{0.48\textwidth}
            \centering
          \begin{overpic}[width=\linewidth]{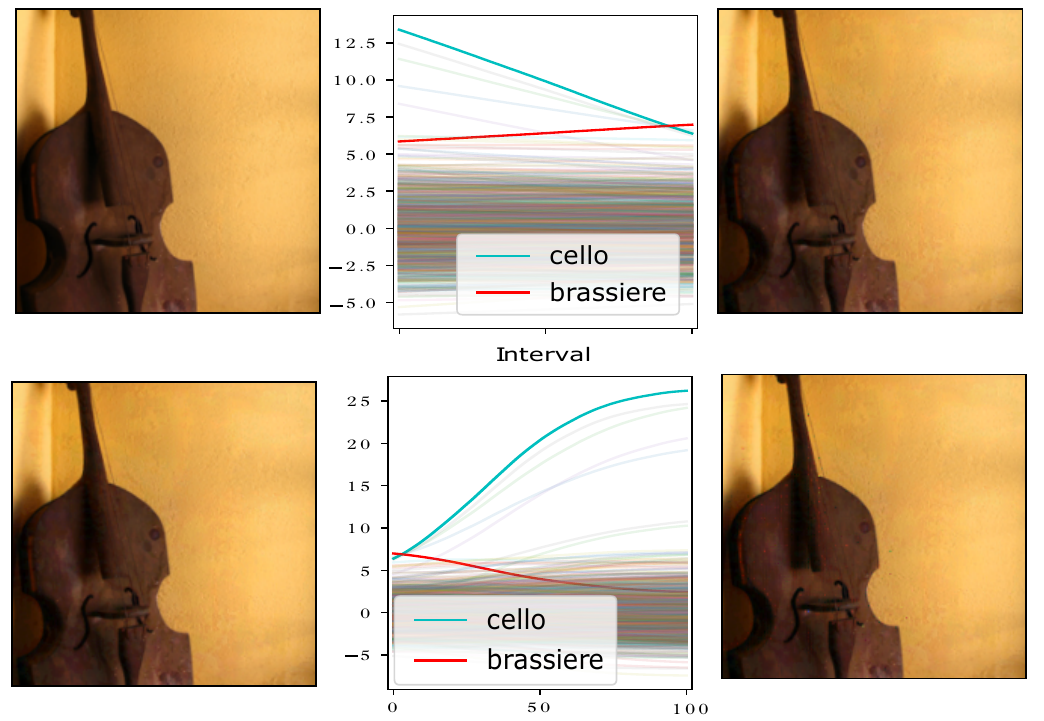}
              \put(14,69){$\bx$}
            \put(14,34){$\bxa$}
            \put(83,68){$\bxa$}
            \put(83,34){$\bxp$}
        \end{overpic}
        \end{subfigure}
    \caption{For each example, we progressively scale the perturbation from 0\% to 100\% in 100 equal steps and plot how the model’s output (logits) changes across this progression. For each individual example, the top row shows this behavior for the adversarial perturbation, while the bottom row shows the same procedure applied to the \etthree (transformation) perturbation.}

        \label{fig:logits_general}
    \end{subfigure}

    \vspace{10pt}
        
    \begin{subfigure}{\textwidth}
        \centering
        \begin{subfigure}{0.58\textwidth}
            \centering
            \begin{overpic}[width=\linewidth]{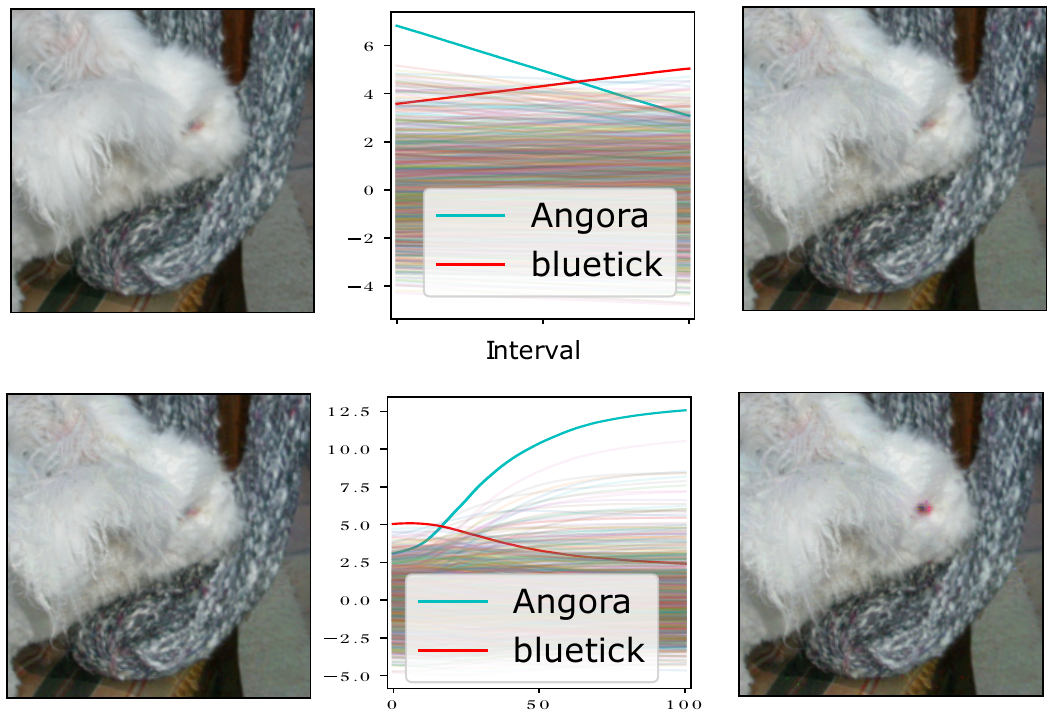}
              \put(14,69){$\bx$}
            \put(14,33){$\bxa$}
            \put(83,69.9){$\bxa$}
            \put(83,34){$\bxp$}
        \end{overpic}
        \end{subfigure}
        \hfill
        \begin{subfigure}{0.38\textwidth}
            \centering
            \raisebox{20pt}{
            \begin{overpic}[width=0.7\linewidth]{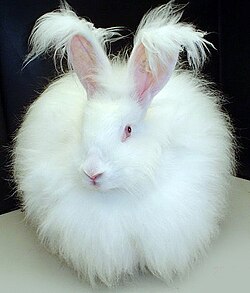}
              \put(5,103){\footnotesize{reference image of an Angora rabbit}}
        \end{overpic}
    }
        \end{subfigure}
        \caption{The left panel shows the \etthree transformation applied to an Angora bunny image that was originally misclassified as a Blue Tick. \etthree enhances salient features, most notably the pinkish eye region, that are essential for recognizing the correct class. The right panel provides a generic reference image of an Angora bunny.}

        \label{fig:logits_bunny}
    \end{subfigure}
    
    \caption{ Presenting a natural image $\bx$, and its adversarial image $\bxa$ wrongly classified by a robust classifier $f_\theta$. Given only $\bxa$ and $f_\theta$, our \etthree produces $\bxp$ which is correctly classified.}
    \label{fig:logits}
\end{figure*}

\newcommand{\bu}{\mathbf{u}}

\onecolumn
\section{Purification in Robust Networks (Proof for Theorem~4.1)} 
\label{app:proof}
\begin{proof}

For $\bv = f(\bx)$, we denote the energy loss function we wish to minimize by

\begin{align*}
    E(\bv) = - \log \sum_{i \in \{-1,1\}} e^{\bv_i}  
\end{align*}

thus its gradient with respect to the logits vector $\bv$ will be a two dimensional vector

\begin{align*}
    \nabla_{\bv} E(\bv) = -\softmax(\bv) = \left(-\frac{e^{\bv_{-1}}}{\sum_{i \in \{-1,1\}} e^{\bv_i}}, -\frac{e^{\bv_{1}}}{\sum_{i \in \{-1,1\}} e^{\bv_i}}\right) ~.
\end{align*}

Therefore, the gradient of the energy w.r.t. the input calculated during the defense \etthree is

\begin{align*}
    \frac{\partial E(f_\theta(\bx))}{\partial \bx} = -\softmax(f_\theta(\bx))^T \frac{\partial f_\theta(\bx)}{\partial \bx}~.
\end{align*}

We denote 
\[\bg_0 = \frac{\partial f_\theta(\bx)_0}{\partial \bx}, ~ \bg_1 = \frac{\partial f_\theta(\bx)_1}{\partial \bx}~,\]
and $e_0 = \softmax\left(f_\theta\left(\bx\right)\right)_0$ and $e_1 = \softmax\left(f_\theta\left(\bx\right)\right)_1$, the defense is calculating the gradient 
\begin{align*}
    \frac{\partial E(f_\theta(\bx))}{\partial \bx} =& -\softmax(f_\theta(\bx))^T \frac{\partial f_\theta(\bx)}{\partial \bx}  \\
    =& -\left( \softmax\left(f_\theta\left(\bx\right)\right)_0 \bg_0 + \softmax\left(f_\theta\left(\bx\right)\right)_1 \bg_1 \right) \\
    =& - \left( e_0 \bg_0 + e_1 \bg_1 \right)~.
\end{align*}

For the defense optimization we take a gradient descent step of a norm upper bounded by $\epsilon$, and get $\bx_p = \bx + \bz$ for 
\begin{align*}
    \bz = \alpha \left( e_0 \bg_0 + e_1 \bg_1\right)~.
\end{align*}

Since $\norm{\bz} \leq \epsilon$, the upper bound for $\alpha$ will be

\begin{align*}
    \alpha \leq \frac{\epsilon}{\norm{e_0 \bg_0 + e_1 \bg_1}}~.
\end{align*}

We remind the reader that $f_\theta(\bx)\in \reals^2$ by definition, leading to $\frac{\partial f_\theta(\bx)}{\partial \bx} \in \reals^2 \times \reals^d$, thus for readability, we denote two functions, $f_0(\bx) = f_\theta(\bx)_{0}$ and $f_1(\bx) = f_\theta(\bx)_1$, concluding that $f_\theta(\bx) = [f_0(\bx), f_1(\bx)]$. Following the local linearity assumption of $f_\theta$ in $\mathcal{B}_\epsilon(\bx)$, $f_0$ and $f_1$ are linear functions in $\mathcal{B}_\epsilon(\bx)$, and we note that for any $\bx' \in \mathcal{B}_\epsilon(\bx)$
\[f_0(\bx') = \inner{\bg_0, \bx'} + a_0~,~ f_1(\bx') = \inner{\bg_1, \bx'} + a_1 \]

for some $a_0, a_1 \in \reals$%

We are now ready to show that for an input $\bx$ with ground truth label $y_t$, the defense permutation $\bz$ leads to
\[
f_\theta(\bx+\bz)_{y_t} > f_\theta(\bx+\bz)_{\hat{y_t}}~.
\]
We look at $f(\bx+\bz)$, having

\begin{align*}
    f_0(\bx + \bz) = \inner{\bg_0, \bx+\bz} + a_0 = f_0(\bx) + \inner{\bg_0, \bz}\\
    f_1(\bx + \bz) = \inner{\bg_1, \bx+\bz} + a_1 = f_1(\bx) + \inner{\bg_1, \bz}\\
\end{align*}
We denote 
\[r_{\bx} = f(\bx)_1 - f(\bx)_{0}~.\]

We assume W.L.O.G that the truth label is $y_t = 1$. The case $y_t = 0$ is proven similarly. We have $C>1$ and
\[ C \norm{e_0 \bg_0} \leq  \norm{e_1 \bg_1}~,\]
leading to 
\begin{align*}
    \norm{\bg_0} \leq \frac{e_1}{C e_0} \norm{\bg_1}~.
\end{align*}
We show that $f_1(\bx + \bz) - f_0(\bx + \bz) > 0$. We have

\begin{align*}
    f_1(\bx + \bz) - f_0(\bx + \bz)  =& f_1(\bx) - f_0(\bx)  + \inner{\bg_1, \bz} -  \inner{\bg_0, \bz} \\
    =& r_x + \inner{\bg_1, \alpha \left( e_0 \bg_0 + e_1 \bg_1\right)} - \inner{\bg_0, \alpha \left( e_0 \bg_0 + e_1 \bg_1\right)} \\
    =& r_x + \alpha \left[e_1 \norm{\bg_1}^2 - e_0 \norm{\bg_0}^2 + \left( e_0 - e_1\right) \inner{\bg_0, \bg_1} \right] \geq \\
    \geq& r_x + \alpha \left[ e_1 \norm{\bg_1}^2 - \frac{e_1^2}{C^2 e_0} \norm{\bg_1}^2 + \left( e_0 - e_1\right) \inner{\bg_0, \bg_1}  \right]\\
    \geq& r_x + \alpha \left[ \left( e_1 - \frac{e_1^2}{C^2 e_0} \right)\norm{\bg_1}^2 + \left( e_0 - e_1\right) \inner{\bg_0, \bg_1}  \right]~,
\end{align*}

for $\alpha \leq \frac{\epsilon}{\norm{e_0 \bg_0 + e_1 \bg_1}}$. We note that 
\begin{align*}
    \norm{e_0 \bg_0 + e_1 \bg_1}^2 =& \norm{e_0 \bg_0}^2 + \norm{e_1 \bg_1}^2 + 2 \inner{e_0 \bg_0 , e_1 \bg_1} \\
    =& e_0^2 \norm{\bg_0}^2 + e_1^2 \norm{\bg_1}^2 + 2e_0e_1\inner{\bg_0, \bg_1} \\
    \leq& \left( \frac{1}{C^2} + 1 \right) \norm{e_1\bg_1}^2 + 2\inner{e_0\bg_0, e_1\bg_1}\\
    \leq& \norm{e_1\bg_1}^2 \left( \left( \frac{1}{C^2} + 1 \right) + \frac{2 \inner{e_0\bg_0, e_1\bg_1}}{\norm{e_1\bg_1}^2} \right) \\
    \leq& \norm{e_1\bg_1}^2 \left( \frac{1}{C^2} + 1 + \frac{2}{C} \right) \\
    \leq& \norm{e_1\bg_1}^2 \left(1+ \frac{1}{C}\right)^2~,
\end{align*}
where the last inequality hold since we assumed that $C \norm{e_0 \bg_0} \leq  \norm{e_1 \bg_1}$, and for any two vectors $\bu_1, \bu_2$ we have that $\frac{\inner{\bu_1, \bu_1}}{\norm{\bu_1}\norm{\bu_2}}\leq 1$. Therefore, if we take

\begin{align*}
    \alpha =  \frac{\epsilon}{e_1  \left(1+ \frac{1}{C}\right) \norm{\bg_1}} \leq \frac{\epsilon}{\norm{e_0 \bg_0 + e_1 \bg_1}}
\end{align*}

we get

\begin{align*}
    f_1(\bx + \bz) - f_0(\bx + \bz)  
    \geq& r_x + \alpha \left[ \left( e_1 - \frac{e_1^2}{C^2 e_0} \right)\norm{\bg_1}^2 + \left( e_0 - e_1\right) \inner{\bg_0, \bg_1}  \right]\\
    \geq& r_x + \frac{\epsilon}{e_1  \left(1+ \frac{1}{C}\right)  \norm{\bg_1}} \left[ \left( e_1 - \frac{e_1^2}{C^2 e_0} \right)\norm{\bg_1}^2 + \left( e_0 - e_1\right) \inner{\bg_0, \bg_1}  \right] \\
    \geq& r_x + \epsilon \norm{\bg_1} \left[ \frac{ e_1 - \frac{e_1^2}{C^2 e_0} }{e_1  \left(1+ \frac{1}{C}\right) } + \frac{\left( e_0 - e_1\right) \inner{\bg_0, \bg_1}}{e_1  \left(1+ \frac{1}{C}\right)  \norm{\bg_1}^2}  \right] \\
    \geq& r_x + \epsilon \norm{\bg_1} \left[ \frac{e_1 - \frac{e_1^2}{C^2 e_0} }{e_1  \left(1+ \frac{1}{C}\right) } - \frac{e_0 - e_1 }{e_1  \left(1+ \frac{1}{C}\right) C}  \right] \\
    \geq& r_x + \epsilon \norm{\bg_1} \left[ \frac{ 1 - \frac{e_1}{C^2 e_0} }{ 1+ \frac{1}{C} } - \frac{e_0 - e_1}{e_1  \left(1+ \frac{1}{C}\right) C}  \right] \\
    \geq& r_x + \epsilon \norm{\bg_1} \left[ \left(\frac{1}{1+ \frac{1}{C}}\right) \left(1 - \frac{e_1}{C^2 e_0}  - \frac{ e_0 - e_1}{e_1  C}  \right)\right] \\
    \geq& r_x + \epsilon \norm{\bg_1} \left[ \left(\frac{1}{1+ \frac{1}{C}}\right) \left(1 - \frac{e_1}{C^2 e_0}  - \frac{ e_0 }{e_1  C} + \frac{1}{C}  \right)\right]~.
\end{align*}

We note that 
\begin{align*}
    \frac{e_0}{e_1} = \frac{\softmax(f(\bx))_0}{\softmax(f(\bx))_1} = \frac{\frac{\exp(f(\bx)_0)}{\exp(f(\bx)_0)+\exp(f(\bx)_1)}}{\frac{\exp(f(\bx)_1)}{\exp(f(\bx)_0)+\exp(f(\bx)_1)}} = \exp(f(\bx)_0 - f(\bx)_1) = \exp(-r_x)~,
\end{align*}

and similarly $\frac{e_1}{e_0} \leq \exp(r_x)$, having 
\begin{align*}
        f_1(\bx + \bz) - f_0(\bx + \bz)  
    \geq& r_x + \epsilon \norm{\bg_1} \left[ \left(\frac{1}{1+ \frac{1}{C}}\right) \left(1 - \frac{e_1}{C^2 e_0}  - \frac{ e_0 }{e_1  C} + \frac{1}{C}  \right)\right] \\
    \geq& r_x + \epsilon \norm{\bg_1} \frac{1}{2}\left(1 - \frac{\exp(r_x)}{C^2}  - \frac{\exp(-r_x)}{C} + \frac{1}{C}  \right)\\
    \geq& r_x + \epsilon \norm{\bg_1} \frac{1}{2}\left(1 - \frac{1}{C^2} - \frac{\exp(|r_x|)}{C}  + \frac{1}{C}  \right)~,
\end{align*}
where the last inequality holds since 
\[
- \frac{\exp(r_x)}{C^2}  - \frac{\exp(-r_x)}{C} = - \frac{\exp(r_x) + C \exp(-r_x)}{C^2} \geq -\frac{\exp(-|r_x|) + C \exp(|r_x|)}{C^2} \geq -\frac{1 + C \exp(|r_x|)}{C^2}~.
\]

Therefore we have

\begin{align*}
        f_1(\bx + \bz) - f_0(\bx + \bz)  
    \geq& r_x + \epsilon \norm{\bg_1} \frac{1}{2}\left(1 - \frac{1}{C^2} - \frac{\exp(|r_x|)}{C}  + \frac{1}{C}  \right)\\
    \geq& r_x + \epsilon \norm{\bg_1} \frac{1}{2}\left(1  - \frac{\exp(|r_x|)}{C} \right)\\
    =& r_x + \frac{\epsilon \norm{\bg_1} }{2} - \frac{\exp(|r_x|)\epsilon \norm{\bg_1} }{2C} \\
    =& \frac{1}{2}  + \frac{r_x}{\epsilon \norm{\bg_1}} - \frac{\exp(|r_x|)}{2C}~.
\end{align*}
We note that $\epsilon$ should satisfy 
\begin{align*}
    2 r_x + \epsilon \norm{\bg_1} > 0\\
    \epsilon > \frac{-2 r_x}{\norm{\bg_1}}
\end{align*}
for the correct classification to be possible in $\mathcal{B}_\epsilon(\bx)$. We note that this condition adds a necessary constraint only for adversarial samples, and applies directly where $\bx$ is already correctly classified.

Finally, for
\[
C > \frac{ \exp(|r_x|) \epsilon \norm{\bg_1}}{\epsilon \norm{\bg_1} + 2 r_x}
\]

the claim follows.

\end{proof}

\end{document}